\documentclass[twoside,11pt]{article}

\usepackage{blindtext}
\usepackage{lastpage}

\usepackage[hyperref,preprint]{jmlr2e}
\usepackage{amsmath,amsfonts,amssymb,mathrsfs,mathtools}%
\def\shownotes{0}
\renewcommand{\cite}{\citep}

\usepackage{macro_math}
\usepackage{enumitem}
\usepackage[us,12hr]{datetime} %
\usepackage{graphicx, subcaption}
\usepackage[export]{adjustbox}
\usepackage{subfloat}
\usepackage{nicefrac}       %
\usepackage{microtype}      %
\usepackage{stmaryrd}
\usepackage[utf8]{inputenc} %
\usepackage{caption,subcaption}
\usepackage{epsfig,epstopdf,url,array}
\usepackage{tikz}
\usepackage{color}
\usepackage{lipsum}
\usepackage{cleveref}
\newtheorem{assumption}[theorem]{Assumption}

\jmlrheading{26}{2025}{1-\pageref{LastPage}}{3/24; Revised 4/25}{6/25}{24-0454}{Fan Yang, Hongyang R. Zhang, Sen Wu, Christopher R\'e, and Weijie J. Su. The first two authors contributed equally and are listed in alphabetical order}

\ShortHeadings{Precise High-Dimensional Asymptotics for Quantifying Heterogeneous Transfers}{Yang, Zhang, Wu, R\'e, and Su}
\firstpageno{1}

\begin{document}

\title{Precise High-Dimensional Asymptotics for\\Quantifying Heterogeneous Transfers}

\author{\name Fan Yang \email fyangmath@mail.tsinghua.edu.cn\\
       \addr Yau Mathematical Sciences Center,
       Tsinghua University\\
       Beijing 100084, China
       \AND
       \name Hongyang R. Zhang \email ho.zhang@northeastern.edu \\
       \addr Khoury College of Computer Sciences,
       Northeastern University\\
       Boston, MA 02115, US
       \AND
        \name Sen Wu \email senwu@cs.stanford.edu \\
       \addr Department of Computer Science,
       Stanford University\\
       Stanford, CA 94305, US
       \AND
        \name Christopher R\'e \email chrismre@stanford.edu \\
       \addr Department of Computer Science,
       Stanford University\\
       Stanford, CA 94305, US
       \AND
       \name Weijie J. Su \email suw@wharton.upenn.edu \\
       \addr Department of Statistics and Data Science,
       University of Pennsylvania\\
       Philadelphia, PA 19104, US
       }

\editor{Mladen Kolar}
\maketitle
\begin{abstract}%
The problem of learning one task using samples from another task is central to transfer learning. In this paper, we focus on answering the following question: when does combining the samples from two related tasks perform better than learning with one target task alone? This question is motivated by an empirical phenomenon known as \emph{negative transfer}, which has been observed in practice. While the transfer effect from one task to another depends on factors such as their sample sizes and the spectrum of their covariance matrices, precisely quantifying this dependence has remained a challenging problem. In order to compare a transfer learning estimator to single-task learning, one needs to compare the risks between the two estimators precisely. Further, the comparison depends on the distribution shifts between the two tasks. This paper applies recent developments of random matrix theory to tackle this challenge in a high-dimensional linear regression setting with two tasks. We show precise high-dimensional asymptotics for the bias and variance of a classical hard parameter sharing (HPS) estimator in the proportional limit, where the sample sizes of both tasks increase proportionally with dimension at fixed ratios. The precise asymptotics apply to various types of distribution shifts, including covariate shifts, model shifts, and combinations of both. We illustrate these results in a random-effects model to mathematically prove a phase transition from positive to negative transfer as the number of source task samples increases. One insight from the analysis is that a rebalanced HPS estimator, which downsizes the source task when the model shift is high, achieves the minimax optimal rate. The finding regarding phase transition also applies to multiple tasks when covariates are shared across all tasks. Simulations validate the accuracy of the high-dimensional asymptotics for finite dimensions.
\end{abstract}

\section{Introduction}\label{sec_introduction}

Given samples from two tasks, does combining the samples from both tasks together yield a better estimator for the target task of interest?
Concretely, suppose there are $n_1$ samples drawn from a $p$-dimensional distribution $D_1$ with real-valued labels in a \emph{source} task.
Suppose there are $n_2$ samples drawn from a $p$-dimensional distribution $D_2$ with real-valued labels in the \emph{target} task.
Does combining the samples from both the source and the target tasks together lead to an estimator whose performance would be better than the learning outcome from utilizing the $n_2$ samples alone?

This question is motivated by scenarios where one would like to use samples from an auxiliary task to augment the sample size of a target task.
It is widely observed that when the two tasks are far apart, negative transfer can occur, where the transfer learning performance is worse than that of single-task learning \cite{PY09,WZR20}.
However, a rigorous proof of when and why negative transfer can happen (even within some simple statistical models) has remained elusive in the statistical learning literature.

Identifying negative transfer requires modeling the relationship between two tasks.
Early work has approached such questions through structural learning \cite{AZ05}, where a shared, low-dimensional structure is extracted from several (possibly unlabeled) tasks via an alternating optimization framework.
Another influential line of early work is the quantification of transfer learning with uniform convergence bounds \cite{BBCKP10}.
Notice that negative transfer can happen if the label distribution of task one, conditioned on the covariates, differs significantly from task two.
For brevity, we refer to settings where $D_1 \neq D_2$ as {covariate shift} and settings where label distributions differ (conditioned on the covariates) as model shift.
However, analyzing these settings requires going beyond the in-distributional assumptions of supervised learning, which has led to recent results in transfer learning for the two-task linear regression setting \cite{lei2021near,dar2022double}.
Recent results by \citet{dhifallah2021phase} shed light on an intriguing phase transition that has been observed in empirical transfer learning. They study both hard transfer and soft transfer via a precise analysis of the transfer rates.
In this paper, we propose to quantify transfer effects in high-dimensional linear regression by employing recently developed techniques from the random matrix theory literature in the proportional limit setting.

Concretely, we study a high-dimensional linear regression setting with two tasks under combinations of covariate and model shifts.
We study a hard transfer estimation procedure in the high-dimensional limit.
This procedure is rooted in the classical literature \cite{C97}, whose theoretical properties are not well-understood.
We study functions involving two sample covariance matrices with arbitrarily different population covariances.
We estimate the asymptotic limits of the bias and variance of hard parameter sharing and illustrate our estimates in a random-effects model.
We achieve these results using tools from modern random matrix theory and free probability theory \cite{bai2009spectral,erdos2017dynamical,nica2006lectures}.
In particular, building on recent developments in random matrix theory \cite{BES_free1,isotropic,Anisotropic}, we derive precise high-dimensional asymptotics involving combinations of two covariance matrices, along with {\itshape almost sharp convergence rates} to the limits.

\subsection{Problem Setup}\label{sec_mot}

\begin{figure}
    \centering
    \includegraphics[width=0.75\textwidth]{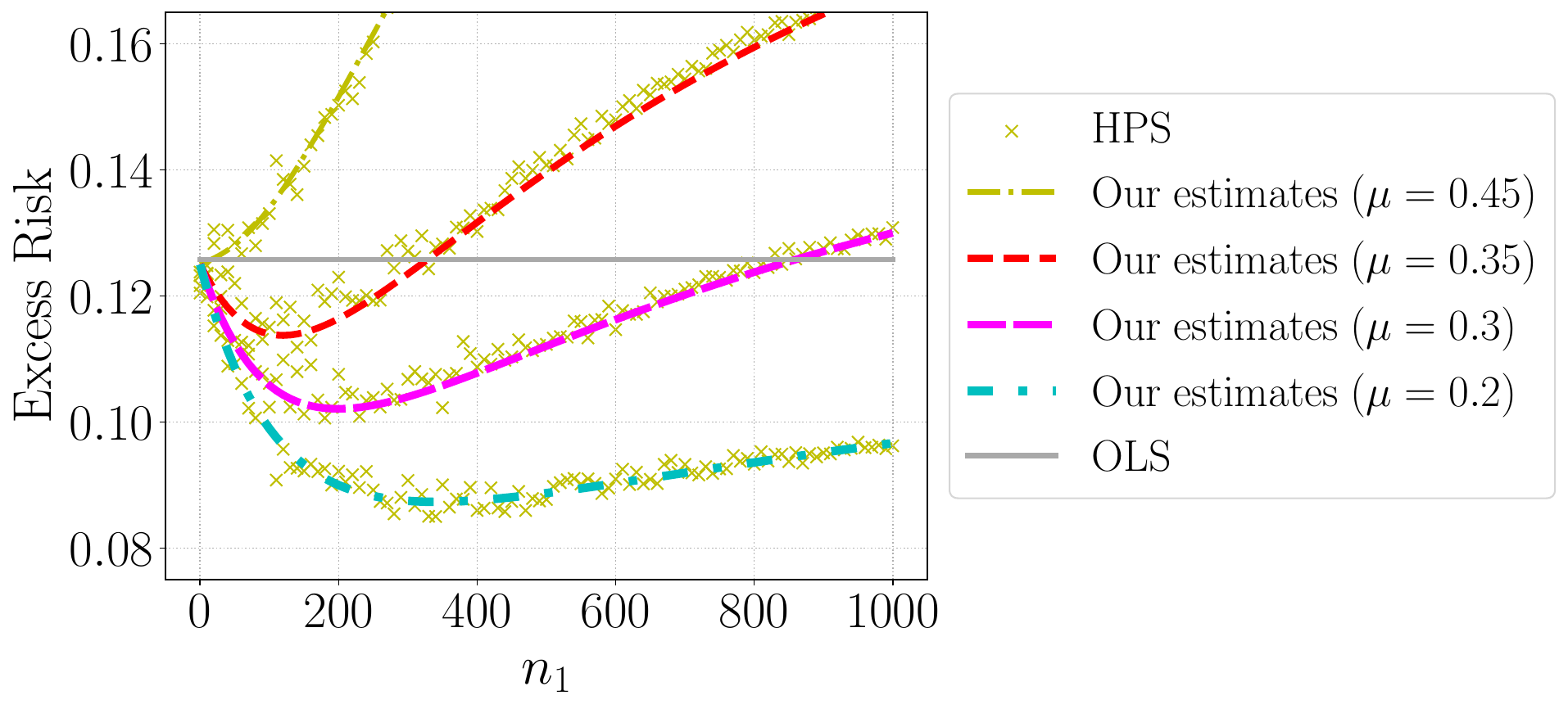}
    \caption{We illustrate the performance of HPS vs. OLS (ordinary least squares) for different levels of model shifts, indicated by $\mu \approx \bignorm{\beta^{(1)} - \beta^{(2)}} / \sqrt{2}$. We consider two linear regression tasks and vary the sample size of the source task, denoted as $n_1$ and $\mu$, while holding the target task fixed.
    The gray line refers to the performance of OLS.
    Any point above the gray line represents negative transfer, while any point below represents positive transfer.
     In this simulation, we set $p = 100$, $n_2 = 300$, and $\sigma^2 = \frac 1 4$.
    We sample the covariates $X$ from a $p$-dimensional isotropic Gaussian and
    plot the excess risk of HPS (see equation \eqref{HPS_loss}).
    For further details about the setup, see Proposition \ref{claim_model_shift}, Section \ref{sec_sizeratio}.}
    \label{fig_motivation}
\end{figure}

We first introduce our problem setup.
For either task $i = 1$ or task $i = 2$, let $x_1^{(i)}, x_2^{(i)}, \dots, x_{n_i}^{(i)}$ denote the feature covariates and %
$y_1^{(i)}, y_2^{(i)}, \dots, y_{n_i}^{(i)}$ denote the prediction labels; Recall that $n_1$ and $n_2$ refer to the number of samples from task one and task two, respectively.
We assume a linear model specified by an unknown parameter vector $\beta^{(i)} \in \real^p$:
\begin{align}
    y_j^{(i)} = \big(x_j^{(i)}\big)^{\top}{\beta^{(i)}} + \varepsilon_j^{(i)}, ~ \text{ for any } \  j = 1,\dots,n_i, \label{eq_linear}
\end{align}
where $\varepsilon^{(i)}_j\in\real$ denotes a random noise variable with mean zero and variance $\sigma^2$.
We refer to the first task as the \emph{source} task and the second task as the \emph{target} task for ease of presentation.
We assume $n_2 > p$ while $n_1$ can be either less than or greater than $p$.

Following the existing random matrix theory literature, we study random designs where the covariates are given by
\[ x_j^{(i)} = \big(\Sigma^{(i)}\big)^{\frac 1 2} z_{j}^{(i)}, \text{ for } i\in\{1,2\} \text{ and for any } j = 1,\dots,n_i, \]
where $z_j^{(i)}$ consists of independent and identically distributed entries of mean zero and variance one, and $\Sigma^{(1)} \in \real^{p \times p}$ and $\Sigma^{(2)} \in \real^{p \times p}$ denote the population covariance matrices.
Let $X^{(i)}=(x_1^{(i)},\ldots,x_{n_i}^{(i)})^\top \in\real^{n_i \times p}$ denote a matrix that corresponds to all the covariates of task $i$, for $i \in \{1, 2\}$.
Let $Y^{(1)}\in \R^{n_1}$, $Y^{(2)}\in \R^{n_2}$ be the vectors of labels.
Then, we minimize the following objective, parameterized by a shared variable $B \in \real^{p}$:
\begin{align}\label{eq_hps}
    \ell(B) = \bignorm{X^{(1)} B - Y^{(1)}}^2 + \bignorm{X^{(2)} B - Y^{(2)}}^2.
\end{align}
Notice that $B$ provides a shared feature vector for both tasks.
The results described in the next section will permit straightforward extensions to variants of objective \eqref{eq_hps}, including adding weights to each task and a ridge penalty.
Thus, we will describe our main results for unweighted and unregularized objectives without losing generality.
One can see that when $n_2 > p$, there is a unique minimizer of $f$. We denote this estimator for the target task as $\hat{\beta}_2^{\MTL}$, which is also called {hard parameter sharing} (HPS) or hard transfer in the multi-task learning literature \citep{C97}.
We will analyze the bias and variance of the HPS estimator and connect the analysis to transfer effects.

Our motivation for studying this simple HPS estimator is that it provides a simple setting in which both positive and negative transfers can happen.
In Figure \ref{fig_motivation}, we illustrate that if the model shift is high (as indicated by $\mu = 0.45$), HPS always performs worse than OLS (ordinary least squares), regardless of the value of $n_1$.
If the model shift is moderate (as indicated by $\mu = 0.2, 0.3, 0.35$), the transfer effect is positive within a small range of $n_1$ until reaching a threshold and turns negative afterward.
In summary, the HPS estimator provides an ideal scenario for understanding the effects of transfer learning under different types of distribution shifts between the source task and the target task.

\subsection{Summary of Our Results}

We give a summary of our results, which follow the assumptions of a proportional limit setting, with both $n_1$ and $n_2$ increasing to infinity in proportion as $p$ goes to infinity. 
Our goal is to find the asymptotic limit of the bias and variance of the HPS estimator applied to the target task ($i = 2$). 
In particular, deriving the asymptotic limit of the bias and variance of the HPS estimator requires studying several functions involving two independent sample covariance matrices with different population covariance matrices. 
Below, we provide a simplified exposition and elaborate on the results later in Section \ref{sec_HPS}.

First, the variance of the HPS estimator is equal to 
{\begin{align}\label{eq_var_simp}
	\bigtr{\Sigma^{(2)} \Big({X^{(1)}}^{\top} X^{(1)} + {X^{(2)}}^{\top} X^{(2)}\Big)^{-1}}.
\end{align}}
The bias of the HPS estimator is more involved and is deferred until Lemma \ref{lem_HPS_loss}. 
The bias and variance both involve the inverse of the sum of the sample covariance matrices of both tasks, which exhibit a covariate shift between them.

When $n_2 > p$, the OLS estimator exists, and it is well-known that the limit of its excess risk 
is equal to ${ \frac {\sigma^2 p} {n_2 - p}}$ \cite{bai2009spectral}.
To the best of our knowledge, the asymptotic limit of equation \eqref{eq_var_simp} (when $n_1/p \rightarrow c_1, n_2 /p \rightarrow c_2$ as $p \rightarrow \infty$) has not been identified in the literature. Our paper takes the first step to fill the gap.
We summarize our main results as follows:
\begin{itemize}
\item First, we consider the covariate shift setting, where $\Sigma^{(1)}$ and $\Sigma^{(2)}$ are arbitrary and $\beta^{(1)} = \beta^{(2)}$.
In Theorem \ref{thm_main_RMT}, we describe the asymptotic limit of the variance formula as a function of the singular values of the covariate shift matrix. 
This result generalizes classical results on the limit of the trace of the inverse of one sample covariance matrix to two covariance matrices under covariate shifts.

\item Second, we consider the model shift setting, where $\beta^{(1)}$ and $\beta^{(2)}$ are arbitrary and $\Sigma^{(1)} = \Sigma^{(2)}$.
In Theorem \ref{cor_MTL_loss}, we describe the asymptotic limit of the bias and variance of HPS as a function of the Euclidean distance between $\beta^{(1)}$ and $\beta^{(2)}$, and the sample sizes.
We then illustrate the results in a random-effects model. Motivated by the phase transition, we show that \emph{with a simple adjustment, HPS can match the minimax lower bound} (see Theorem \ref{prop_lb}).

\item Third, we extend the above results in two aspects:
i) When both covariate and model shifts are present;
ii) When there are multiple tasks in the input.
The detailed results, along with supporting simulations, are provided in Section \ref{sec3_combined}.
In particular, we demonstrate that the findings regarding complementary covariate shifts and phase transitions in transfer learning are applicable to more general settings.
\end{itemize}

Finally, we present a case study comparing the statistical behavior of a soft transfer estimator with HPS.
The soft parameter sharing (SPS) estimator (which is also called soft transfer in \citet{dhifallah2021phase}) is defined by minimizing the following for $\lambda>0$:
\begin{align}\label{target-SPS}
    \ell(B, z) = \bignorm{X^{(1)}(B + z) - Y^{(1)}}^2 + \bignorm{X^{(2)} B - Y^{(2)}}^2 +  {\lambda}  \bignorm{z}^2.
\end{align}
Let $\hat B,\hat z$ denote the minimizer, and the SPS estimator for the target task is defined as $\hat{\beta}_2^{\SPS}(\lambda) = \hat B$.
The SPS estimator makes sense in the model shift setting where $\beta^{(1)}$ is far from $\beta^{(2)}$.
We will provide the bias and variance of SPS under deterministic designs.
Then, we provide an example to contrast the behavior of SPS with that of HPS.
We find that the regularization effect of SPS is similar to downsizing the source task in HPS when the model shift is large.
 
We illustrate the above results in a random-effects model, where $\beta^{(i)}$ is equal to a shared model vector across all $i$ plus an independent random effect for each task.
The random-effects model provides a natural way to measure the heterogeneity of each task's underlying linear model and has been used to study distributed ridge regression \cite{dobriban2020wonder}.

First, we demonstrate that covariate shifts can either enhance or hinder transfer learning performance compared to single-task learning. 
We describe an example in Proposition \ref{claim_dichotomy}, showing that when $n_1 < n_2$, transferring from any covariate-shifted data source (i.e.,  $\Sigma^{(1)} \neq \Sigma^{(2)}$) achieves a lower excess risk of HPS than transferring from the data source with $\Sigma^{(1)} = \Sigma^{(2)}$.
On the other hand, when $n_1 > n_2$, transferring from any covariate-shifted data source always incurs a higher excess risk of HPS than transferring from the data source with $\Sigma^{(1)} = \Sigma^{(2)}$.

Second, we identify three transfer regimes in the random-effects model. 
Let $\mu$ denote the Euclidean distance between $\beta^{(1)}$ and $\beta^{(2)}$. We show the following dichotomy between $\mu$ and $n_1, n_2$:
\begin{itemize}%
    \item When ${\mu^2}{} \frac{\bigtr{\Sigma^{(1)}}} {p} \le \frac{\sigma^2 p}{2(n_2 - p)}$, the transfer effect is positive for any value of $n_1 > 0$: HPS always performs better than OLS.
	\item When $\frac{\sigma^2 p}{2(n_2 - p)} < {\mu^2} \frac{\bigtr{\Sigma^{(1)}}}{p} < \frac{\sigma^2 n_2}{2(n_2 - p)}$, there exists a deterministic constant $\rho > 0$ such that the transfer effect is positive if and only if $\frac {n_1} p < \rho$.
    This corresponds to the crossing point between HPS and OLS in Figure \ref{fig_motivation}.
	\item When ${\mu^2}\frac{\bigtr{\Sigma^{(1)}}}{p} \ge \frac{\sigma^2 n_2}{2(n_2 - p)}$, the transfer effect is negative for any value of $n_1 > 0$: HPS always performs worse than OLS.
\end{itemize}

\textbf{Summary of contributions.} This paper presents a rigorous study of the transfer effects of HPS estimators under various types of data set shifts between two tasks.
We contribute to the transfer learning literature by using random matrix theory to analyze negative transfer of HPS in the \emph{high-dimensional} asymptotic setting under covariate shifts.
In particular, we derive the asymptotic limit of the HPS estimator under covariate shift (Theorem \ref{thm_main_RMT}), which generalizes a classical result in the random matrix theory literature.
For the model shift setting, our result (Theorem \ref{cor_MTL_loss}) can explain a phase transition (Proposition \ref{claim_model_shift}) that has been observed in empirical studies of transfer learning.
We extend these results to show that the insights hold more generally beyond the specific examples.
Simulation studies demonstrate the accuracy of the derived asymptotic in finite dimensions.
We highlight numerous technically challenging open questions, and we hope our work can inspire further studies on applying random matrix theory to understand transfer learning.

\subsection{Related Work}\label{sec_related}

Our work expands the existing statistical learning literature by contributing a random matrix theory perspective to quantify transfer effects in high-dimensional linear regression with two tasks.
Our perspective differs from the existing literature, which uses uniform convergence arguments \cite{B00,BS03,M06} and generalization bounds \cite{crammer2008learning,BBCKP10} to quantify the performance of transfer learning. 
The benefit of using precise asymptotics to study transfer learning performance compared to generalization bounds is that they can be used to compare the performances of different estimators, and their predictions remain accurate in finite dimensions.
This is crucial for explaining phenomena such as phase transitions, as highlighted in the simulation.

There has been a long line of work that tries to formulate the notion of information or knowledge transfer in the statistical learning literature.
In the context of the many Normal means models, lasso regularization can be used to leverage the shared support of multiple tasks \cite{kolar2011union}, and similar insights can also be fleshed out in multi-task learning with marginal regression \cite{kolar2012marginal}.
The work of \citet{dhifallah2021phase} provides a precise analysis of both the hard transfer and soft transfer approaches in a student-teacher framework.
They provide the convergence of the training loss in the asymptotic setting where $n_1/p, n_2/p$ converge to a constant.
Then, they provide illustrative examples of the precise analysis in regression and classification models to explain the phase transition in transfer learning.
One of the key differences between our results and their work is that we address covariate shifts in the high-dimensional setting, while also providing sharp convergence rates. This requires the use of random matrix theory techniques from more recent literature. 
\citet{dar2022double} derive generalization error rates for the two-task transfer learning in the important \emph{overparameterized} regime (see also \citet{ju2023generalization} for further studies in the overparameterized transfer learning regime).
They derive a generalization error formula for the minimum $\ell_2$-norm solution of the target task, which can reproduce a double descent behavior and explain the negative transfer phenomenon \cite{dar2024common}.
In this vein, the high-dimensional asymptotics we have provided under various distribution shifts may be viewed as complementary to this line of work, contributing to the mathematical formalization of transfer learning.

In the context of linear regression tasks, recent studies have investigated the optimal estimators for transfer learning under distribution shifts.
\citet{lei2021near} give an in-depth study of minimax estimators for transfer learning from two linear regression tasks.
They derive minimax optimal estimators that can be formulated as solving a convex program, and the approach can be applied under both covariate and model shifts.
In the setting of multiple linear regression, \citet{li2020transfer} design an adaptive estimation procedure in a setting where many tasks with Gaussian covariates are present. 
\citet{duan2023adaptive} introduce an adaptive estimation procedure when many tasks are present in multi-task learning.
There is another related line of work on multi-view regression, where the input variable (which is a real vector) can be partitioned into two different views, and it is assumed that either view of the input is sufficient to make accurate predictions \cite{kakade2007multi}.
\citet{zhao2024trans} design a two-stage procedure for matrix estimation that uses a multi-task learning objective to capture shared and unique features, and then refines them to adjust for structural differences between the target and source matrices.
Our paper complements these existing works, as we focus on the proportional limit setting to examine the performance of commonly used estimators and provide insights.

When covariates are sampled from Gaussian distributions, the precise asymptotic limit of the inverse of a sample covariance matrix can be derived directly from the properties of the Wishart distribution. 
In the high-dimensional setting, the eigenvalues of a Wishart matrix satisfy the well-known Marchenko–Pastur (MP) law, whose Stieltjes transform characterizes the variance limit.
Furthermore, it is well known that the MP law holds universally regardless of the underlying data distribution of the covariates (see, e.g., \citet{bai2009spectral}). \citet{isotropic} obtain a sharp convergence rate of the empirical spectral distribution (ESD) to the MP law for sample covariance matrices with isotropic population covariances. \citet{Anisotropic} later extend this result to sample covariance matrices with arbitrary population covariances. These results are proved by establishing the optimal convergence estimates of the Stieltjes transforms of sample covariance matrices, also known as \emph{local laws} in the random matrix theory literature. We refer interested readers to \citet{erdos2017dynamical} and the references therein for a detailed review.
One technical contribution of this work is to extend these techniques to the two-task setting and prove an almost sharp local law for the sum of two sample covariance matrices with arbitrary covariate shifts. This local law allows us to derive the precise variance limit depending on the singular values of the covariate shift matrix.
Finally, recent work has also studied the statistical behavior of interpolators in the high-dimensional setting \citep{sur2019modern}, which would be another promising setting for exploring statistical transfer learning.

The asymptotic limit of equation \eqref{eq_var_simp} may also be derived using free probability theory. %
However, this approach is not fully justified when the covariates are sampled from non-Gaussian distributions with non-diagonal covariate shift matrices.
Furthermore, our result provides almost sharp convergence rates to the asymptotic limit, while it is unclear how to obtain such rates using free probability techniques.
The bias term involves asymmetric matrices in terms of two sample covariance matrices, whose analysis is technically involved.
Our techniques are inspired by free additions of random matrices \citep{nica2006lectures} and recent results \citep{BES_free1}. %
In particular, we provide the first precise bias limit in the model shift setting, assuming that $\Sigma^{(1)} = \Sigma^{(2)}$ or that $\Sigma^{(2)}$ is isotropic, and the covariates are sampled from Gaussian distributions.
Showing the asymptotic bias limit under arbitrary covariate and model shifts is an interesting open problem for future work (See Remarks \ref{remark_bias} and \ref{remark_open} for further discussions).

\subsection{Organization}

The rest of this paper is organized as follows.
In Section \ref{sec_HPS}, we state the data model and its underlying assumption.
Then, we connect the transfer effect of the HPS estimator with its bias-variance decomposition.
In Section \ref{sec_main}, we present the high-dimensional asymptotic limits of the bias and variance of HPS under covariance shifts.
In Section \ref{sec_sizeratio}, we characterize the high-dimensional asymptotic limits of the bias and variance of HPS under model shifts.
Section \ref{sec3_combined} extends our main findings from the above two settings to more general settings.
Section \ref{subsec_SPS} provides a preliminary study of the bias and variance of the SPS estimator.
Finally, in Section \ref{sec_conclude}, we state the conclusion of our paper.
Appendix \ref{app_tool} to Appendix \ref{sec:pf_decomp_SPS} provides detailed proofs of our results.

\section{Preliminaries}\label{sec_HPS}

This section describes the data model that we will work with, along with the underlying assumptions.
Then, we describe the bias and variance of the HPS estimator.
Finally, we connect the bias-variance decomposition to the transfer analysis in the rest of the paper.

\subsection{Data Model and Assumptions}\label{sec_data}

Recall that we have two tasks.
For $i = 1, 2$, $X^{(i)} \in \real^{n_i \times p}$ corresponds to task $i$'s covariates and
$Y^{(i)} \in \real^{n_i}$ corresponds to their labels. Moreover, let $\varepsilon^{(i)} \in \real^{n_i}$ be the vector notation corresponding to the additive noise of data set $i$. Then, equation \eqref{eq_linear} can be reformulated as
$Y^{(i)}= X^{(i)}\beta^{(i)} + \varepsilon^{(i)}.$
Assume that $\beta^{(1)}$ and $\beta^{(2)}$ are two arbitrary (deterministic or random) vectors that are independent of $X^{(i)}$ and $\varepsilon^{(i)}$.
Throughout the paper, we make the following assumptions on  $X^{(i)}$ and  $\varepsilon^{(i)}$, which are standard in the random matrix theory literature  (see, e.g., \citet{tulino2004random,bai2009spectral}).

First, the row vectors of $X^{(i)}$ are i.i.d.~centered random vectors with $p\times p$ population covariance matrix $\Sigma^{(i)}$ and an $n_i\times p$ random matrix $Z^{(i)}=[Z^{(i)}_{jk}]$ with independent entries of zero mean and unit variance:
\begin{align}\label{XofZ}
    X^{(i)} = Z^{(i)} (\Sigma^{(i)})^{1/2} \in \real^{n_i\times p}.
\end{align}
Let $\tau>0$ be a small constant.
Suppose the $\varphi$-th moment of each entry $Z^{(i)}_{jk}$ is bounded from above by $1/\tau$,   for a constant $\varphi>4$:
\begin{align} \label{conditionA2}
\ex{\big\vert Z^{(i)}_{jk} \big\vert^\varphi}  \le {\tau}^{-1}.
\end{align}
The eigenvalues of $\Sigma^{(i)}$, denoted as $\si^{(i)}_1,\cdots,\si^{(i)}_p$, are all bounded between $\tau$ and $\tau^{-1}$:
\begin{equation}\label{assm3}
   \tau \le  \si^{(i)}_p \le\cdots\le \si^{(i)}_2 \le \si^{(i)}_1 \le \tau^{-1}.
\end{equation}
Second,  $\varepsilon^{(i)} \in \real^{n_i}$ is a random vector with independent entries having mean zero, variance $\sigma^2$, and bounded moments up to any order, i.e., for any fixed $k\in \N$, there exists a constant $C_k>0$ such that
\begin{align}\label{eq_highmoments}
\ex{\big| \ve^{(i)}_{j} \big|^k} \le C_k.
\end{align}
Third, the sample sizes are comparable to the dimension $p$.
Denote by $\rho_1 = n_1 / p$ and $\rho_2 = n_2 / p$.
Assume
\begin{align}\label{assm2}
	0 < \rho_1 \le p^{\tau^{-1}},\quad  1+\tau \le \rho_2 \le p^{\tau^{-1}},\quad 0 < {\rho_1}/{\rho_2}\le \tau^{-1}.
\end{align}
The condition $\rho_2\ge 1 + \tau$ ensures that the target task's sample covariance matrix is of full rank with high probability.
The upper bound $\rho_i \le p^{\tau^{-1}}$ is a mild condition; otherwise, standard concentration results, such as the central limit theorem, already give accurate estimates in the linear model.
The condition ${\rho_1}/{\rho_2}\le \tau^{-1}$ ensures the sample size imbalance between the two tasks is bounded by a factor that does not grow with $p$.
In the transfer learning setting, one can think of $\rho_1$ and $\rho_2$ as fixed, positive constants that do not grow or diminish with $p$ to model practical applications.
To summarize, the underlying assumptions of the data model are as follows.
\begin{assumption}\label{assm_big1}
Let $\tau > 0$ be a small constant.
Suppose $X^{(1)}$, $X^{(2)}$, $\varepsilon^{(1)}$, and $\varepsilon^{(2)}$ are mutually independent. Moreover, suppose the following holds for $i=1$ and $i=2$: %
\begin{enumerate}%
\item  $X^{(i)}$ takes the form of equation \eqref{XofZ}, where $Z^{(i)}$ is a random matrix with i.i.d.~entries having zero mean, unit variance, and bounded moments as in equation \eqref{conditionA2}, and $\Sig^{(i)}$ is a deterministic positive definite symmetric matrix satisfying equation \eqref{assm3}.

\item $\varepsilon^{(i)} \in \real^{n}$ is a random vector consisting of i.i.d.~entries of zero mean, variance $\sigma^2$, and bounded moments as in equation \eqref{eq_highmoments}.

\item $\rho_{i}$ satisfies equation \eqref{assm2}. %
\end{enumerate}
\end{assumption}

\subsection{Bias and Variance}\label{sec_risk}

The risk of $\hat{\beta}_i^{\MTL}$ over an unseen sample $(x,y)$ of the target task $i$ is given by (under the mean squared loss):
$$\exarg{x, y}{\left\|y-x^\top \hat{\beta}_i^{\MTL} \right\|^2}= \left\|{\Sigma^{(i)}}^{\frac 1 2} \left(\hat{\beta}_i^{\MTL} - \beta^{(i)}\right)\right\|^2 + \sigma^2.$$
The excess risk is the difference between the above risk and the expected risk of the population risk optimizer:
\begin{align}\label{HPS_loss}
    L(\hat{\beta}_i^{\MTL}) := \left\| {\Sigma^{(i)}}^{\frac 1 2} \left(\hat{\beta}_i^{\MTL} - \beta^{(i)}\right)\right\|^2.
\end{align}
We shall treat task two as the target task and treat task one as the source task.  
The single-task learning estimator for the target task, which is the OLS estimator denoted by $\hat{\beta}_2^{\STL} $, is well-defined given Assumption \ref{assm_big1}, because the covariance matrix of the target task is invertible with high probability under the assumption.

Next, we present a bias-variance decomposition of the excess risk of HPS.\footnote{One way is to estimate the asymptotic limit of equation \eqref{HPS_loss} directly. However, this would involve an estimation that needs to take into account the randomness of the covariates $X$ along with the noise $\varepsilon$ simultaneously. Instead, given that $X$ and $\varepsilon$ are independent, our approach is to first derive the bias-variance decomposition of the excess risk, which is equal to the expectation of $L(\hat \beta_i^{\MTL})$ over the randomness of $\varepsilon$, and then compare the bias and variance of two different estimators.}
Denote the sum of both tasks' sample covariance matrix as:
\begin{align}\label{Sigma_a}
    \hat{\Sigma} = {X^{(1)}}^{\top} X^{(1)} + {X^{(2)}}^{\top} X^{(2)}.
\end{align}
The next result provides the bias and variance of the HPS estimator.

\begin{lemma}\label{lem_HPS_loss}
     Under Assumption \ref{assm_big1}, for any small constant $\e>0$, with high probability over the randomness of the training samples $X^{(1)}, Y^{(1)}, X^{(2)}, Y^{(2)}$, the following estimates hold: 
    \begin{align}
        L(\hat{\beta}_2^{\MTL}) =& \left(1+\OO(p^{-\frac 1 2 + c})\right)  \left(L_{\bias} + L_{\vari}\right)     , \label{L_HPS_simple}
    \end{align}
    where the bias and variance formulas are defined as %
    \begin{align}
        L_{\bias}  &= \left\| {\Sigma^{(2)}}^{\frac 1 2}\hat \Sigma^{-1} \bigbrace{{X^{(1)}}^\top X^{(1)}} \big(\beta^{(1)}- \beta^{(2)}\big) \right\|^2,  \label{Lbias} \\
        L_{\vari}   &= \sigma^2  \tr\big[{\Sigma^{(2)}\hat \Sigma^{-1}  }\big].  \label{Lvar}
    \end{align}
\end{lemma}

Hereafter, an event $\Xi$ is said to hold \emph{with high probability} (w.h.p.) if $\mathbb P(\Xi)\to 1$ as $p\to \infty$. Moreover, the big-$\OO$ notation $A=\OO(B)$ means that $|A|\le C|B|$ for a constant $C>0$ depending on the model parameters in Section \ref{sec_data} (i.e., $\tau$, $\varphi$ and $C_k$'s), but not on $p$, $n_1$ and $n_2$. Hence, the event that equation \eqref{L_HPS_simple} holds w.h.p. can be equivalently stated as
\[ \lim_{p\to \infty}\P\left[ \left|L(\hat{\beta}_2^{\MTL}) - L_{\bias} - L_{\vari}\right| \le  Cp^{-\frac 1 2 + c}(L_{\bias} + L_{\vari})\right] = 1 \]
for a constant $C>0$. The proof of Lemma \ref{lem_HPS_loss} can be found in Appendix  \ref{app_firstpf}.

\subsection{Using Bias and Variance to Analyze Transfer Effects}

The bias and variance decomposition above allows us to reason about transfer effects by comparing the bias and variance of HPS with that of OLS.
Notice that the bias of OLS is zero since $n_2 \ge (1 + \tau) p$ under Assumption \ref{assm_big1}.
Hence, the bias of HPS is always larger than that of OLS.
On the other hand, the variance of HPS is always lower than that of OLS.
Let $X^{(1)}=0$ in equation \eqref{Lvar}. 
By Woodbury matrix identity, the variance of OLS is always higher than formula \eqref{Lvar}:
\begin{align*}
    \bigtr{\Sigma^{(2)} \big({X^{(2)}}^{\top} X^{(2)}\big)^{-1}}
    \ge
    \bigtr{\Sigma^{(2)} {\hat{\Sigma}}^{-1}}. %
\end{align*}
Thus, the transfer effect of HPS is determined by the bias-variance decomposition: the bias always increases while the variance always decreases. 
Whether or not the transfer effect of HPS is positive depends on which effect dominates.

The above discussion highlights the need for precise bias-variance estimates to determine transfer effects. However, finding the exact limits is challenging because of data set shifts. Moreover, these shifts manifest in various forms.
To make progress on this important yet challenging problem, we divide our study into various combinations of covariate and model shifts.
Section \ref{sec_main} is devoted to the precise estimate of $L_{\vari}$ for the HPS estimator under an arbitrary covariate shift but without model shift.
Section \ref{sec_sizeratio} studies the effect of model shifts: Section \ref{subsect_modelshiftonly} gives precise estimates of $L_{\bias}$ and $L_{\vari}$ for the HPS estimator. 
Then, Section \ref{sec_adjusted} discusses the adjusted HPS procedure. Finally, Section \ref{sec_minimax} states the minimax lower bound. Section \ref{sec3_combined} extends our main findings from the previous two sections to settings with both covariate and model shifts, as well as multiple tasks.

\section{Covariate Shifts}\label{sec_main}

This section presents a precise estimate of the variance $L_{\vari}$ for the HPS estimator in equation \eqref{Lvar} with an almost sharp convergence rate under two different population covariance matrices $\Sigma^{(1)}, \Sigma^{(2)}$.
Then, we will provide two examples to illustrate the asymptotic limit, which provide varying levels of covariate shifts without model shifts.
Our simulation study shows that the asymptotic estimates are quite accurate, even when the feature dimension $p$ ranges from $50$ to $100$.
Finally, we provide numerical comparisons of the excess risk of HPS with several other transfer learning estimators.

\subsection{Main Results}\label{sec3_cov}

Suppose the two tasks satisfy the same linear model ($\beta^{(1)}=\beta^{(2)}$) but have different population covariance matrices ($\Sigma^{(1)}\ne \Sigma^{(2)}$).
Recall that the matrix $\hat{\Sigma}$ in equation \eqref{Sigma_a} is a sum of two sample covariance matrices.
Thus, the expectation of \smash{$\hat{\Sigma}$} is equal to a mixture of $\Sigma^{(1)}$ and $\Sigma^{(2)}$, with mixing proportions determined by the sample sizes $n_1$ and $n_2$.
Intuitively, the spectrum of $\hat{\Sigma}^{-1}$ not only depends on $n_1, n_2$, but also depends on how well aligned $\Sigma^{(1)}$ and $\Sigma^{(2)}$ are.
To capture this alignment, we introduce the covariate shift matrix
$M \define {\Sigma^{(1)}}^{\frac 1 2}{\Sigma^{(2)}}^{- \frac 1 2}$.
Let $\lambda_1 \ge \lambda_2 \ge \dots\ge \lambda_p $ be the singular values of $M$ in descending order.
Our first main result is the following theorem on the variance limit, which characterizes the exact dependence of $L_{\vari}$ on the singular values of $M$ and the sample sizes $n_1$ and $n_2$.

\begin{theorem}[Precise estimates under covariate shifts]\label{thm_main_RMT}
    Under Assumption \ref{assm_big1}, for any small constant $c>0$, with high probability over the randomness of $(X^{(1)}, X^{(2)})$, the following holds: 
	\begin{align}\label{lem_cov_shift_eq}
		\bigabs{L_{\vari}- \frac{\sigma^2}{n_1+n_2}\bigtr{  \Bigbrace{\alpha_1  M^\top M + \alpha_2\id_{p\times p}}^{-1}  }}
		\le \sigma^2 p^{\frac 1 2}(n_1 + n_2)^{\frac 2 {\varphi} -  \frac 3 2 + c},
	\end{align}
    where $\alpha_1$ and $\alpha_2$ are the unique positive solutions of the following system of equations:
	\begin{align}
		\alpha_1 + \alpha_2 = 1- \frac{p}{n_1 + n_2}, \quad
		\alpha_1 + \frac1{n_1 + n_2}   \sum_{i=1}^p \frac{\lambda_i^2 \alpha_1}{\lambda_i^2 \alpha_1 + \alpha_2}  = \frac{n_1}{n_1 + n_2}. \label{eq_a12extra}
	\end{align}
\end{theorem}

Recalling that $\varphi > 4$, for a small enough constant $c$, equation \eqref{lem_cov_shift_eq} characterizes the limit of $L_{var}$ with an error term that is smaller than the deterministic leading term by a factor $\oo({p}^{-1/2})$.

Theorem \ref{thm_main_RMT} generalizes a classical result in multivariate statistics to the sum of two independent sample covariance matrices with arbitrary covariate shift: with high probability over the randomness of $X^{(2)}$,
\begin{align}
    \bigtr{\Sigma^{(2)} \bigbrace{{X^{(2)}}^\top X^{(2)}}^{-1} }
	= \frac{p}{n_2 - p} +  \OO\left( p^{\frac 1 2}n^{\frac{2}{\varphi}-\frac{3}{2}+c}\right). \label{fact_tr}
\end{align}
To see this, note that \eqref{fact_tr} is a special case of Theorem \ref{thm_main_RMT} with $n_1 = 0$.
For this case, equation \eqref{eq_a12extra} implies that $\alpha_1 = 0$ and $\alpha_2 = (n_2-p)/{n_2}.$
Therefore,
\begin{align*}
    \frac{1}{n_1 + n_2} \bigtr{\Bigbrace{\alpha_1  M^{\top} M + \alpha_2\id_{p\times p}}^{-1}} = \frac{p}{n_2 - p}. %
\end{align*}
The result \eqref{fact_tr} has a rich history in the literature of random matrix theory.
When $Z^{(2)}$ is a Gaussian random matrix, the limit (without a convergence rate) follows from properties of the inverse Wishart distribution.
Otherwise, one can use the Stieltjes transform method to derive the estimate (Lemma 3.11, \citet{bai2009spectral}).
The estimate \eqref{fact_tr} with sharp convergence rates can be found in Theorem 2.4 \cite{isotropic}.

Theorem \ref{thm_main_RMT} follows from a sharp \emph{local law} on the resolvent of $\hat \Sigma$, which we will establish based on recently developed techniques in the random matrix theory literature \citep{Anisotropic}.
We refer the reader to Appendix \ref{appendix RMT} for more details.

As a remark, in the covariate shift setting, since $\beta^{(1)} = \beta^{(2)}$, we only need to study hard transfer, which does not add any model adjustment to the source task.

\subsection{Illustrative Examples}

The formulas in the above result are complex.
To interpret their scaling, we provide two illustrative examples of Theorem \ref{thm_main_RMT}, based on which we revisit the effect of covariate shift upon transfer.
Notice that the answer to this question is not always clear in the transfer learning literature due to the intricacy of covariate shifts. Thus, we will leverage the precise asymptotics we have developed to provide a preliminary study of this question. Our goal is to provide several illustrative examples that are not meant to be exhaustive.
Also note that in the case where both tasks share the same linear model ($\beta^{(1)} = \beta^{(2)}$), HPS always incurs a \emph{lower} risk than OLS.
Therefore, the transfer effect is strictly \emph{positive} in the covariate shift setting, and our comparison will be between HPS estimators with different degrees of covariate shift.

In the first example, we show that the impact of covariate shift depends on the ratio between $n_1$ and $n_2$.
In this example, we demonstrate that having greater covariate shifts between source and target tasks can help improve estimation.
Let $p$ be an even integer, and let
    \[ g(M) = \frac{\sigma^2}{n_1 + n_2} \bigtr{(\alpha_1 M^{\top} M + \alpha_2\id_{p\times p})^{-1}}. \]
We compare $g(M)$ for different $M$ matrices within a set $\cS$ such that $\lambda_{p +1 - i}  =  1 /{\lambda_i}$ for $i = 1,2,\dots,p /2$. %
Therefore, every matrix in $\cal S$ is normalized with $\det(M)=1$.\footnote{Note that the covariate shift intensifies when we scale $M$ by a larger scalar $C>1$. Moreover, intuitively, $g(CM)$ generally would decrease as the scalar $C$ increases. Thus, to ensure a fair comparison between scenarios with and without covariate shift, we fix a normalization for the $M$ matrix by requiring that $\det(M)=1$.}
In particular, $M = \id_{p \times p}$ represents two tasks with the same population covariance matrix.
We show that when $n_1 < n_2$, the presence of covariate shifts can be beneficial for HPS.
On the other hand, when $n_1 > n_2$, the presence of covariate shifts can hurt the performance of hard transfer.

\begin{proposition}\label{claim_dichotomy}
    Within the set $M\in\cS$, the following dichotomy regarding $g(\id_{p\times p})$ and $g(M)$ holds:
	i) When $n_1 < n_2$, $g(M) \le g(\id_{p\times p})$ for any $M \in \cS$.
    ii) When $n_1 \ge n_2$, $g(\id_{p\times p}) \le g(M)$ for any $M \in \cS$.
\end{proposition}

\begin{figure}[!t]
	\begin{subfigure}[b]{0.495\textwidth}
		\centering
		\includegraphics[width=6.5cm]{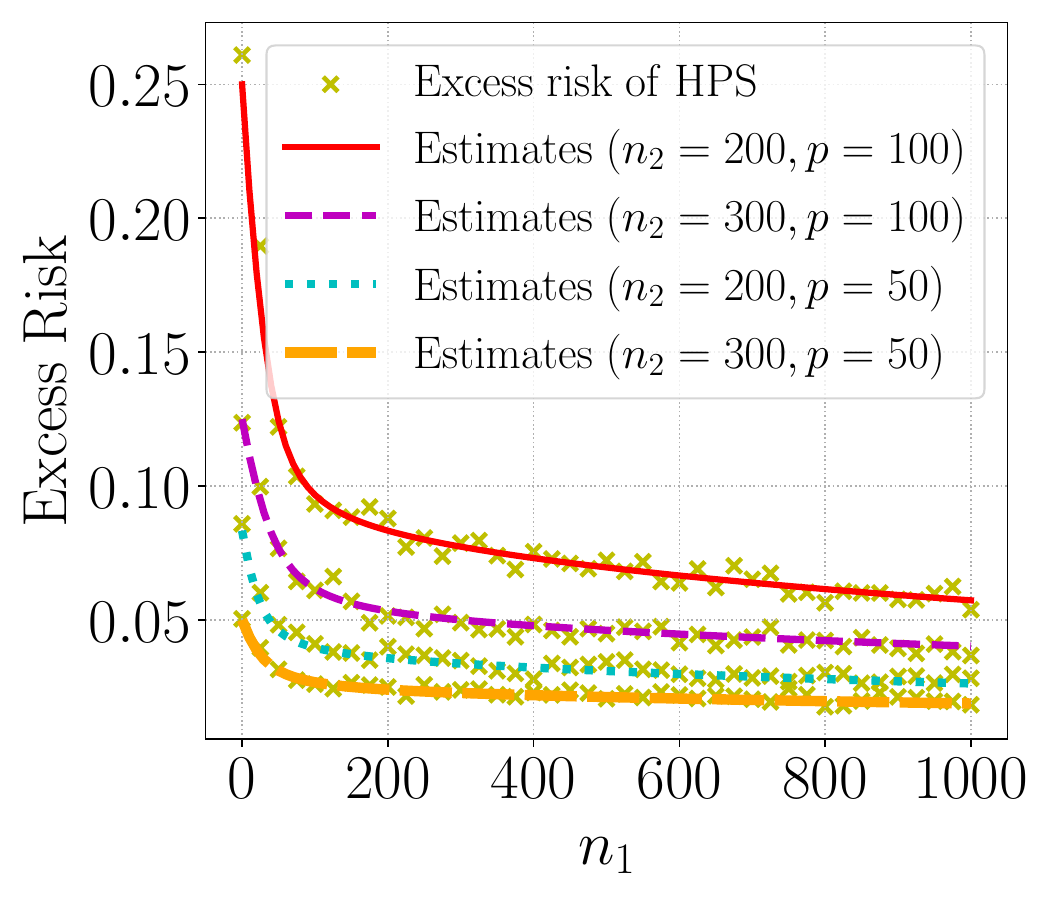} %
		\caption{Varying $n_2$ and $p$}
		\label{fig_sec3_verify_cov}
	\end{subfigure}
	\begin{subfigure}[b]{0.495\textwidth}
		\centering
		\includegraphics[width=6.5cm]{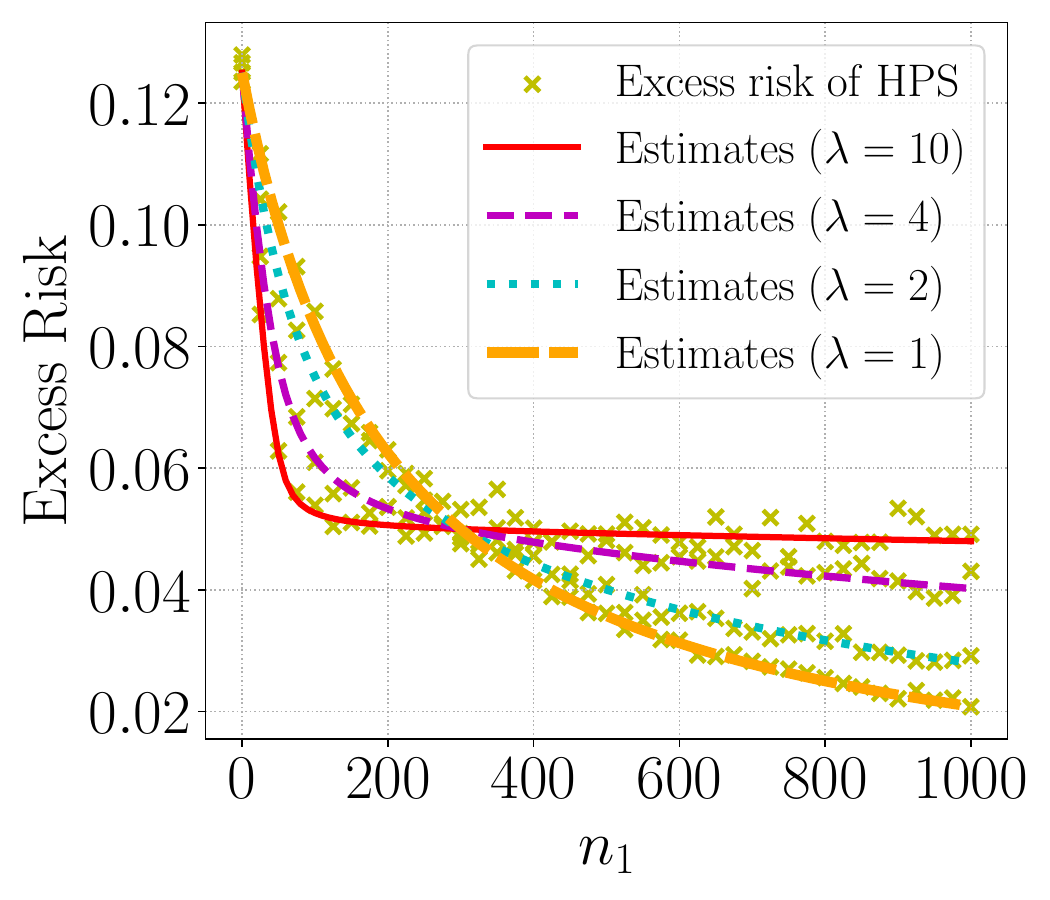} %
		\caption{Varying covariate shift level $\lambda$}
		\label{fig_sec3_covariate}
	\end{subfigure}
	\caption{Illustration of Theorem \ref{thm_main_RMT} under various sample sizes and dimensions.
	Figure \ref{fig_sec3_verify_cov} shows that the variance estimate can accurately fit the excess risk of HPS.
	This simulation fixes $\lambda = 4$ and varies $n_1, n_2, p$.
	Figure \ref{fig_sec3_covariate} shows the dichotomy in Proposition \ref{claim_dichotomy}.
	When $n_1 < n_2$, the lowest risk is achieved by transferring from a covariate-shifted data set.
	When $n_1 \ge n_2$, the lowest risk is achieved by transferring from a data set with $\Sigma^{(1)} = \Sigma^{(2)}$.
	This simulation fixes $p = 100, n_2 = 300$ and varies $n_1, \lambda$. Both simulations use $\sigma = 1/2$.
    The covariates are sampled independently from an isotropic Gaussian distribution.}
	\label{fig_sec31}
\end{figure}

Figure \ref{fig_sec31} illustrates a special case where $\lambda_1 = \cdots = \lambda_{{p /2}} = \lambda > 1$ and $\lambda_{{p /2}+1}  = \cdots = \lambda_{p} =1/ \lambda $.
Thus, higher $\lambda$ corresponds to a worse covariate shift. %
We plot the theoretical estimate using $g(M)$ and the excess risk using equation \eqref{Lvar}.
Our theoretical estimate in Theorem \ref{thm_main_RMT} matches the empirical risk incredibly well.
As a result, we indeed observe the dichotomy in Proposition \ref{claim_dichotomy}.
Furthermore, for larger $\lambda$, the excess risk of HPS decreases more slowly, indicating a worse ``rate of transfer'' from task one.
As a remark, impossibility results for transfer learning under covariate shift have been observed for classification \citep{david2010impossibility}.
Our results are stated in the high-dimensional regression setting.
A related result regarding the effect of covariate shift on transfer has also appeared in the work of \citet{lei2021near}.
They show that the minimax estimator depends on the singular values of the $M$ matrix.

In the second example, we consider the case where $n_1 \ge n_2$.
A typical scenario in transfer learning is that the source data set is much larger than the target data set. 
We show that when $n_1$ is larger than $n_2$ times a sufficiently large constant, $M = \id_{p\times p}$ indeed minimizes $g(M)$ within a bounded set of matrices whose determinants are equal to one.

\begin{proposition}\label{prop_covariate}
	Let $c \in (0, 1)$ be a fixed constant.
	Let $\cS'$ be a set of matrices such that for any $M\in\cS'$:
	(i) $\det(M) = 1$;
	(ii) the eigenvalues $\lambda_1^2,\ldots,\lambda_p^2$ of $M^\top M$ are bounded between $c$ and $1 / c$.
    If $n_1-p\ge c^{-1}n_2$, then $g$ achieves the global minimum at $\id_{p\times p}$, i.e.,
	\begin{align} g(\id_{p\times p}) \le  g(M),~\text{for any}~M \in \cS'. \label{eq_claim_id}
	\end{align}
\end{proposition}

We remark that the condition that the eigenvalues $\lambda_1^2,\ldots,\lambda_p^2$ are bounded would be necessary for equation \eqref{eq_claim_id} to hold.
For example, choose $M=M_x$ with $1 < \lambda_1^2 = x$ and $\lambda_2^2=\cdots=\lambda_{p}^2=x^{-(p-1)^{-1}}$. Let $(\al_1(x),\al_2(x))$ be the solution to the system of equations
\begin{align*}
		\alpha_1 + \alpha_2 = 1- \frac{p}{n_1 + n_2}, 
		\alpha_1 + \frac1{n_1 + n_2}  \Big( \frac{(p-1)x^{-(p-1)^{-1}}\al_1 }{x^{-(p-1)^{-1}}\al_1 +\al_2}+ \frac{x\al_1 }{x\al_1  + \al_2}\Big) = \frac{n_1}{n_1 + n_2}. %
\end{align*}
From this equation, we obtain that $\al_1(x)\to \frac{n_1-1}{n_1+n_2}$ and $\al_1(x)\to \frac{n_2+1-p}{n_1+n_2}$ as $x\to \infty$. 
Thus, 
\begin{align*}
    \lim_{x\to \infty} g(M_{x})&=\lim_{x\to\infty}\frac{\sigma^2}{n_1+n_2}\left(\frac{p-1}{ x^{-(p-1)^{-1}}\al_1(x) + \al_2(x)} + \frac{1}{x\al_1(x)  + \al_2(x)}\right)\\
    &=\frac{\sigma^2(p-1)}{n_2+1-p} \ge \frac{\sigma^2 p }{n_1+n_2-p}=g(\id_{p\times p}).
\end{align*}
Similar arguments, by comparing $g(\id_{p\times p})$ with $g(M)$ for certain $M$ on the boundary of $\cal S'$, show that $(n_1-p)/n_2$ should have a proper lower bound for $g(\id_{p\times p})$ to be the global minimum.
We will use an induction argument to show that $(n_1-p)/n_2\ge c^{-1}$ is a sufficient condition.
The proof of Propositions \ref{claim_dichotomy} and \ref{prop_covariate} can be found in Appendix \ref{sec:claim_dich} and Appendix \ref{appendix RMT0}, respectively.

Note that in both Propositions \ref{claim_dichotomy} and \ref{prop_covariate}, our goal is to illustrate the exact asymptotics of Theorem \ref{thm_main_RMT}, since the formula is intricate. 
Relatedly, \citet{dhifallah2021phase} have also conducted a precise analysis of the hard transfer and soft transfer estimation. These precise estimates seem necessary when examining the effects of transfer learning and studying negative transfer, as we need to compare the performance of joint learning with that of single-task learning.
The main intuition behind the above examples is that when the source and target tasks have complementary spectra, combining both tasks can help transfer by increasing the smallest eigenvalue in the combined covariance matrix $\hat \Sigma$.
Later in Section \ref{sec_cov_mod_shift}, we will extend these examples to a setting involving both covariate and model shifts (see Figures \ref{fig_sec3_cov_mo_a} and \ref{fig_sec3_cov_mo_b} for the illustration).

\subsection{Numerical Comparisons}

\begin{figure}%
    \begin{subfigure}[b]{0.32\textwidth}
        \centering
	    \includegraphics[width=0.99\textwidth]{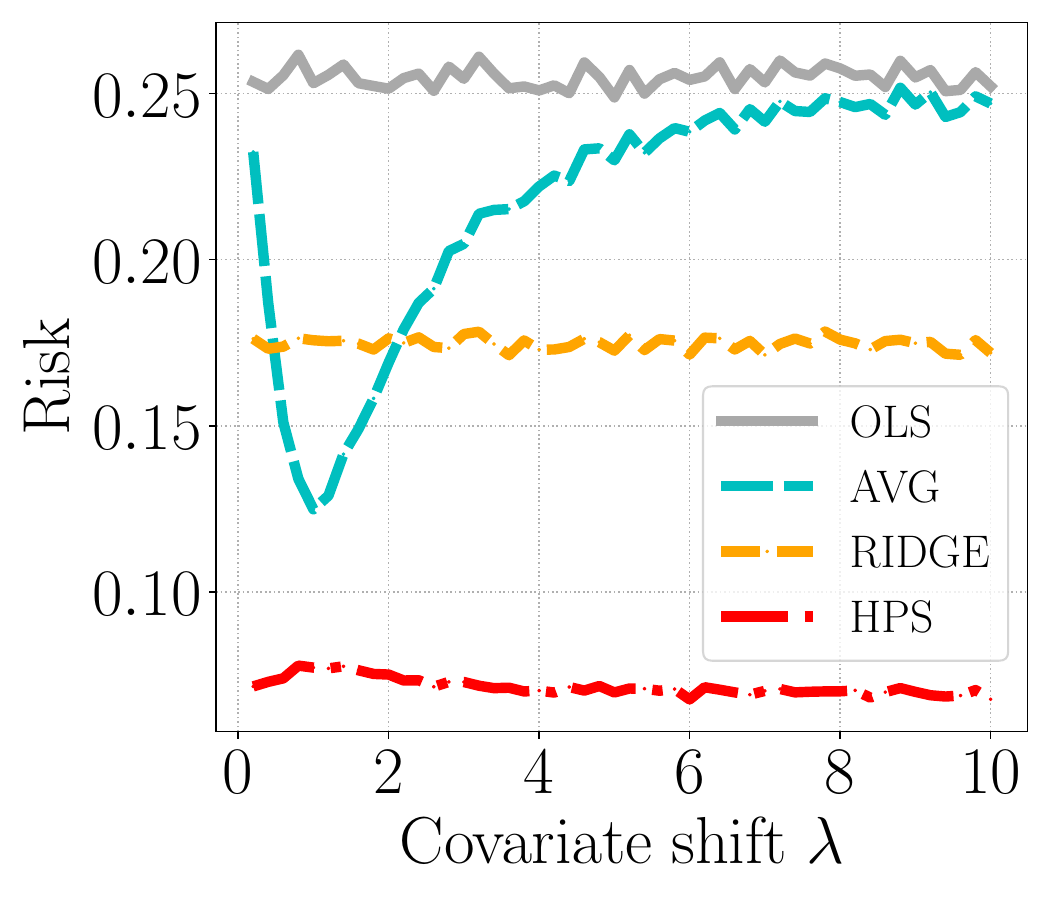}
        \caption{Varying covariate shift ($\lambda$) without model shift}\label{fig_cov_shi_ri}
	\end{subfigure}\hfill
	\begin{subfigure}[b]{0.32\textwidth}
		\centering
		\includegraphics[width=0.99\textwidth]{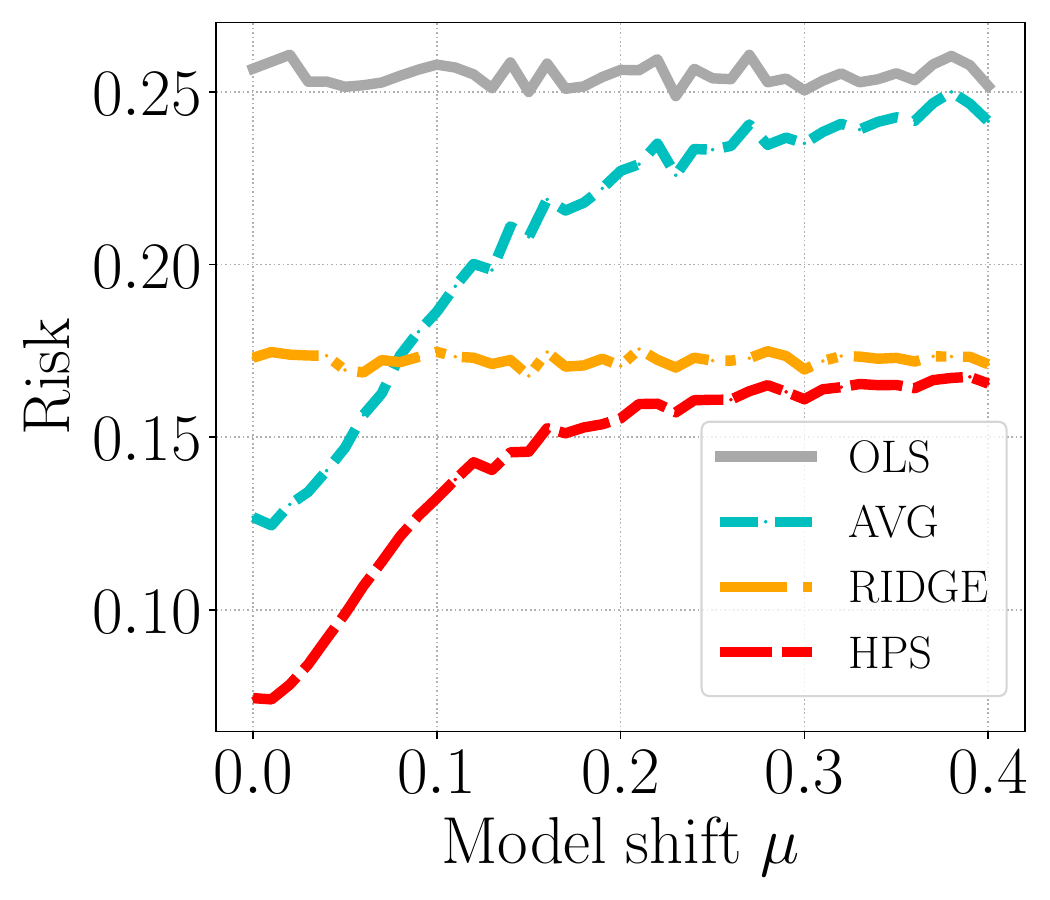}
		\caption{Varying model shift ($\mu)$ without covariate shift}
		\label{fig_sec5_model}
	\end{subfigure}\hfill%
	\begin{subfigure}[b]{0.32\textwidth}
		\centering
		\includegraphics[width=0.99\textwidth]{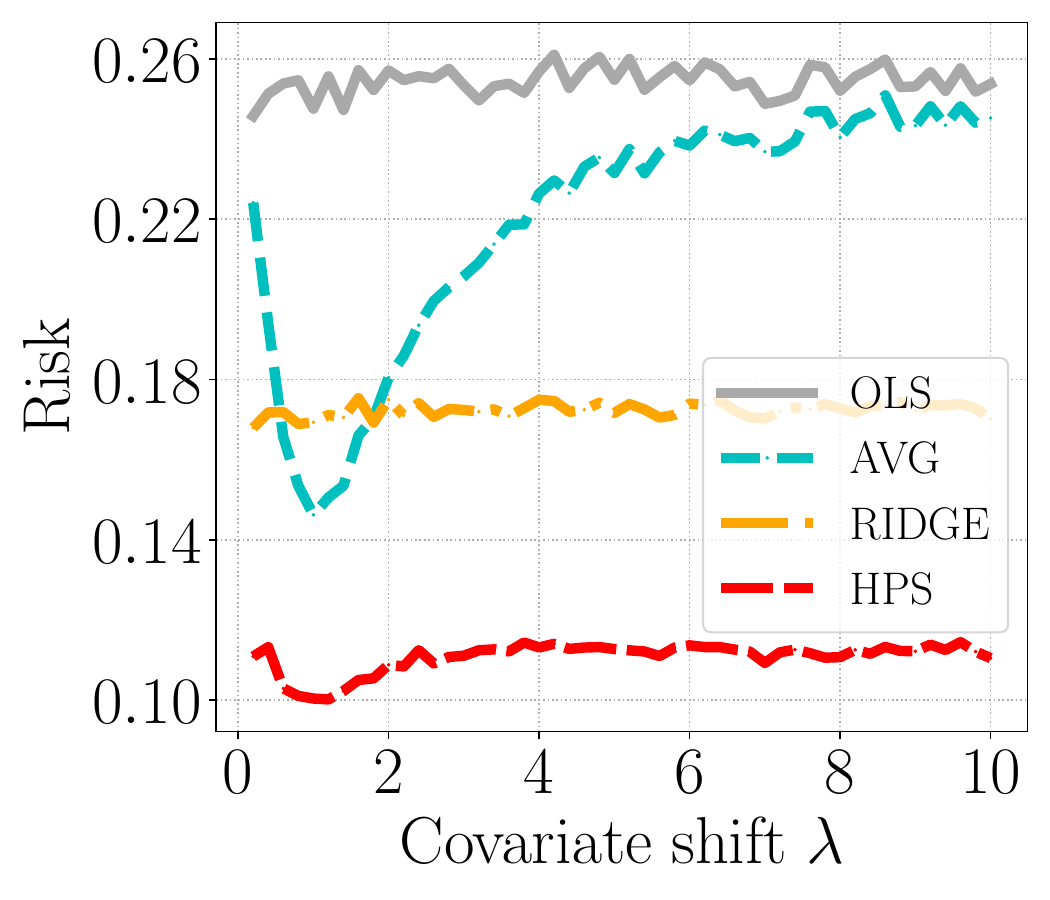}
		\caption{Varying covariate shift with fixed model shift ($\mu = 0.05$)}
		\label{fig_sec5_cov_model}
	\end{subfigure}
	\caption{Comparing the excess risks of several estimators, showing that HPS achieves the lowest excess risk compared to OLS, AVG, and RIDGE in various settings with a combination of covariate and model shifts. }\label{fig_sec5_covariate}
\end{figure}
We show that the HPS estimator provides strong empirical performance compared to several transfer learning estimators.
For this comparison, we consider the OLS estimator and the ridge estimator (RIDGE), given by:
\[ \hat{\beta}_2^{\RIDGE}(k) = \big({X^{(2)}}^{\top} X^{(2)} + k \cdot \id_{p\times p}\big)^{-1} {X^{(2)}}^{\top} Y^{(2)}. \]
We also consider an averaging estimator (AVG), which takes a convex combination of their OLS estimators:
\[ \hat{\beta}_2^{\text{AVG}}(b) =b \cdot \hat{\beta}^{\STL}_1 + (1 - b) \cdot \hat{\beta}^{\STL}_2. \]
The parameters $b$ and $k$ are optimized using a validation set independent of the training set. 
For HPS, a weight parameter is added for each task, along with a ridge penalty.
Both of these are optimized using the validation set.
We remark that all the high-dimensional asymptotic limits for HPS can be extended to this setting.

For the comparison in Figure \ref{fig_cov_shi_ri}, we sample the entries of $Z^{(1)}$ and $Z^{(2)}$ from Gaussian distributions and consider the same covariate shift as in Figure \ref{fig_sec31}, that is, $\lambda_1 = \cdots = \lambda_{\lfloor\frac p 2\rfloor} = \lambda > 1$ and $\lambda_{\lfloor{\frac p 2}\rfloor + 1}  = \cdots = \lambda_{p} = \lambda^{-1}$. %
Then, we compare the excess risk of equation \eqref{HPS_loss} for different estimators under varying levels of $\lambda$.
For this simulation, we set $p = 50, n_1 = n_2 = 100,$ and $\sigma = \frac 1 2$.
Figure \ref{fig_cov_shi_ri} shows that HPS can consistently outperform OLS, RIDGE, and  AVG in this simulation.

In Figures \ref{fig_sec5_model} and \ref{fig_sec5_cov_model}, we show that similar results continue to hold in more general settings with a combination of covariate and model shifts. In particular, we generate covariate-shifted features and different linear models for the two tasks.
We take the average of 100 random seeds because of high variances due to small sample sizes.
We find that HPS achieves strong empirical performance under various settings of covariate and model shifts.
The experiment code for reproducing these simulation results can be found at \url{https://github.com/VirtuosoResearch/Transfer_learning_random_matrix_simulations}.

\section{Model Shifts}\label{sec_sizeratio}

This section presents precise estimates of the bias and variance of HPS under model shifts, assuming no covariate shifts are present.
With these results, we can connect the bias-variance decomposition to transfer effects by comparing the bias and variance of transfer learning estimators with those of single-task learning.
In particular, we will present a detailed analysis in a random-effects model, which uncovers a phase transition in transfer learning.
This motivates us to design an adjusted HPS procedure that downsizes the source task when the model shift is too high. 
We show that the performance of this procedure can match a minimax lower bound for estimating one task from the samples of two linear regression tasks.

\subsection{Precise Estimates}\label{subsect_modelshiftonly}

In this subsection, we study the impact of model shifts on the performance of the HPS estimator when there is no covariate shift.
Particularly, both tasks have the same population covariance matrix ($\Sigma^{(1)} = \Sigma^{(2)}$) but follow different linear models ($\beta^{(1)} \neq \beta^{(2)}$).
The following result states the exact asymptotic limit of the excess risk of HPS in this case.

\begin{theorem}[Precise estimates under model shifts]\label{cor_MTL_loss}
Let $\Sigma^{(1)}=\Sigma^{(2)}$ be the same.
Under Assumption \ref{assm_big1}, suppose $Z^{(1)}$ and $Z^{(2)}$ are both Gaussian random matrices. 
Then, for any small constant $\e>0$,  with high probability over the randomness of $Z^{(1)}, Z^{(2)}$, the following estimates hold:
\begin{align}
    L_{\vari} &= \sigma^2 L_1 +    \OO\left(\sigma^2\sqrt p {(n_1+ n_2)^{-\frac 3 2 + c}}\right), \label{Lvar_samplesize} \\
    L_{\bias}&= \bignorm{{\Sigma^{(1)}}^{\frac 1 2}\big(\beta^{(1)}-\beta^{(2)}\big)}^2   L_2 
    + \OO\left( \frac{p^{- \frac 1  2 + c} n_1^2}{(n_1+n_2)^2} \left\|\beta^{(1)}-\beta^{(2)} \right\|^2 \right).\label{Lbias_samplesize}
\end{align}
Above, $L_1$ and $L_2$ are defined as 
\begin{align*}%
    L_1 & = \frac{p}{n_1+n_2-p}, \quad L_2  = \frac{n_1^2(n_1 + n_2 - p) + p n_1 n_2}{(n_1 + n_2)^2 (n_1 + n_2 - p)}.
\end{align*}
\end{theorem}

Combining equations \eqref{cor_MTL_loss} with \eqref{lem_HPS_loss} results in an exact estimate for the excess risk of HPS under model shifts.
The variance estimate \eqref{Lvar_samplesize} is a special case of Theorem \ref{thm_main_RMT} with $M=\id_{p\times p}$ and $\varphi\to\infty$ (since  Gaussian random variables have bounded moments up to any order).

The bias estimate \eqref{Lbias_samplesize} requires the assumption that both $Z^{(1)}$ and $Z^{(2)}$ are Gaussian.
We briefly describe the proof ideas.
Let $\bv=(\Sigma^{(1)})^{1/2}\left(\beta^{(1)}- \beta^{(2)}\right)$.
Then, the bias formula \eqref{Lbias} can be written as:
\begin{align*}
    L_{\bias} &=\bv^\top {Z^{(1)}}^\top Z^{(1)} { \left({Z^{(1)}}^\top Z^{(1)}+ {Z^{(2)}}^\top Z^{(2)}  \right)^{-2}}{Z^{(1)}}^\top Z^{(1)} \bv.
\end{align*}
Since both $Z^{(1)}$ and $Z^{(2)}$ are Gaussian random matrices, the distributions of ${Z^{(1)}}^\top Z^{(1)}$ and ${Z^{(2)}}^\top Z^{(2)}$ are rotation-invariant.
Thus, the following (approximate) identity holds up to a small error:
\begin{align}%
    L_{\bias} &\approx \frac {\|\bv\|^2} p\bigtr{ \big({Z^{(1)}}^\top Z^{(1)}\big)^2 { \left({Z^{(1)}}^\top Z^{(1)}+ {Z^{(2)}}^\top Z^{(2)}\right)^{-2}}} \nonumber\\
            &= \frac{\|\bv\|^2}{p} \left. \frac{\dd }{\dd x}\right|_{x=0}\bigtr{  \Big( {Z^{(1)}}^\top Z^{(1)} + x \big({Z^{(1)}}^\top Z^{(1)}\big)^2 + {Z^{(2)}}^\top Z^{(2)} \Big)^{-1}  }. \label{Lbias_idea2} 
\end{align}
Due to the rotation invariance, the asymptotic limit of equation \eqref{Lbias_idea2} is determined by the free addition of ${Z^{(1)}}^\top Z^{(1)}+ x({Z^{(1)}}^\top Z^{(1)})^2$ and  ${Z^{(2)}}^\top Z^{(2)}$. 
Building on free probability techniques (see, e.g., \citet{nica2006lectures,BES_free1}), an explicit formula for this free addition can be derived for any $x$ around zero.
This observation allows us to derive equation \eqref{Lbias_samplesize} by taking the derivative with respect to $x$ as in equation \eqref{Lbias_idea2}.
More details including the complete proof of Theorem \ref{cor_MTL_loss} are presented in Appendix \ref{subsec:HPSmodel}.

\begin{remark}\label{remark_bias}
We conjecture that the bias limit \eqref{Lbias_samplesize} is still the exact asymptotic form even if $Z^{(1)}$ and $Z^{(2)}$ are non-Gaussian random matrices.
One approach to show this is using the local laws for polynomials of random matrices. This requires checking certain technical regularity conditions, which are left as an open question for future work.
\end{remark}

\subsection{Examples and Adjustments}\label{sec_adjusted} %

We will consider a random-effects model \citep{dobriban2020wonder}.
Each $\beta^{(i)}$ consists of two components, in this case, one shared by all tasks and one that is task-specific.
Let $\beta_0$ be the shared component and $\gamma_{i}$ be the $i$-th task-specific component.
For any $i$, the $i$-th model vector is equal to
    $\beta^{(i)}=\beta_0 + \gamma_{i}.$ %
The entries of the task-specific component $\gamma_i$ are drawn independently from a Gaussian distribution with mean zero and variance $p^{-1}\mu^2$, for a parameter $\mu > 0$.
In expectation, the Euclidean distance between the two model vectors is equal to $2\mu^2$.

Based on Theorem \ref{cor_MTL_loss}, we present a precise analysis of information transfer in this random-effects model using the precise limits. The proof of the following proposition will be presented in Appendix \ref{sec:claim_model_shift}.

\begin{proposition}[Phase transition in the random-effects model]\label{claim_model_shift}
Under the assumptions of Theorem \ref{cor_MTL_loss}, suppose the random-effect model applies and $n_2 \ge 3p$. 
  For any small constant $c > 0$, the following statements hold with high probability over the randomness of training samples and model vectors.
    \begin{enumerate}%
        \item If $ {\mu^2} \frac {\bigtr{\Sigma^{(1)}}}{p} \le \frac{\sigma^2 p}{2(n_2 - p)}$, then the transfer effect is always positive:
        \begin{align}\label{HPS_le_OLS}
        L(\hat{\beta}_2^{\MTL}) \le \left(1+\OO(p^{-1/2+c})\right)\cdot L(\hat{\beta}_2^{\STL}) . %
        \end{align}
        \item If $\frac{\sigma^2 p}{2(n_2 - p)} < {\mu^2} \frac{\bigtr{\Sigma^{(1)}}}{p} < \frac{\sigma^2 n_2}{2(n_2 - p)}$, then there exists a deterministic value $n_0>0$ (which may not be an integer) such that if $n_1  \le n_0$, then equation \eqref{HPS_le_OLS} holds;
        otherwise if $n_1  > n_0$, then
        \begin{align}\label{HPS_ge_OLS}
        L(\hat{\beta}_2^{\STL}) \le \left(1+\OO(p^{-1/2+c})\right)\cdot L(\hat{\beta}_2^{\MTL}). %
        \end{align}
        \item If ${\mu^2} \frac{\bigtr{\Sigma^{(1)}}}{p} \ge \frac{\sigma^2 n_2} {2(n_2 - p)}$, then equation \eqref{HPS_ge_OLS} holds for any $n_1 > p$.  
    \end{enumerate}
\end{proposition}

Figure \ref{fig_motivation} illustrates phases 2 and 3 of Proposition \ref{claim_model_shift} for multiple values of $\mu$.
We plot our estimate vs. the excess risk using equation \eqref{Sigma_a}. 
This simulation shows that our estimate accurately matches the bias and variance.
For this simulation, we set $\Sigma^{(1)}=\Sigma^{(2)}=\id_{p\times p}$, and fix $p = 100$, $n_2 = 300$, $\sigma = 1/2$ while varying $n_1$ and $\mu$.
Notice that the above result requires $n_2 \ge 3p$; when $p < n_2 < 3p$, the threshold conditions can be derived with a similar analysis.

Figure \ref{fig_motivation} suggests that a natural training procedure is to downsample the sample size of the source task until the transfer learning performance starts to drop.
In the setting of Figure \ref{fig_motivation}, this procedure will terminate at the optimal value of $n_1$.
We can rigorously show this {\it dichotomy} in the setting of Proposition \ref{claim_model_shift} (see Appendix \ref{sec:claim_model_shift} for further details):
    \begin{enumerate}
        \item If $\mu^2 \frac {\bigtr{\Sigma^{(1)}}}{p} \ge  \frac{\sigma^2n_2}{2(n_2 - p)}$, then HPS gets strictly worse as $n_1$ increases in $[0,\infty)$. \label{claim_inc_item1}

        \item If $\mu^2 \frac {\bigtr{\Sigma^{(1)}}}{p} < \frac{\sigma^2n_2}{2(n_2 - p)}$, then there exists $n_0>0$ such that HPS first improves as $n_1$ increases in $[0, n_0]$ but turns worse as soon as $n_1 > n_0$. %
    \end{enumerate}

The above claim demonstrates that a simple adjustment to the HPS estimator can be made to reduce negative transfers.
Essentially, we will choose between HPS and OLS, whichever performs better for the target task.
In other words, if $\mu$ is too large, we will throw away the source task data.
Next, we will compare the performance of this adjusted HPS estimation with a minimax lower bound, which shows that our estimate for the adjusted HPS estimator matches the lower bound.

\subsection{A Matching Lower Bound}\label{sec_minimax}

Lastly, we complement our study in the model shift setting with a minimax lower bound.
Suppose we are trying to estimate an unknown parameter denoted as $\beta^{(2)}$.
We are given $n_2$ samples drawn from a linear parametric model, following $\beta^{(2)}$, with isotropic Gaussian covariates $X^{(2)}$, contaminated by Gaussian noise with variance $\sigma^2$.
Then, we are given $n_1$ samples from another linear parametric model, following $\beta^{(1)}$, again with isotropic Gaussian covariates $X^{(1)}$, contaminated by Gaussian noise with variance $\sigma^2$.
The parameter vectors belong to the set \[\Theta(\mu) = \set{\beta^{(1)} \in \real^p, \beta^{(2)} \in \real^p: \|{\beta^{(1)} - \beta^{(2)}}\| \le \mu, \|{\beta^{(1)}}\| \le 1, \|{\beta^{(2)}}\| \le 1}. \]
Note that our proof can be readily extended to cases with anisotropic Gaussian covariates as in equation \eqref{XofZ}, and the length constraints on ${\beta^{(1)}}$ and ${\beta^{(2)}}$ can be replaced with any other constant.\footnote{In particular, to extend the proof to the case with covariate shifts, one can replace the bound in equation \eqref{eq_norm}, Appendix \ref{proof_lb}, with different values of $c$ for each respective $\Sigma^{(1)}$ and $\Sigma^{(2)}$. The rest of the covering arguments will still go through. Then, one would set $\mu = 0$ in the specification of $\Theta(\mu)$.}

Let $\hat\beta$ be any estimation procedure that, given the above $n_1 + n_2$ samples, produces an estimate of the unknown vector $\beta^{(2)}$.
We prove the following minimax rate on the estimation error of $\hat\beta$.

\begin{theorem}\label{prop_lb}
    In the setting described above, let $\hat\beta$ be a fixed estimation procedure.
    Assume that $n_1 \ge (1+\tau) p$ and $n_2 \ge (1+\tau) p$ for a constant $\tau>0$.
    For any $(\beta^{(1)}, \beta^{(2)})$ within the set $\Theta(\mu)$, we have that
    \begin{align}
        \inf_{\hat\beta} \sup_{\Theta(\mu)} \ex{\frac 1 {n_2}\bignorm{\wt X^{(2)}\left(\hat\beta - \beta^{(2)}\right)}^2} \ge 
        c{\Big({\min\Big(\frac{n_1^2\mu^2}{(n_1+n_2)^2}, \frac{\sigma^2 p} {n_2} \Big) + \frac{\sigma^2 p} {n_1 + n_2}}\Big)}, \label{eq_lb}
    \end{align}
    where the expectation is over the randomness of $X^{(1)}, X^{(2)}, Y^{(1)}, Y^{(2)}$ and an independently drawn $\wt X^{(2)}$ that follows the same distribution as $X^{(2)}$.
    $c$ is a fixed constant that does not grow with $p$.
\end{theorem}
The lower bound in the right-hand side of equation \eqref{eq_lb} involves two parts:
\begin{itemize}
    \item For the first part, if $\mu$ is large, the source samples are not helpful, so the rate $\frac{\sigma^2 p}{n_2}$ is the OLS rate using only the target task samples. If $\mu$ is small and $n_1$ is large, then the rate $\mu^2$ appears if we use the OLS estimator for the source task.

    \item For the second part, $\frac{\sigma^2 p} {n_1 + n_2}$ is the rate in the case without model drift, i.e., $\beta^{(1)}=\beta^{(2)}$.
\end{itemize}
We now compare the minimax lower bound in equation \eqref{eq_lb} to the rates in Theorem \ref{cor_MTL_loss} and show that for an adjusted HPS estimation, the two bounds match up to each other up to constant factors.
Let $\bignorm{\beta^{(1)} - \beta^{(2)}} = \mu$.
In the setting of Theorem \ref{cor_MTL_loss}, one can see that the bias in equation \eqref{Lvar_samplesize} and the variance in equation \eqref{Lbias_samplesize} are given by the following orders:
\begin{equation}\label{eq:showorder1}
    L_{\bias} = \Theta\left(\frac{n_1^2 \mu^2}{(n_1 + n_2)^2}\right),\quad 
    L_{\vari} = \Theta\left(\frac{\sigma^2 p}{n_1+n_2}\right).
\end{equation}
These estimates hold more generally for the bias-variance decomposition in Lemma \ref{lem_HPS_loss} (beyond the setting of Theorem \ref{cor_MTL_loss});
See the derivation in Appendix \ref{app_firstpf} (equations \eqref{Op_norm2} and \eqref{Op_norm2_add}).
In particular, we can derive these estimates using concentration estimates on the eigenvalues of random matrices.

If $\mu^2 \le \frac{\sigma^2 p (n_1 + n_2)^2} {n_1^2 n_2}$, we can see that $L_{\bias} + L_{\vari}$ already matches the lower bound in equation \eqref{eq_lb}.
Then, by invoking Lemma \ref{lem_HPS_loss}, we conclude that the excess risk of HPS matches the minimax lower bound. 

Otherwise, if $\mu^2 > \frac{\sigma^2 p (n_1 + n_2)^2} {n_1^2 n_2}$, the adjusted HPS procedure---which throws away all the source task data---gives an excess risk at most $\frac{\sigma^2 p}{n_2 - p}$ (plus small concentration errors).
On the other hand, the lower bound in equation \eqref{eq_lb} becomes $\OO(\sigma^2 p (\frac 1 {n_1} + \frac 1 {n_1 + n_2}))$.
Since $n_1$ and $n_2$ are within a constant away of each other, the rate of the adjusted HPS still matches the lower bound in the case that $\mu$ is larger than the above threshold.
This concludes our study of the model shift setting.
The proof of this minimax lower bound is based on a covering argument (see, e.g., Chapter 12, \citet{zhang2023mathematical}), and the details can be found in Appendix \ref{proof_lb}.

\section{Extensions}\label{sec3_combined}

This section presents two extensions to reinforce our findings from previous sections.
First, we consider a setting that involves both covariate shifts and model shifts, and our goal is to show that the insights obtained in Section \ref{sec_main} can still be extended to this setting.
Second, we consider a setting involving multiple regression tasks with shared feature covariates, but following different linear models. The goal is to leverage these shared structures to help estimate the target task. See, e.g., \citet{AZ05,wang2016distributed} for some prior works that study similar settings.
In particular, we can still uncover a phase transition similar to our findings from Section \ref{sec_sizeratio}.

\subsection{A Setting with Both Covariate and Model Shifts}\label{sec_cov_mod_shift}

We first consider a setting with both covariate and model shifts. 
Define a function that depends on the source task's covariance:
\begin{align}\label{g1new} g_0(x)= \frac{1}{n_1}\bigtr{{\Sigma^{(1)}}\bigbrace{\id_{p\times p} + x\cdot \Sigma^{(1)}}^{-1}}  -\frac{1}{x}.
\end{align}
Let $g_0'(x)$ be the derivative of $g_0(x)$. %
It can be shown that there is a unique positive solution to the equation
\begin{align}\label{y0_simple_eq}
 ( 1+    x)g_0(x) = -n_1^{-1}(n_1+n_2-p).
\end{align}
Let $y_0$ denote this solution. The following theorem provides the exact asymptotic variance and bias limits when the population covariance of the target task is isotropic.

\begin{theorem}[Precise estimates under covariate and model shifts]\label{thm_Sigma2Id}
    Under Assumption \ref{assm_big1}, suppose that $\Sigma^{(2)}=\id_{p\times p}$ and  $Z^{(1)}$ and $Z^{(2)}$ are both Gaussian random matrices. Suppose further that the random-effects model applies. Let
    \begin{align}
        f_1 = \frac{n_1}{p}y_0 + \frac{n_1-p}{p\cdot g_0(y_0)} ,\ f_2 = \frac{n_1}{p\cdot g_0'(y_0)} - \frac{n_1-p}{p\cdot (g_0(y_0))^2} , \text{ and } f_3 = - g_0(y_0). \label{f1f2f3_new_eq}
    \end{align}
    Then, for any small constant $\e>0$, with high probability over the randomness of training samples and model vectors, the following estimates hold:
    \begin{align*}%
        &L_{\vari} =\left(1+\OO(p^{-\frac 1 2 + c})\right) \sigma^2 {p}n_1^{-1}f_1 ,  \\
        &L_{\bias} =\left(1+\OO(p^{-\frac 1 2 + c})\right) 2\mu^2 \frac{1 - 2f_1 f_3  +f_2 f_3^2  }{ 1 - n_2^{-1}{p}f_2 f_3^2 }. %
    \end{align*}
\end{theorem}

When the population covariance of the source task is also isotropic (i.e., $\Sigma^{(1)}=\id_{p\times p}$), we can verify that the above result is consistent with the variance and bias limits in Theorem \ref{cor_MTL_loss}.
The proof of Theorem \ref{thm_Sigma2Id} uses similar techniques as Theorem \ref{cor_MTL_loss} (see Appendix \ref{sec_add_modelcov}).

\subsubsection{The Anisotropic Case}

Theorem \ref{thm_Sigma2Id} states exact estimates when the target task's population covariance matrix is isotropic. In the general anisotropic case, it is possible to estimate the bias formula \eqref{Lbias} by approximating the source task's sample covariance matrix with its expectation.
It turns out that this renders the analysis of the bias formula similar to that of the variance formula in Theorem \ref{thm_main_RMT}.
However, this approximation results in an error term that decreases to zero (relative to the estimated term) as $n_1/p$ increases to infinity; see the right-hand side of equation \eqref{lem_cov_derv_eq}.
Thus, the accuracy of the estimate increases as $n_1/p$ increases.
This approximation is motivated by transfer learning scenarios where the source task is significantly larger than the target task.
The statement of this result for the anisotropic case (Theorem \ref{prop_main_RMT}), together with a complete proof, is described in Appendix \ref{sec_approx_cm}.

\begin{remark}\label{remk_hard_nonH}
In order to obtain an exact asymptotic limit of $L_{\bias}$ under general forms of $\Sigma^{(2)}$, one needs to study the singular values and singular vectors of an asymmetric random matrix 
\be\label{eq:asym_mat}\big({X^{(1)}}^{\top} X^{(1)}\big)^{-1} {X^{(2)}}^{\top} X^{(2)} + \id_{p\times p}.\ee
The eigenvalues of this matrix have been studied in the name of Fisher matrices \cite{Fmatrix}. Notice, however, that its singular values are different from its eigenvalues because of asymmetry.
This asymmetry makes it difficult to characterize the asymptotic behavior of its resolvent. We leave deriving a precise estimate of the anisotropic case to future work.
\end{remark}

\subsection{Numerical Comparisons}
We complement our theoretical analysis of HPS with empirical evaluations.
Figure \ref{fig_sec51} illustrates the result for a setting where half of the eigenvalues of $\Sigma^{(1)}$ are equal to $\lambda>1$ and the other half are equal to $1/\lambda$. Our theoretical estimates consistently match the empirical risks. 
Figure \ref{fig_sec3_cov_mo_a} shows a similar dichotomy as in Figure \ref{fig_sec3_covariate}.
We fix model shift $\mu = 0.1$ while varying covariate shift $\lambda$ for each curve.
Figure \ref{fig_sec3_cov_mo_b} illustrates different transfer effects as in Figure \ref{fig_motivation}.
We fix the level of covariate shift $\lambda = 4$ while varying $\mu$ for each curve.
Both simulations use $p = 100,$ $n_2 = 300$, and $\sigma = 1/2$.

\begin{figure}[!t]
	\begin{subfigure}[b]{0.48\textwidth}
		\centering
		\includegraphics[width=6.85cm]{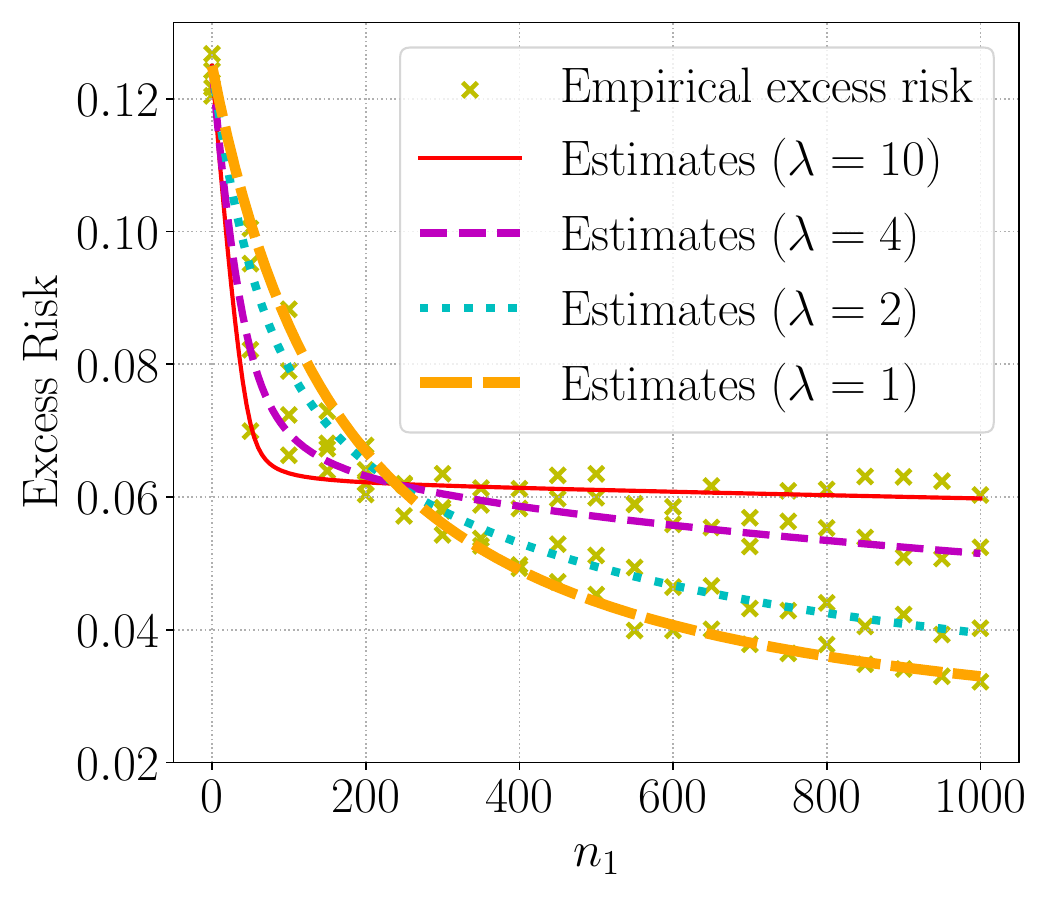} %
		\caption{Varying covariate shift with fixed model shift}
		\label{fig_sec3_cov_mo_a}
	\end{subfigure}\hfill
	\begin{subfigure}[b]{0.48\textwidth}
		\centering
		\includegraphics[width=6.85cm]{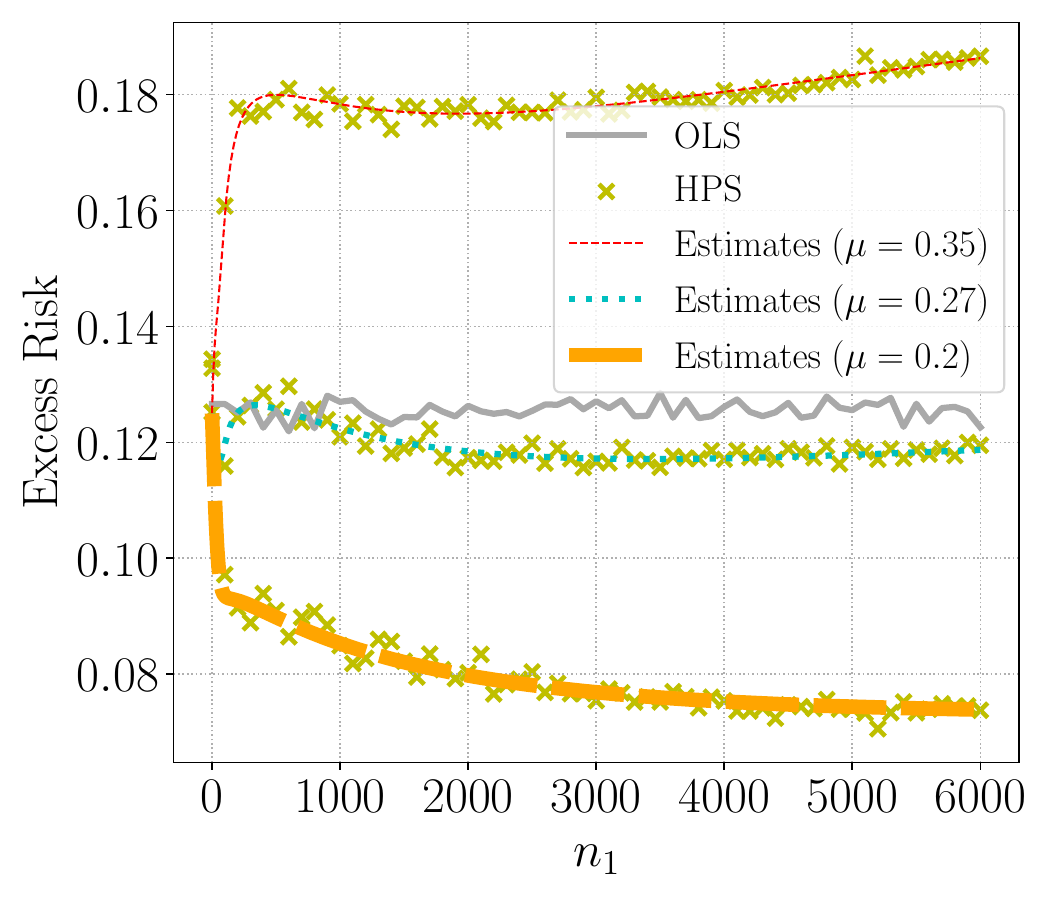} %
		\caption{Varying model shift with fixed covariate shift}
		\label{fig_sec3_cov_mo_b}
	\end{subfigure}
	\caption{Showing that our findings from Section \ref{sec_main} continue to hold under the presence of both covariate and model shifts in the data.}
	\label{fig_sec51}%
\end{figure}

\subsection{Multitask Learning with Shared Features and Different Models}\label{sec_same}

Next, we extend our findings from Section \ref{sec_sizeratio} to multiple source tasks involving model shifts.
Suppose there are $t$ tasks whose feature covariates are all equal to $X \in \real^{n \times p}$. The label vector of the $i$-th task follows a linear model with an unknown $p$ dimensional vector $\beta^{(i)}$, for $i=1, 2,\dots, t$:
\begin{align}\label{eq_mtl_data}
    Y^{(i)} = X \beta^{(i)} + \varepsilon^{(i)}.
\end{align}
This setting may also be referred to as multiple label regression, and variants of this setting have been studied in the prior literature \cite{kolar2011union,wang2016distributed}.
Similar to the two-task case, one of the tasks is viewed as the primary target task of interest, while the others are used as source tasks that aid in learning.

We assume that $X = Z\Sigma^{\frac 1 2} \in \real^{n \times p}$ is a random matrix satisfying the same assumption as $X^{(2)}$ in Assumption \ref{assm_big1}, and $\varepsilon^{(i)} \in \real^{n}$, $i=1, 2,\dots, t$, are independent random vectors, each of which is independent of $X$ and satisfies the same assumption as $\varepsilon^{(2)}$ in Assumption \ref{assm_big1}.
Furthermore, each $\beta^{(i)} \in \real^{p}$ is a (random or deterministic) vector independent of any other $\beta^{(j)}$ for $j \neq i$, the matrix $X$, and $\varepsilon^{(j)}$ for all $j=1,\dots,t$.
The sample size $n$ satisfies $1+\tau \le \frac n p\le p^{\tau^{-1}}$.

We combine the samples from all the tasks with a shared feature matrix $B$ and a task-specific prediction vector for each task.
Specifically, let $B \in \real^{p \times r}$ denote the shared feature matrix and let $A = [A_1, A_2, \dots, A_t] \in \real^{r \times t}$ denote the combined prediction variables.
We note that, unlike the two-task setting concerning covariate shifts, both $A$ and $B$ are part of the model in the multi-task setting. The reason is twofold.
First, this paper focuses on the under-parameterized setting. Thus, in the two-task setting, one can show that the optimal rank of $B$ would be one, i.e., $r = 1$, in which case $B$ becomes a scalar. 
Second, to tackle more than two tasks, we need to incorporate more of the structural information from the other tasks into the model \cite{AZ05}, which is encoded as the low-rank structures of $A$ and $B$. By contrast, in the two-task setting, we can focus on the setting where we have a shared parameter vector for both tasks.

We can now write down the loss objective as follows:
\begin{align}
	\ell(A, B) = \sum_{j=1}^t \bignorm{X B A_j - Y^{(j)}}^2, \label{eq_mtl_same_cov}
\end{align}
In particular, we focus on the case where $r < t$.
Otherwise, if $r \ge t$, the problem reduces to single-task learning (for reference, see Proposition 1 from \citet{WZR20}).

Let $(\hat A, \hat B)$ denote the global minimizer of $\ell(A, B)$.
In particular, we compute the global minimizer of $\ell(A, B)$ by first solving $\hat B$ as a function of $\hat A$---this turns out to be $(X^{\top}X)^{-1} X^{\top}Y$ multiplied by  $A^{\top} (A A^{\top})^+$, with $(AA^{\top})^{+}$ denoting the pseudo-inverse of $AA^{\top}$. 
Next, we could find the optimal $\hat A$ by taking the rank-$r$ SVD of $Y^{\top} X (X^{\top}X)^{-1} X^{\top}Y$ to get the leading $r$ left singular vectors as $U_r \in\real^{t\times r}$; See Appendix \ref{sec_multiproof} for the derivation. Then, we can obtain $\hat B \hat A$ as $(X^{\top} X)^{-1} X^{\top} Y U_r U_r^{\top}$.

We now define the HPS estimator for task $i$ as $\hat \beta_i^{\MTL} = \hat B \hat A_i$, where $\hat A_i$ denotes the $i$-th column of $\hat A$.
The excess risk of $\hat{\beta}_i^{\MTL}$ is equal to:
\begin{align}\label{ith_loss}
    L_i(\hat{\beta}_i^{\MTL}) = \bignorm{\Sigma^{\frac 1 2} \left(\hat{\beta}^{\MTL}_i - \beta^{(i)}\right)}^2.
\end{align}

\subsubsection{Results}

We demonstrate that in the aforementioned multi-task setting, hard transfer identifies the optimal rank-$r$ approximation, which is utilized to share information across multiple tasks.
To describe the result, we introduce several notations.
Let $B^\star \define [{\beta}^{(1)},{\beta}^{(2)},\dots,{\beta}^{(t)}] \in \real^{p\times t}$ be the matrix of concatenated model vectors.
Let $A^{\star} {A^{\star}}^{\top}$ be the best approximation of ${B^{\star}}^\top\Sigma B^{\star}$ in the set of rank-$r$ subspaces in $\real^t$:
\begin{align}\label{eq_A_star}
	A^{\star} \define \argmax{U\in\real^{t\times r} :\ U^{\top} U = \id_{r\times r}} \inner{U U^{\top}} {{B^{\star}}^{\top} \Sigma B^{\star}},
\end{align}
where $\langle \cdot ,\cdot \rangle $ denotes the Frobenius inner product between two matrices.
Let $a_i^{\star}\in\real^r$ be the $i$-th column vector of $A^{\star}{A^{\star}}^{\top}$.
For any matrix $X$, let $\norm{X}_2$ be its spectral norm and $\norm{X}_F$ be its Frobenius norm.
A precise estimate of the excess risk of HPS is stated below.

\begin{theorem}\label{thm_many_tasks}
Suppose the setting described above holds.
Let $r < t$ be a positive integer.
Suppose the $r$-th largest eigenvalue $\lambda_r$ of ${B^\star}^\top \Sigma B^\star$ is strictly larger than its $(r+1)$-th largest eigenvalue $\lambda_{r+1}$.
Then, for $i = 1,\dots,t$ and any small constant $c>0$, the following estimate holds with high probability over the randomness of the training samples:
\begin{align}
     \left|{L_i(\hat{\beta}_i^{\MTL}) - L_i(B^{\star}a_i^{\star}) -\frac{ \sigma^2 p \bignorm{a_i^{\star}}^2}{n-p}  }\right|  
	\le 
    \Bigbrace{\bignorms{{B^\star}^\top\Sigma B^\star} +  \sigma^2} \sqrt{\frac{ \frac{\bignorms{{B^\star}^\top\Sigma B^\star}}  {n^{\frac 1 2 - \frac 2 {\varphi} -   c}}  + \frac{\sigma^2} {n^{\frac 1 2 - c}}}{\lambda_r - \lambda_{r+1}}}. \label{Li_multi1}
\end{align}
\end{theorem}

In equation \eqref{Li_multi1}, $L_i(B^{\star} a_i^{\star})$ is the asymptotic limit of the bias of HPS, while $\frac{ \sigma^2 p}{n - p} \bignorm{a_i^{\star}}^2$ is the asymptotic limit of the variance of HPS. 
The proof of Theorem \ref{thm_many_tasks} relies on a characterization of the global minimizer of $\ell(A, B)$ by rank-$r$ SVD.
The detailed proof can be found in Appendix \ref{proof_mtl}.

\begin{remark}\label{remark_open}
The data model in equation \eqref{eq_mtl_data} studied in this section has also been considered by several prior works in which there is a shared set of nonzero coordinates in the model vectors $\beta^{(1)},  \dots, \beta^{(t)}$ among all tasks \cite{kolar2011union,wang2016distributed}. The goal is to leverage this shared sparsity to recover the support set of nonzero coordinates. 
    \citet{duan2023adaptive} examined an adaptive multi-task learning setting that does not require $X$ to be shared among all tasks.
    They introduced an adaptive estimation procedure that can achieve minimax optimality. 
    Notice that in both the work of \citet{duan2023adaptive} and the setting of Theorem \ref{thm_many_tasks}, only model shifts are assumed to be present, and instead, covariate shifts are not imposed across different tasks.
    It is an interesting question to study a combination of covariate and model shifts in the multi-task learning setting.
    More recently, \citet{li2023identification} introduce a task modeling framework to capture task relationships through influence functions.
    \citet{li2024scalable} design a boosting procedure on top of this influence estimation framework.
    It is an interesting question to theoretically formalize the benefit of boosting for multitask learning when tasks have strong negative interference.
    One promising direction is to revisit methods from active clustering and connect that task structures \cite{voevodski2012active}.
\end{remark}

\subsubsection{Illustrative Examples}

We illustrate the above result with the random-effects model.
The model vector of every task is equal to a shared vector $\beta_0$ plus a task-specific component $\gamma_{i}$, for $i=1,2,\cdots, t$: 
\begin{align}
    \beta^{(i)} = \beta_0 + \gamma_{i}.\label{eq_re_mt}
\end{align}
The entries of $\gamma_{i}$ are drawn independently from a Gaussian distribution with mean zero and variance $p^{-1} \mu^2$. 
We study two natural questions:
\begin{enumerate}
    \item What is the optimal rank $r$ of the shared feature matrix $B$?
    \item When does HPS transfer positively to a particular task, depending on the sample size $n$ and the model shift parameter $\mu$?
\end{enumerate}

For the first question, since the setting is symmetric in the $t$ tasks, we analyze the averaged limiting risk
$$g_r(n, \mu):=\frac1t\bignormFro{\Sigma^{1/2} B^{\star} (A^\star {A^\star}^{\top} - \id_{t\times t})}^2 + \frac{ \sigma^2 p}{n-p} \frac{r}{t},$$
which accurately approximates $t^{-1}\sum_{i=1}^t L_i(\hat{\beta}_i^{\MTL})  $ by Theorem \ref{thm_many_tasks}. To see this, we notice that the average of the bias and variance limits in equation \eqref{Li_multi1} is 
    $$\frac1t\sum_{i=1}^t \left(L_i(B^{\star}a_i^{\star}) + \frac{ \sigma^2 p}{n-p}  \bignorm{a_i^{\star}}^2\right) =\frac1t\bignorm{\Sigma^{1/2} B^{\star} (A^\star {A^\star}^{\top} - \id_{t\times t})}_F^2 + \frac{ \sigma^2p }{n-p} \frac{ r }{t}.$$
    Above, we have used the matrix notation $A^{\star}$ to rewrite the bias component. Moreover, we apply the identity $\sum_{i=1}^t \norm{a_i^{\star}}^2 = \tr[(A^{\star} {A^\star}^{\top})^2]=r$ to the variance component, since ${A^{\star}}^{\top} A^{\star} = \id_{r\times r}$ following Definition \eqref{eq_A_star}.
Then, the optimal rank of $B$ can be identified by sweeping through $g_r(n, \mu)$ for different values of $r$.

For the second question, recall that by equation \eqref{fact_tr}, the excess risk of the OLS estimator is given by $\frac{\sigma^2 p}{n - p}$.
Thus, by comparing $g_r(n, \mu)$ to $\frac{\sigma^2 p}{n - p}$, the exact threshold between positive and negative transfer can be identified.
This is stated as follows.

\begin{proposition}[Phase transition in multi-task learning]\label{claim_re_multi}
    Consider the multi-task setting stated above.
    Assume further that the random-effects model in equation \eqref{eq_re_mt} holds.
    Then, for any small constant $c>0$, the following claims hold with high probability over the randomness of the training data samples and the underlying linear models:
    \begin{enumerate}
        \item When $\mu^2 > \left(1 + p^{-\frac 1 2 + c}\right) \frac{\sigma^2 p^2}{(n - p) \tr[\Sigma]}$, then $g_r(n, \mu) \ge \frac{\sigma^2 p}{n - p} $ for any $1\le r < t$.
    	\item When $\mu^2 < \left(1 - p^{-\frac 1 2 + c}\right) \frac{\sigma^2 p^2}{(n - p)\tr[\Sigma]} $, then $g_r(n, \mu)$ is minimized when $r=1$.  Furthermore, when $r = 1$, we have that $g_1(n, \mu) \le \frac{\sigma^2 p}{n - p}$.
    \end{enumerate}
\end{proposition}

\begin{figure}[!t]
	\begin{subfigure}[b]{0.49\textwidth}
		\centering
		\includegraphics[width=0.99\textwidth]{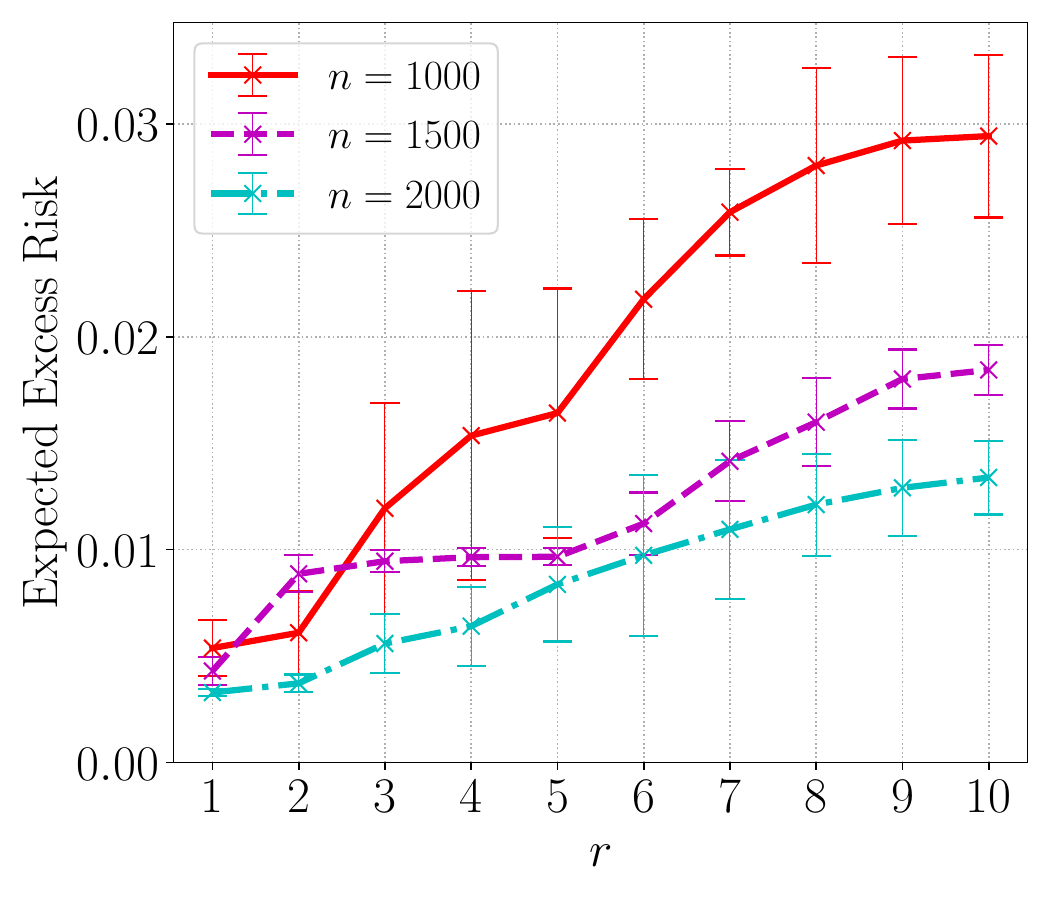}
		\caption{Varying the rank of the shared matrix $B$}
		\label{fig_sec4_width}
	\end{subfigure}
	\begin{subfigure}[b]{0.49\textwidth}
		\centering
		\includegraphics[width=0.99\textwidth]{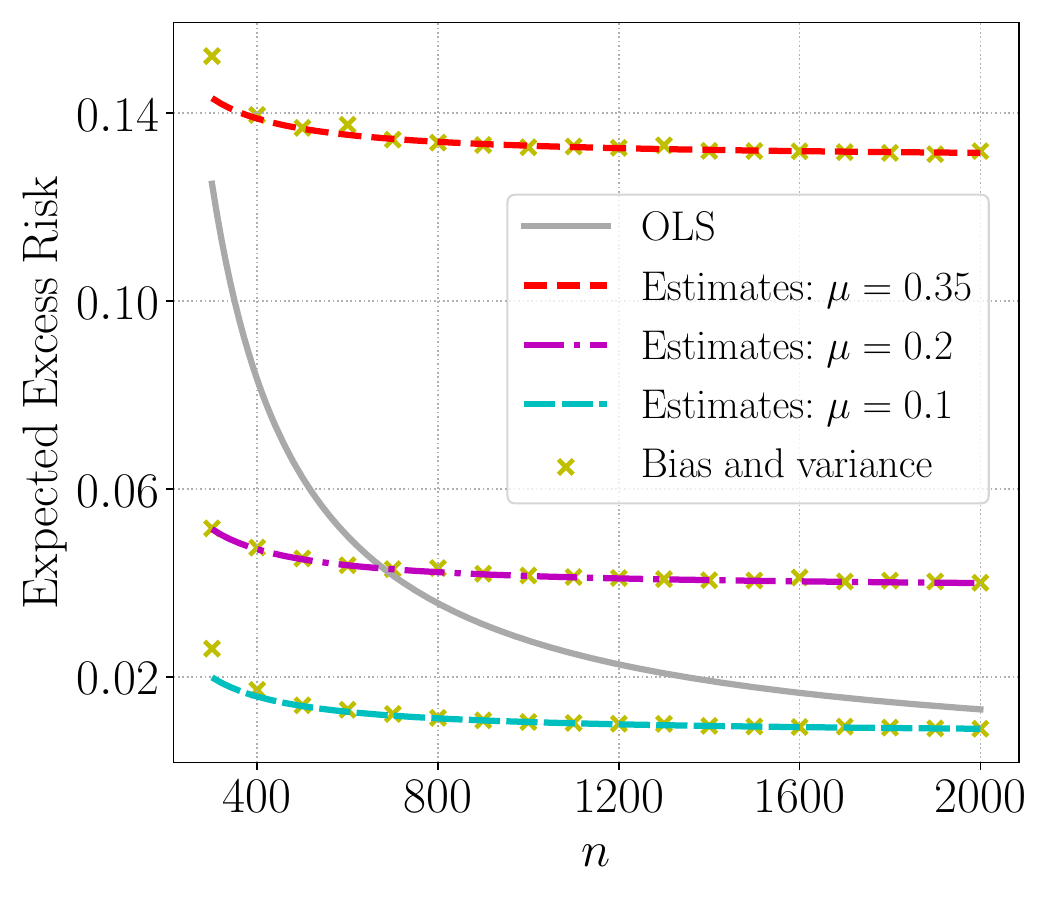}
		\caption{Varying model shift $\mu$}
		\label{fig_sec4_transfer}
	\end{subfigure}
	\caption{Illustration of transfer effects in the random-effects model with multiple source tasks.
	Figure \ref{fig_sec4_width} shows that for rank $r$ from $1$ to $10$, the lowest excess risk of task $t$ is achieved at $r = 1$ (averaged over three random seeds).
	Figure \ref{fig_sec4_transfer} fixes $r = 1$ and varies $\mu, n$.
    Similar to Figure \ref{fig_motivation}, the transfer effect of HPS can be either positive or negative depending on $\mu, n$.}
	\label{fig_sec4}
\end{figure}

Figure \ref{fig_sec4} illustrates the above result with ten tasks of dimension $p = 100$ and noise variance $\sigma = \frac 1 2$.
We plot the theoretical estimate using $g_r(n, \mu)$ and the empirical value of bias plus variance of $\hat{\beta}_t^{\MTL}$ (equation \eqref{eq_mtl_bv} in Section \ref{sec_multiproof}).
Figure \ref{fig_sec4_width} shows that the performance of HPS is optimized when $r = 1$.
Figure \ref{fig_sec4_transfer} shows different transfer effects by varying $n$ and $\mu$ in the multi-task setting.
The results under different values of $\mu$ also match the conditions in Proposition \ref{claim_re_multi}.
The proof of Proposition \ref{claim_re_multi} can be found in Appendix \ref{sec:pf-claim_re_multi}.

\section{Comparison to Soft Transfer}\label{subsec_SPS}

In this section, we provide a case study of the statistical behavior of soft transfer estimation.
Soft transfer works by penalizing the distance between the linear models for the target task and the source task.
In general, soft transfer offers a more flexible regularization approach for adjusting the distance.
Thus, one can expect the statistical behavior of soft transfer to be different from that of HPS.
In this section, we provide a case study in several settings to illustrate this difference.

\subsection{Bias-Variance under Deterministic Designs}
We first derive the expectation of the excess risk $ L(\hat\beta^{\SPS}_2(\lambda))$ over the randomness of $\varepsilon^{(1)}$ and $ \varepsilon^{(2)}$ to get the bias and variance of SPS.
For simplicity, we denote $\hat\Sigma^{(1)} = {X^{(1)}}^{\top} X^{(1)}$, $\hat\Sigma^{(2)} = {X^{(2)}}^{\top} X^{(2)}$, and $\hat\Sigma=\hat\Sigma^{(1)}+\hat\Sigma^{(2)}$ as in \eqref{Sigma_a}.
Let $\SigSPS$ (which depends on the regularization parameter $\lambda>0$\footnote{The case that $\lambda = 0$ is not interesting since there would be no information transfer in this case.}) be
\begin{equation}\label{eq:SigSPS}
    \SigSPS: = \lambda^{-1}\hat\Sigma^{(1)}\hat\Sigma^{(2)} +\hat\Sigma.
\end{equation}
The bias and variance of the SPS estimator can be derived from straightforward algebraic calculations.  

\begin{lemma}\label{lem_SPS_loss}
In expectation over the randomness of $\varepsilon^{(1)}$ and $\varepsilon^{(2)}$, we have the following equation for any $\lambda > 0$: 
    \begin{align}
        &\exarg{\varepsilon^{(1)}, \varepsilon^{(2)}}{L\big(\hat{\beta}_2^{\SPS}(\lambda)\big)} = L_{\bias}(\lambda) +  L_{\vari}(\lambda), \label{L_SPS_simple} 
    \end{align}
    where the bias and variance formulas are defined as  
    \begin{align}
        &L_{\bias} (\lambda) = \bignorm{ {\Sigma^{(2)}}^{\frac 1 2}  {\big({\SigSPS}\big)^{-1} \hat\Sigma^{(1)} \left(\beta^{(1)} - \beta^{(2)}\right)}}^2,  \label{Lbias_SPS} \\
        &L_{\vari} (\lambda) = \sigma^2 \bigtr{ \bigg(\big({\SigSPS}\big)^{-1}\bigg)^\top \Sigma^{(2)} \big({\SigSPS}\big)^{-1} \Bigbrace{\hat\Sigma^{(1)} + \Bigbrace{ \frac{\hat\Sigma^{(1)}}{\lambda} + \id} \hat\Sigma^{(2)}\Bigbrace{ \frac{\hat\Sigma^{(1)}}{\lambda} + \id}}}. \label{Lvar_SPS}
    \end{align}
\end{lemma}
As a sanity check, one can see that when $\lambda$ tends to infinity, the above reduces to the bias and variance of SPS, since ${\SigSPS}$ tends to $\hat \Sigma$ and $\hat\Sigma^{(1)}/\lambda$ tends to zero.

\subsection{An Illustrative Example} 
Above, we have considered $\hat\Sigma^{(1)}$ and $\hat\Sigma^{(2)}$ to be deterministic covariance matrices by conditioning on $X^{(1)}$ and $X^{(2)}$. 
Next, we are given a special example to compare the behavior of SPS with HPS, assuming that $(Z^{(2)})^\top Z^{(2)} / n_2 = \id_{p\times p}$ while task one still follows a random design.
Given any $\lambda>0$, denote 
$$\wt \Sigma(\lambda)=\frac{n_1}{\lambda} \Sigma^{(2)} + \frac{n_1}{n_2 } \id_{p\times p} .$$
Let $a_0$ be the unique positive solution to the equation 
\begin{equation*}%
\frac{1}{a_0}=1 + \frac{1}{n_1}\tr\left[\frac{\wt\Sigma(\lambda)}{\id_{p\times p}+a_0\wt \Sigma(\lambda)}\right].
\end{equation*}
Then, let
\begin{equation}\label{x0-SPS} 
    x_0 = \frac{1}{n_1}\bigtr{{\Big(\id_{p\times p}+a_0\wt \Sigma(\lambda)\Big)^{-2}}},\ \ 
    y_0 ={a_0^2 x_0}\Big(1-\frac{p}{n_1}-2a_0 + x_0\Big)^{-1}.
\end{equation}
Lastly, define two matrices as follows:
\begin{equation}\label{def-M12-SPS}
    M_1 =\frac{n_1 a_0}{n_2} \Big({\id_{p\times p}+a_0\wt\Sigma(\lambda)}\Big)^{-1},\ \
    M_2 = \Big(\frac{n_1}{n_2}\Big)^2\left(a_0^2-y_0\right){\Big(\id_{p\times p}+a_0 \wt \Sigma(\lambda)\Big)^{-2}}.
\end{equation}
We state the limits of bias and variance of SPS as follows.
\begin{proposition}\label{thm_SPS}
Suppose Assumption \ref{assm_big1} holds.
Additionally, assume that $\Sigma^{(1)}=\Sigma^{(2)}$ and ${Z^{(2)}}^\top Z^{(2)} / n_2 = \id_{p\times p}$. Then, for any small constant $\e>0$,  with high probability over the randomness of the training samples, the following estimates hold: 
\begin{align}
    L_{\bias} &=\big(\beta^{(1)}- \beta^{(2)}\big)^\top {\Sigma^{(2)}}^{\frac 1 2} M_2{{\Sigma^{(2)}}}^{\frac 1 2}\big(\beta^{(1)}- \beta^{(2)}\big) 
    +     \OO\left( {n^{\frac{1}{\varphi}-\frac1  4 + c}\|\beta^{(1)}- \beta^{(2)}\|^2} \right), \label{Lvar_sps} \\
    L_{\vari} &= \frac{\sigma^2 p}{n_2} - \frac{\sigma^2}{n_2} \tr\big[{M_1}\big]- \frac{\sigma^2}{\lambda} \tr\big[\Sigma^{(2)} M_2\big]  +  \OO\left(\Big(\frac{p\sigma^2 }{n_2}  +  \frac{p\sigma^2 }{\lambda}\Big)n^{\frac1\varphi -\frac 14+c} \right),\label{Lbias_sps}
\end{align}
where we denote $n=n_1+n_2.$
\end{proposition}

Recalling that $\varphi > 4$, for a small enough constant $c$, equations \eqref{Lvar_sps} and \eqref{Lbias_sps} characterize the limits of $L_{\vari}$ and $L_{\bias}$ with an error term that is smaller than the deterministic leading term by an asymptotically vanishing factor.
The proof of Lemma \ref{lem_SPS_loss} and Proposition \ref{thm_SPS} can be found in Appendix \ref{sec:pf_decomp_SPS}.

It is an interesting question to derive the asymptotic bias and variance limits for the SPS estimator under random designs for both $\hat\Sigma^{(1)}$ and $\hat\Sigma^{(2)}$, with arbitrary covariate and model shifts. This requires developing new tools beyond the current random matrix theory literature.
Similar to the issues in Remark \ref{remk_hard_nonH}, one needs to deal with an even more complicated asymmetric random matrix than equation \eqref{eq:asym_mat}, which is left for future work:
$$(\hat\Sigma^{(1)})^{-1}\SigSPS = \lambda^{-1} \hat\Sigma^{(2)} + (\hat\Sigma^{(1)})^{-1} \hat\Sigma^{(2)} + \id_{p\times p}.$$

\subsection{Numerical Comparisons}
We conduct a numerical comparison of SPS, OLS, and HPS.
In Figure \ref{fig_sps}, we use the same data-generating process as Figure \ref{fig_motivation}, except that we now take ${Z^{(2)}}^\top Z^{(2)} =n_2 \id_{p\times p}$.
We find that the regularization effect of SPS can generally help reduce negative transfer, especially when the model shift is large.
Additionally, our estimates in Proposition \ref{thm_SPS} match the bias and variance values in finite dimensions. Here, we set $p=100$.

\begin{figure}[t!]
    \centering
    \includegraphics[width = 0.8\textwidth]{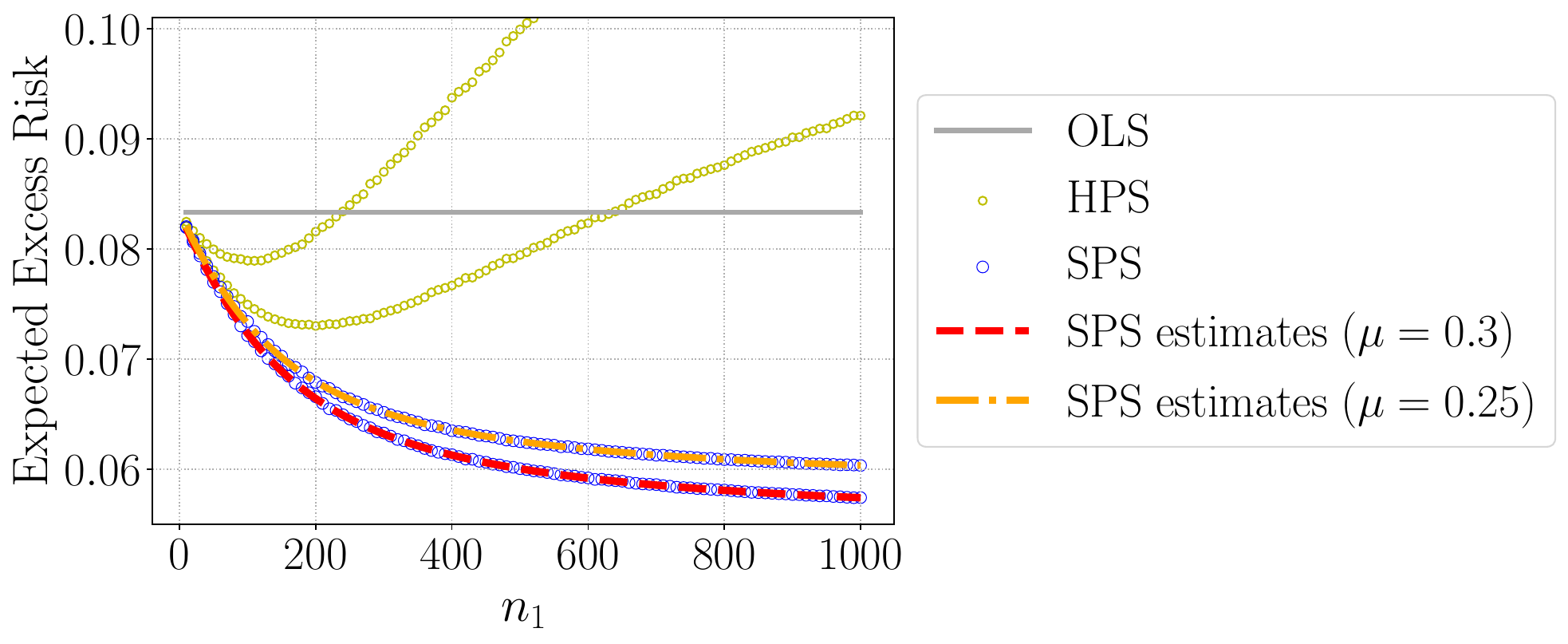}
    \caption{We find that the regularization effect of SPS can help mitigate negative transfer, which is more prominent when $\mu$ is higher (as indicated by the two lines where $\mu = 0.3$ and $\mu = 0.25$. We leave a thorough understanding of this phenomenon, including soft transfer, for future work.}\label{fig_sps}
\end{figure}

\subsection{Discussion}
Our analysis of SPS suggests that it behaves similarly to HPS when the model shift is low or modest, and SPS can generally help reduce negative transfer when the model shift is high. In this sense, the regularization of SPS behaves similarly to the data set shrinkage applied to HPS. More general forms of soft transfer can be found in transfer learning with foundation models, such as fine-tuning and instruction tuning, where overfitting often occurs when a large model is fine-tuned on a small target data set. For such scenarios, applying the regularization in soft transfer can effectively mitigate overfitting. It would be interesting to further study the regularization effect of soft transfer in large models, which would also require developing tools for analyzing over-parameterized models (and perhaps a separate data model). These are left for future work.

\section{Conclusion}\label{sec_conclude}

This paper presents several high-dimensional asymptotic results for analyzing transfer learning with two linear regression tasks: one is the source task, and the other is the target task.
The main technical ingredients involve estimating the high-dimensional asymptotic limits of several functions involving two independent sample covariance matrices with different population covariance matrices.
In particular, our main result involves:
\begin{itemize}
    \item A generalization of a well-known result from the random matrix theory literature on the trace of the inverse of one sample covariance matrix to the sum of two sample covariance matrices under covariate shifts (Theorem \ref{thm_main_RMT}), along with an almost sharp convergence rate to this limit.
    \item A phase transition of transfer learning under model shifts between two high-dimensional linear regression tasks, depending on the level of model shifts and the sample sizes of source and target tasks (Theorem \ref{cor_MTL_loss}). With a simple adjustment, HPS is nearly optimal, matching the minimax lower bound up to constant factors.
    \item An extension of the findings to a setting involving both covariate shift and model shift, and a multi-task learning setting where the feature covariates are shared among all tasks.
\end{itemize}
Numerical simulations are provided to show the accuracy of our estimates in finite dimensions.

Our work establishes a rigorous connection between transfer learning and random matrix theory, highlighting several technically challenging questions that are worth further study.
The theoretical modeling of transfer learning is still a relatively new research area. 
We hope our work can inspire future studies that apply random matrix theory to transfer learning.

\acks{Thanks to Tony Cai, Edgar Dobriban, Kaizheng Wang, Yaqi Duan, Wei Hu, and Hongji Wei for helpful discussions at various stages of this project.
We are grateful to the action editor and the anonymous referees for their helpful comments, which helped improve this paper.

Hongyang Zhang is supported in part by NSF award IIS-2412008.
Fan Yang is supported in part by the National Key R\&D Program of China (No. 2023YFA1010400), and is also affiliated with Beijing Institute of Mathematical Sciences and Applications.
Fan Yang and Hongyang Zhang were partly supported by the Wharton Dean's Fund for Postdoctoral Research during the early stage of this project at UPenn.
Weijie Su is supported in part by NSF through DMS-1847415 and the Wharton Dean's Fund for Postdoctoral Research.}

\appendix

\section{Basic Tools}\label{app_tool}

Here is a roadmap of the appendix materials.
In Appendix \ref{sec:notation} and \ref{sec:concentration}, we describe some basic tools used to establish the proofs.
Then, in Appendix \ref{app_firstpf}, we prove the bias-variance decomposition of the HPS estimation from Section \ref{sec_HPS}.
In Appendix \ref{proof_sec_cov}, we prove the results stated in Section \ref{sec_main}.
In Appendix \ref{app_iso_cov}, we give the complete proofs for the estimates under model shifts from Section \ref{sec_sizeratio} and the extension in Section \ref{sec_cov_mod_shift}.
We state the proofs for the multi-task case in Appendix \ref{sec_multiproof}.
Finally, Appendix \ref{sec:pf_decomp_SPS} provides proofs for the SPS estimation.

\subsection{Notations}\label{sec:notation}

In this supplement, we will use the following notations. The fundamental large parameter in our proof is $p$. All quantities that are not explicitly constant may depend on $p$, and we usually omit $p$ from our notations.
Given any matrix $X$, let $\norm{X}$ or $\bignorms{X}$ denote its operator norm
(or equivalently, the largest singular value); let $\bignormFro{X}$ denote its Frobenius norm; let $\lambda_{\max}(X)$ and $\lambda_{\min}(X)$ denote its largest and smallest singular values, respectively; let $\lambda_1(X)\ge \lambda_2(X)\ge \cdots $ denote the singular values of $X$ in descending order;
let $X^+$ denote the Moore-Penrose pseudoinverse of $X$. 
As a special case, the operator or Frobenius norm of a vector $v$ is also its Euclidean norm, i.e., $\|v\|=\|v\|_2=\|v\|_F$. 
We will often write an identity matrix as $1$ if this does not cause any confusion in the specific context.

We say an event $\Xi$ holds with high probability if $\P(\Xi)\to 1$ as $p\to \infty$.

We will use the big-O notation $g(p) = \OO(f(p))$ if there exists a constant $C>0$ such that $|g(p)| \le C f(p)$ for large enough $p$. Moreover, for any  $g(p)\ge 0$ and $f(p)\ge 0$, we will use the notations $g(p)\lesssim f(p)$ if $g(p) = \OO(f(p))$, and $g(p)\sim f(p)$ if $g(p) \lesssim f(p)$ and $f(p) \lesssim g(p)$.

\subsection{Concentration Estimates}\label{sec:concentration}

In this section, we collect some useful tools that will be used in the proof. 
First, it is convenient to introduce the following notation.
\begin{definition}[Overwhelming probability]
We say an event $\Xi$ holds \emph{with overwhelming probability} (w.o.p.) if for any constant $D>0$, $\mathbb P(\Xi)\ge 1- p^{-D}$ for large enough $p$. Moreover, we say $\Xi$ holds with overwhelming probability in an event $\Omega$ if for any constant $D>0$, $\mathbb P(\Omega\setminus \Xi)\le p^{-D}$ for large enough $p$.
\end{definition}

The following notion of stochastic domination is commonly used to study random matrices.

\begin{definition}[Stochastic domination]\label{stoch_domination}
Let $\xi\equiv \xi^{(p)}$ and $\zeta\equiv \zeta^{(p)}$ be two $p$-dependent random variables.
We say that $\xi$ is stochastically dominated by $\zeta$, denoted by $\xi\prec \zeta$ or $\xi=\OO_\prec(\zeta)$, if for any small constant $\e > 0$ and large constant $D > 0$, there exists a $p_0(\e, D)\in \N$ such that for all $p > p_0(\e, D)$,
\[ \bbP\left(|\xi| >p^\e |\zeta|\right)\le p^{-D}. \]
In other words, $\xi\prec \zeta$ if $|\xi| \le p^\e |\zeta|$ with overwhelming probability for any small constant $\e>0$. If $\xi(u)$ and $\zeta(u)$ are functions of $u$ supported in a set $\cal U$, then we say $\xi(u)$ is stochastically dominated by $\zeta(u)$ uniformly in $\cal U$ if %
\[ \bbP\left(\cup_{u\in \cal U}\{|\xi(u)|>p^\e |\zeta(u)|\}\right)\le p^{-D} \]
for large enough $p$. Given any event $\Omega$, we say $\xi \prec \zeta$ on $\Omega$ if $\mathbf 1_{\Omega}\xi \prec \zeta$.
\end{definition}

\begin{remark}\label{rem_stoch_add}
We make two simple remarks.
First, since we allow for an $p^\e$ factor in stochastic domination, we can ignore $\log$ factors without loss of generality since $(\log p)^C\prec 1$ for any constant $C>0$.
Second, given a random variable $\xi$ with unit variance and finite moments up to any order as in \eqref{eq_highmoments}, we have that $|\xi|\prec 1$.
This is because, by Markov's inequality, there is
$$ \P(|\xi|\ge p^{\e})\le p^{-k\e}{\E |\xi|^k}\le p^{-D},$$
as long as $k$ is taken to be larger than $D/ \e$.
\end{remark}

The following lemma collects several basic properties of stochastic domination that will be used tacitly in the proof. Roughly speaking, it says that the stochastic domination ``$\prec$" can be treated as the conventional less-than sign ``$<$" in some sense.

\begin{lemma}[Lemma 3.2 in \citet{isotropic}]\label{lem_stodomin}
Let $\xi$ and $\zeta$ be two families of nonnegative random variables depending on some parameters $u\in \cal U$ and $v\in \cal V$. Let $C>0$ be an arbitrary constant. 
\begin{itemize}
\item[(i)] {\bf Sum.} Suppose that $\xi (u,v)\prec \zeta(u,v)$ uniformly in $u\in \cal U$ and $v\in \cal V$. If $|\cal V|\le p^C$, then $\sum_{v\in \cal V} \xi(u,v) \prec \sum_{v\in \cal V} \zeta(u,v)$ uniformly in $u$.

\item[(ii)] {\bf Product.} If $\xi_1 (u)\prec \zeta_1(u)$ and $\xi_2 (u)\prec  \zeta_2(u)$ uniformly in $u\in \cal U$, then $\xi_1(u)\xi_2(u) \prec \zeta_1(u) \zeta_2(u)$ uniformly in $u\in \cal U$.

\item[(iii)] {\bf Expectation.} Suppose that $\Psi(u)\ge p^{-C}$ is a family of deterministic parameters, and $\xi(u)$ satisfies $\mathbb E\xi(u)^2 \le p^C$. If $\xi(u)\prec \Psi(u)$ uniformly in $u$, then we also have $\mathbb E\xi(u) \prec \Psi(u)$ uniformly in $u$.
\end{itemize}
\end{lemma}

Let $Q>0$ be a ($p$-dependent) deterministic parameter. We say a random matrix $Z \in\real^{n \times p}$ satisfies the {\it{bounded support condition}} with $Q$ (or $Z$ has \emph{bounded support} $Q$) if
\begin{equation}
	\max_{1\le i \le n, 1 \le j \le p}\vert Z_{i j} \vert \prec Q. \label{eq_support}
\end{equation}
As shown in Remark \ref{rem_stoch_add}, if the entries of $Z$ have finite moments up to any order, then $Z$ has bounded support $Q=1$.
More generally, if every entry of $Z$ has a finite $\varphi$-th moment as in \eqref{conditionA2} and $n\ge p$, then using Markov's inequality and a simple union bound we get that %
\begin{align}
	\P\left(\max_{1\le i\le n, 1\le j \le p}|Z_{i  j}|\ge (\log n) n^{\frac{2}{\varphi}}\right) &\le \sum_{i=1}^n \sum_{j=1}^p \P\left(|Z_{i j}|\ge (\log n) n^{\frac{2}{\varphi}}\right)  \nonumber\\
	&\lesssim \sum_{i=1}^n \sum_{j=1}^p  \left[(\log n) n^{\frac{2}{\varphi}}\right]^{-\varphi} \le (\log n)^{-\varphi}.\label{Ptrunc}
	\end{align}
In other words, $Z$ has bounded support $Q=n^{{2}/{\varphi}}$ with high probability.

The following lemma, which can also be found from Lemma 3.8 in \citet{EKYY1}, gives sharp concentration bounds for linear and quadratic forms of random variables with bounded support. %
\begin{lemma}\label{largedeviation}
Let the sequence $(x_i)$, $(y_j)$ be families of centered and independent random variables, and $(A_i)$, $(B_{ij})$ be families of deterministic complex numbers. Suppose the entries $x_i$ and $y_j$ have variance at most $1$, and satisfy the bounded support condition (\ref{eq_support}) for a deterministic parameter $Q \ge 1$. %
Then, we have the following estimates:
\begin{align}
& \Big| \sum_{i=1}^n A_i x_i \Big\vert \prec Q \max_{i=1}^n \vert A_i \vert+ \Big(\sum_{i=1}^n |A_i|^2 \Big)^{1/2} ,\label{eq largedev10}  \\ 
& \Big\vert  \sum_{i,j=1}^n x_i B_{ij} y_j \Big\vert \prec Q^2 B_d  + Q n^{1/2}B_o +  \Big(\sum_{1\le i\ne j\le n} |B_{ij}|^2\Big)^{{1}/{2}},\label{eq largedev11}  \\
& \Big\vert  \sum_{i=1}^n (|x_i|^2-\mathbb E|x_i|^2) B_{ii}  \Big\vert  \prec Q n^{1/2}B_d   ,\label{eq largedev20}\\ 
&\Big\vert  \sum_{1\le i\ne j\le n} x_i B_{ij} x_j \Big\vert  \prec Qn^{1/2}B_o +  \Big(\sum_{1\le i\ne j\le n} |B_{ij}|^2\Big)^{{1}/{2}} ,\label{eq largedev21}
\end{align}
where we denote $B_d:=\max_{i} |B_{ii} |$ and $B_o:= \max_{i\ne j} |B_{ij}|.$ Moreover, if $ x_i$ and $ y_j$ have finite moments up to any order, then we have the following stronger estimates:
\begin{align}
& \Big\vert \sum_{i=1}^n A_i x_i \Big\vert \prec  \Big(\sum_{i=1}^n |A_i|^2 \Big)^{1/2} , \label{eq largedev0} \\
&  \Big\vert \sum_{i,j=1}^n x_i B_{ij} y_j \Big\vert \prec  \Big(\sum_{i, j=1}^n |B_{ij}|^2\Big)^{{1}/{2}}, \label{eq largedev1} \\
& \Big\vert  \sum_{i=1}^n (|x_i|^2-\mathbb E|x_i|^2) B_{ii}  \Big\vert  \prec  \Big( \sum_{i=1}^n |B_{ii} |^2\Big)^{1/2}  ,\label{eq largedev2} \\
&   \Big\vert  \sum_{1\le i\ne j\le n} x_i B_{ij} x_j \Big\vert  \prec \Big(\sum_{1\le i\ne j\le n } |B_{ij}|^2\Big)^{{1}/{2}}.\label{eq largedev3}
\end{align}
\end{lemma}

It is well-known that the empirical spectral distributions of ${Z^{(1)}}^\top Z^{(1)}$ and ${Z^{(2)}}^\top Z^{(2)}$ satisfy the Marchenko-Pastur (MP) law \citep{MP}. Moreover, their eigenvalues are all inside the support of the MP law with high probability. In the proof, we will need a slightly stronger result that holds with overwhelming probability, as given by the following lemma.
 
\begin{lemma}\label{SxxSyy}
Suppose $Z\in \R^{n\times p}$ is an $n\times p$ random matrix satisfying the same assumptions as $Z^{(2)}$ in Assumption \ref{assm_big1}. Suppose $1+\tau \le {n}/{p} \le p^{\tau^{-1}}$, and $Z$ satisfies the bounded support condition \eqref{eq_support} for a deterministic parameter $Q$ such that $ 1\le Q \leq n^{ 1/2 - c_Q} $ for a constant $c_Q>0$.  Then, we have that %
\be\label{op rough2} (\sqrt{n}-\sqrt{p})^2 - \OO_\prec(\sqrt{n}\cdot Q) \le  \lambda_p (Z^\top Z)  \le  \lambda_1(Z^\top Z) \le (\sqrt{n}-\sqrt{p})^2 + \OO_\prec(\sqrt{n}\cdot Q)  .
\ee
\end{lemma}
\begin{proof}
When $Q$ is of order 1, this lemma follows from Theorem 2.10 of \citet{isotropic}. %
The result for the general case with $ 1\le Q \leq n^{ 1/2 - c_Q} $ follows from Lemma 3.11 of \citet{DY}. 
\end{proof}

Using a standard cut-off argument, we can extend Lemma \ref{largedeviation} and Lemma \ref{SxxSyy} to random matrices whose entries satisfy only certain moment assumptions but not necessarily the bounded support condition.

\begin{corollary}\label{fact_minv}
Suppose $Z\in \R^{n\times p}$ is an $n\times p$ random matrix satisfying the same assumptions as $Z^{(2)}$ in Assumption \ref{assm_big1}. Suppose  $1+\tau \le {n}/{p} \le p^{\tau^{-1}}$. Then, equation \eqref{op rough2} holds on a high probability event with $Q=n^{2/\varphi}$, where $\varphi$ refers to the constant in equation \eqref{conditionA2}.
\end{corollary}
\begin{proof}
We introduce a truncated matrix $\wt Z$ with entries 
\be\label{truncateZ} \wt Z_{ij}:= \mathbf 1\left( |Z_{ij}|\le Q \log n\right)\cdot Z_{ij}.\ee
From equation \eqref{Ptrunc}, we get that
\begin{equation}\label{XneX222}
\mathbb P(\wt Z= Z) = 1- \P\left(\max_{i,j}|Z_{i  j}| > Q \log n \right)=1-\OO ( (\log n)^{-\varphi}).
\end{equation}
By definition, we have 
\be \label{EwtZ}
\begin{split}
 \E  \wt  Z_{ij} &= - \mathbb E \left[ \mathbf 1\left( |Z_{ij}|> Q \log n \right)Z_{ ij}\right] ,\\ 
\E  |\wt  Z_{ ij}|^2 & = 1 - \mathbb E \left[ \mathbf 1\left( |Z_{ ij}|> Q \log n \right)|Z_{ ij}|^2\right] .
\end{split}
\ee
Using the tail probability expectation formula, we can check that
\begin{align*}
  \mathbb E \left| \mathbf 1\left( |Z_{ ij}|> Q\log n \right)Z_{ ij}\right| &= \int_0^\infty \P\left( \left| \mathbf 1\left(  |Z_{ ij}|> Q\log n \right)Z_{ ij}\right| > s\right)\dd s \\
& = \int_0^{Q\log n}\P\left( |Z_{ ij}|> Q\log n \right)\dd s +\int_{Q\log n}^\infty \P\left(|Z_{ ij}| > s\right)\dd s  \\
& \lesssim \int_0^{Q\log n}\left(Q\log n \right)^{-\varphi}\dd s +\int_{Q\log n}^\infty s^{-\varphi}\dd s \le n^{-2(\varphi-1)/\varphi},
\end{align*}
where in the third step we use the finite $\varphi$-th moment condition \eqref{conditionA2} for $Z_{ij}$ %
and Markov's inequality. Similarly, we can obtain that
\begin{align*}
  \mathbb E \left| \mathbf 1\left( |Z_{ ij}|> Q \log n \right)Z_{ ij}\right|^2  &=  2\int_0^\infty s \P\left( \left| \mathbf 1\left( |Z_{ij}|>Q\log n \right)Z_{ij}\right| > s\right)\dd s \\
&=  2\int_0^{Q\log n} s \P\left( |Z_{ ij}|> Q\log n \right)\dd s +2\int_{Q\log n}^\infty s\P\left(|Z_{ ij}| > s\right)\dd s  \\
& \lesssim  \int_0^{Q\log n}s\left(Q\log n \right)^{-\varphi}\dd s +\int_{Q\log n}^\infty s^{-\varphi+1}\dd s \le n^{-2(\varphi-2)/\varphi}.
\end{align*}
Plugging the above two estimates into \eqref{EwtZ} and using $\varphi>4$, we get that
\be\label{meanshif}
|\mathbb E  \wt Z_{ij}| =\OO(n^{-3/2}), \quad  \mathbb E |\wt Z_{ij}|^2 =1+ \OO(n^{-1}).
\ee
From the first estimate in equation \eqref{meanshif}, we also get a bound on the operator norm:
\be\label{EZ norm}\|\E[\wt Z]\| = \OO(n^{-1/2}).
\ee
Then, we centralize and rescale $\wt Z$ as 
\be\label{center_truncateZ} \wh Z:=(\wt Z - \E \wt Z)/\sqrt{\E|\wt Z_{11}|^2} .\ee 
Now, $\wh Z$ is a matrix satisfying the assumptions of Lemma \ref{SxxSyy} with bounded support $Q$. Thus, we get that
\be\nonumber (\sqrt{n}-\sqrt{p})^2 - \OO_\prec(\sqrt{n}\cdot Q) \le  \lambda_p (\wh Z^\top \wh Z)  \le  \lambda_1(\wh Z^\top \wh Z) \le (\sqrt{n}-\sqrt{p})^2 + \OO_\prec(\sqrt{n}\cdot Q).
\ee
Combining this estimate with lines \eqref{meanshif} and \eqref{EZ norm}, we can readily get that equation \eqref{op rough2} holds for the eigenvalues of $\wt Z^\top \wt Z$, which concludes the proof by equation \eqref{XneX222}. 
\end{proof}

\begin{corollary} \label{cor_largedeviation}
Suppose $Z\in \R^{n\times p}$ is an $n\times p$ random matrix satisfying the same assumptions as $Z^{(2)}$ in Assumption \ref{assm_big1}. Then, there exists a high probability event on which the following estimate holds for any deterministic vector $v\in \R^p$: %
\be\label{Zv_cor}\left|\|Zv\|^2- n\|v\|^2\right|\prec n^{1/2}Q \|v\|^2,\quad \text{for} \quad Q=n^{2/\varphi}.\ee
\end{corollary}
\begin{proof}
Similar to the proof of Corollary \ref{fact_minv}, we truncate $Z$ as in \eqref{truncateZ} and define $\wh Z$ by \eqref{center_truncateZ}. By \eqref{meanshif} and \eqref{EZ norm}, we see that to conclude \eqref{Zv_cor}, it suffices to show that
\be\label{show_whZ}
 \left|\|\wh Zv\|^2- n\|v\|^2\right| \prec n^{1/2} Q\|v\|^2.
\ee 
To prove equation \eqref{show_whZ}, we first notice that %
$\wh{Z} v \in \R^n$ is a random vector with i.i.d. entries of mean zero and variance $\|v\|^2$. Furthermore, with equation \eqref{eq largedev10} we get that %
$$ |(\wh Zv)_i| \prec  Q\max_{j=1}^{p}|v_j| +  \|v\| \le 2 Q \|v\|,\quad i=1,2,\cdots,n. $$
Hence, $ {(\wh Zv)}/{\|v\|}$ consists of i.i.d.~random entries with zero mean, unit variance, and bounded support $Q$. Then, applying equation \eqref{eq largedev20}, we get that
$$ \left|\|\wh Zv\|^2 - n\|v\|^2\right|= \|v\|^2\left|\sum_{i=1}^n \left(\frac{|(\wh Zv)_i|^2}{\|v\|^2}-\E\frac{|(\wh Zv)_i|^2}{\|v\|^2}\right)\right|\prec Q n^{1/2}  \|v\|^2 .
 $$
This concludes the proof.
 \end{proof}
 
From Lemma \ref{largedeviation}, we can obtain the following concentration estimates for zero-mean noise vectors (which may also be found in standard texts such as \citet{wainwright2019high}).
 
\begin{corollary}\label{cor_calE}
Suppose $\ve^{(1)},\cdots,\ve^{(t)}\in \R^{n}$ are independent random vectors satisfying Assumption \ref{assm_big1} (ii).  
Then, we have that for any deterministic vector $v\in \R^n$,  
\be\label{vcalE}
	|v^\top \epsilon^{(i)}| \prec  \sigma \|v\| ,\quad i=1,\cdots, t, %
\ee
and for any deterministic matrix $B \in \real^{n\times n}$, %
\begin{align}\label{vcalA2}
	\left|(\epsilon^{(i)})^\top B \epsilon^{(j)} - \delta_{i,j}\cdot \sigma^2\tr[B]\right| \prec \sigma^2 \bignormFro{B}, ~\text{ for any }   i, j = 1, \dots, t,
\end{align}
where $\delta_{i,j} = 1$ if $i = j$, else $\delta_{i, j} = 0$ if $i \neq j$.
\end{corollary}
\begin{proof}
Note that $\epsilon^{(i)}/\sigma$ is a random vector with independent entries of zero mean, unit variance, and bounded moments up to any order. Then, the result in equation \eqref{vcalE} implies that equation \eqref{eq largedev0} holds.
Using equation \eqref{eq largedev1}, we obtain that for $i\ne j$,
 \begin{align*}
\left|(\epsilon^{(i)})^\top B \epsilon^{(j)}\right| = \Big|\sum_{k,l=1}^n \epsilon^{(i)}_k\epsilon^{(j)}_l B_{kl} \Big| \prec \sigma^2 \Big(\sum_{k,l=1}^n|B_{kl}|^2\Big)^{1/2}=\sigma^2 \|B\|_F.
 \end{align*}
Using the two estimates \eqref{eq largedev2} and \eqref{eq largedev3}, we obtain that 
\begin{align*}
\left|(\epsilon^{(i)})^\top B \epsilon^{(i)} - \sigma^2 \tr(B)\right| &\le \Big| \sum_{k}  \left(|\epsilon^{(i)}_k|^2-\E|\epsilon^{(i)}_k|^2\right)B_{kk} \Big|+\Big| \sum_{ k\ne l } \epsilon^{(i)}_k \epsilon^{(i)}_l B_{kl} \Big| \\
&\prec \sigma^2 \Big( \sum_{k=1}^p |B_{kk}|^2\Big)^{\frac 1 2}+ \sigma^2 \Big( \sum_{k\ne l}|B_{kl}|^2\Big)^{\frac 1 2} \le \sqrt 2 \sigma^2 \|B\|_F.
\end{align*}
Hence, we have proved that the estimate \eqref{vcalA2} hold.
\end{proof}

\subsection{Proof of Lemma \ref{lem_HPS_loss}}\label{app_firstpf}

\begin{proof}
For the optimization objective $\ell(B)$ in equation \eqref{eq_hps}, using the local optimality condition $\nabla \ell(B) = 0$, we can solve that
\begin{align*}
	\hat{B} = \hat \Sigma^{-1} \left({X^{(1)}}^{\top}Y^{(1)} +  {X^{(2)}}^{\top}Y^{(2)}\right), %
\end{align*}
where $\hat \Sigma$ is the sum of the sample covariance matrices (see definition in equation \eqref{Sigma_a}). %
Then, the HPS estimator for task two is equal to
$\hat{\beta}_2^{\MTL} = \hat B.$
Plugging it into the excess risk $L(\hat{\beta}_2^{\MTL})$, we find that it is equal to
\begin{align*}
	\left\| {\Sigma^{(2)}}^{1/2}\hat \Sigma^{-1} {X^{(1)}}^\top X^{(1)} (\beta^{(1)}- \beta^{(2)}) +{\Sigma^{(2)}}^{1/2}\hat \Sigma^{-1} \left( {X^{(1)}}^\top \epsilon^{(1)}+ {X^{(2)}}^\top \epsilon^{(2)}  \right)\right\|^2, %
\end{align*}
which can be expanded as $L_{\bias}+2h_1 +2 h_2 +h_3  +h_4  +2h_5$, where
\begin{align*}
& h_1 := {\epsilon^{(1)}}^\top X^{(1)} \hat \Sigma^{-1} \Sigma^{(2)} \hat \Sigma^{-1} {X^{(1)}}^\top X^{(1)} ( \beta^{(1)}-  \beta^{(2)})  ,\\
& h_2 :=   {\epsilon^{(2)}}^\top X^{(2)} \hat \Sigma^{-1} \Sigma^{(2)} \hat \Sigma^{-1} {X^{(1)}}^\top X^{(1)} (  \beta^{(1)}- \beta^{(2)})  ,\\
& h_3:= {\epsilon^{(1)}}^\top X^{(1)} \hat \Sigma^{-1} \Sigma^{(2)} \hat \Sigma^{-1}{X^{(1)}}^\top \epsilon^{(1)}   , \\
& h_4:= {\epsilon^{(2)}}^\top X^{(2)} \hat \Sigma^{-1} \Sigma^{(2)} \hat \Sigma^{-1}  {X^{(2)}}^\top \epsilon^{(2)} , \\
&h_5:={\epsilon^{(1)}}^\top X^{(1)} \hat \Sigma^{-1} \Sigma^{(2)} \hat \Sigma^{-1}{X^{(2)}}^\top \epsilon^{(2)}.
\end{align*}
We estimate these terms one by one using Lemma \ref{SxxSyy} and Corollary \ref{cor_calE}. 

First, under assumption \eqref{assm3}, we can obtain the following estimates using Corollary \ref{fact_minv}.
There exists a high probability event on which all of the following estimates hold
\be\label{Op_norm}
    \bignorms{X^{(1)}} \lesssim \max(\sqrt{n_1}, \sqrt{p}), \quad \bignorms{X^{(2)}}  \lesssim  \sqrt{n_2},\quad \bignorms{\hat\Sigma ^{-1}}  \lesssim (n_1 + n_2)^{-1}.
\ee
As for the last bound, we can derive it in more detail as follows (recall that $\lambda_1, \lambda_p$ refer to the largest and smallest singular values, respectively): using estimate \eqref{op rough2} and the assumptions $\rho_2\ge 1+\tau$ and $\rho_2\ge \tau \rho_1$ in equation \eqref{assm2}, we obtain that with high probability, 
\begin{align}
&\lambda_1(\hat\Sigma) \le \lambda_1({X^{(1)}}^{\top} X^{(1)}) + \lambda_1({X^{(2)}}^{\top} X^{(2)})\lesssim n_1+n_2, \label{Sigmahat_upper}\\
&\lambda_p(\hat\Sigma)\ge  \lambda_p({X^{(2)}}^{\top} X^{(2)}) \gtrsim n_2,\label{Sigmahat_lower}
\end{align}
which is again at the order of $n_1 + n_2$ since they are within a constant factor of each other.
With equation \eqref{Sigmahat_lower}, we conclude the last estimate in equation \eqref{Op_norm}:
\begin{align}
\bignorms{\hat\Sigma ^{-1}} = \frac 1 {\lambda_p(\hat\Sigma)} \lesssim \frac 1 {n_1 + n_2}.\label{eq_l1_shat}
\end{align}
In addition, based on the estimates \eqref{Sigmahat_upper}--\eqref{eq_l1_shat}, we can obtain that $L_{\bias}$ and $L_{\vari}$ (cf. equations \eqref{Lbias} and \eqref{Lvar}) are of the following orders (up to constants that do not grow with $p$ from both above and below):
\be\label{Op_norm2}
    L_{\bias}  \sim \frac{1}{( n_1+n_2)^2} \left\|{X^{(1)}}^{\top} X^{(1)}\left(  \beta^{(1)}-  \beta^{(2)}\right)  \right\|^2, \quad L_{\vari} \sim  \frac{p\sigma^2}{ n_1 + n_2}.
\ee
We first look at $L_{\vari}$. Notice that
\begin{align*}
    L_{\vari} = \sigma^2 \tr[\Sigma^{(2)} \hat\Sigma^{-1}] 
    &\le \sigma^2 p \bignorms{\Sigma^{(2)} \hat \Sigma^{-1}} \\
    &\le \sigma^2 p \bignorms{\Sigma^{(2)}} \cdot \bignorms{\hat \Sigma^{-1}} \\
    &\le \sigma^2 p \tau^{-1} \cdot \bignorms{\hat \Sigma^{-1}} \tag{by Assumption \eqref{assm3} on $\Sigma^{(2)}$} \\
    &\lesssim \frac{\sigma^2 p} {n_1 + n_2} \tag{by estimate \eqref{eq_l1_shat}}.
\end{align*}
Likewise, using the assumption that $\lambda_p(\Sigma^{(2)}) \ge \tau$, one can show that
\begin{align*}
    L_{\vari} \ge \sigma^2 p \tau \cdot \lambda_p(\hat\Sigma^{-1}) \gtrsim \frac{\sigma^2 p} {n_1 +n _2},
\end{align*}
where we have used that $\lambda_p(\hat\Sigma^{-1})=\lambda_1(\hat\Sigma)^{-1} \gtrsim (n_1 + n_2)^{-1}$ by estimate \eqref{Sigmahat_upper}. 
Next, we look at $L_{\bias}$. We first examine its upper estimate:
\begin{align*}
    L_{\bias} &= \bignorm{{\Sigma^{(2)}}^{\frac 1 2} \hat\Sigma^{-1} \big({X^{(1)}}^{\top} X^{(1)}\big) (\beta^{(1)} - \beta^{(2)})}^2 \\
    &\le \bignorms{\Sigma^{(2)}} \bignorms{\hat\Sigma^{-1}}^2 \bignorm{{X^{(1)}}^{\top} X^{(1)}\left(\beta^{(1)} - \beta^{(2)}\right)}^2 \\
    &\lesssim \frac 1{(n_1 + n_2)^{2}} \bignorm{{X^{(1)}}^{\top} X^{(1)}\left(\beta^{(1)} - \beta^{(2)}\right)}^2
    \tag{by Eq. \eqref{assm3} and estimate \eqref{eq_l1_shat}}. 
\end{align*}
Then, we can derive the lower estimate with similar arguments as follows:
\begin{align*}
   L_{\bias} &\ge \left(\lambda_p({\Sigma^{(2)}}^{\frac 1 2}\hat\Sigma^{-1})\right)^2 \cdot \bignorm{{X^{(1)}}^{\top}X^{(1)} (\beta^{(1)} - \beta^{(2)})}^2 \\
    &\gtrsim \frac{1}{(n_1 + n_2)^2} \bignorm{{X^{(1)}}^{\top} X^{(1)} (\beta^{(1)} - \beta^{(2)})}^2 \tag{by Eq. \eqref{assm3} and estimate \eqref{Sigmahat_upper}} .
\end{align*}
In the first step, we use the fact that $\bignorm{X v} \ge (\lambda_{\min}(X))^2 \bignorm{v}^2$ for any matrix $X$ and vector $v$. In the second step, we use that $u^{\top}\hat\Sigma^{-1}\Sigma^{(2)}\hat\Sigma^{-1} u \ge \lambda_{\min} (\Sigma^{(2)})\cdot  u^{\top} \hat\Sigma^{-2} u$ for any vector $u$, so $\lambda_p({\Sigma^{(2)}}^{\frac 1 2} \hat\Sigma^{-1})\gtrsim (n_1 + n_2)^{-1}$.
Taken together, we have proved the estimates in \eqref{Op_norm2}.

Note that if we further assume that $\rho_1\ge 1+\tau$ as in Theorem \ref{prop_lb}, then by Assumption \eqref{assm3} on $\Sigma^{(1)}$ and the estimate \eqref{op rough2}, we have that with high probability,
$n_1\lesssim \lambda_p({X^{(1)}}^{\top} X^{(1)})\le \lambda_1({X^{(1)}}^{\top} X^{(1)})\lesssim n_1.$
Consequently, the estimation of $L_{\bias}$ in equation \eqref{Op_norm2} simplifies to
\be\label{Op_norm2_add}
    L_{\bias}  \sim \frac{n_1^2}{( n_1+n_2)^2} \left\|\beta^{(1)}-  \beta^{(2)} \right\|^2.
\ee
This equation has been used to establish the estimate \eqref{eq:showorder1}. 

The rest of the proof is always restricted to the high probability event $\Xi$, on which the estimates \eqref{Op_norm}--\eqref{Op_norm2} all hold true. 
Using equation \eqref{vcalE} in Corollary \eqref{cor_calE}, we can bound $h_1$ on $\Xi$ as
\begin{align}
    \left|h_1 \right| &\prec  \sigma  \left\|  X^{(1)} \hat \Sigma^{-1} \Sigma^{(2)} \hat \Sigma^{-1} {X^{(1)}}^\top X^{(1)} ( \beta^{(1)}- \beta^{(2)})\right\|  \nonumber\\
    &\le  \sigma  \left\| {X^{(1)}}^\top X^{(1)}\left( \beta^{(1)}- \beta^{(2)} \right)\right\|\cdot \bignorms{  X^{(1)} \hat \Sigma^{-1} \Sigma^{(2)} \hat \Sigma^{-1} } \nonumber \\
    &\lesssim {\sigma   \left\| {X^{(1)}}^\top X^{(1)}\left( \beta^{(1)}- \beta^{(2)} \right)\right\|} \cdot  \frac{ \max(\sqrt{n_1}, \sqrt{p})} {(n_1 + n_2)^2} \tag{by estimate \eqref{Op_norm}}\\
    &\le \frac{ p^{-\frac 1 4} \big\| {X^{(1)}}^\top X^{(1)}\left( \beta^{(1)}- \beta^{(2)} \right)\big\| }{n_1 + n_2 }\cdot \frac{p^{\frac 1 4} \sigma  }{(n_1 + n_2)^{\frac 1 2}} \tag{since $\max(\sqrt{n_1}, \sqrt p) \le \sqrt{n_1 + n_2}$} \\ %
    &\le \frac{ p^{-\frac 1 2} \big\| {X^{(1)}}^\top X^{(1)}\left( \beta^{(1)}- \beta^{(2)} \right)\big\|^2  }{ ( n_1 + n_2)^2 }+ \frac{\sigma^2 \sqrt p}{ n_1 + n_2} \tag{by the AM-GM inequality}\\
    &\lesssim \frac{L_{\bias} + L_{\vari}}{\sqrt p}. \tag{by estimate \eqref{Op_norm2}}
\end{align}
With a similar argument, we can show that $\left|h_2  \right|  \prec \frac{L_{\bias}  + L_{\vari}}{\sqrt p}.$

Next, using equation \eqref{vcalA2} in Corollary \ref{cor_calE}, we can estimate $h_3$ as follows
\begin{align*}
	\left|h_3  -\sigma^2   \bigtr{X^{(1)} \hat \Sigma^{-1} \Sigma^{(2)} \hat \Sigma^{-1}{X^{(1)}}^\top } \right| 
   & \prec  \sigma^2   \left\| X^{(1)} \hat \Sigma^{-1} \Sigma^{(2)} \hat \Sigma^{-1}{X^{(1)}}^\top \right\|_F \\ 
    &\lesssim  \frac{\sigma^2 }{( n_1+ n_2)^2}  \left\| X^{(1)}  {X^{(1)}}^\top \right\|_F \tag{by equations \eqref{assm3} and \eqref{Op_norm}}\\
	&\lesssim \frac{\sigma^2  p^{\frac 1 2} \max(n_1, p)}{(n_1+ n_2)^2} \tag{by estimate \eqref{Op_norm}}\\
    &\lesssim \frac{L_{\vari}}{\sqrt p}, \tag{by estimate \eqref{Op_norm2}} %
\end{align*}
where we use the following estimate in the third line
$$ \left\| X^{(1)}  {X^{(1)}}^\top \right\|_F \le \sqrt p \left\| X^{(1)}  {X^{(1)}}^\top \right\|_2 \lesssim \sqrt p \max(n_1, p).$$
With a similar argument, we can show that
\begin{align*}
	 &\left|h_4 -\sigma^2   \bigtr{X^{(2)} \hat \Sigma^{-1} \Sigma^{(2)} \hat \Sigma^{-1}  {X^{(2)}}^\top} \right| \prec   \frac {L_{\vari}}{\sqrt p},
\end{align*}
and $\left|h_5   \right| \prec  {L_{\vari}}/{\sqrt p}.$

Finally, combining the above estimates on $h_1, h_2, h_3, h_4, h_5$, and using that
\begin{align*}
& \sigma^2  \bigtr{X^{(1)} \hat \Sigma^{-1} \Sigma^{(2)} \hat \Sigma^{-1}{X^{(1)}}^\top }  + \sigma^2   \bigtr{X^{(2)} \hat \Sigma^{-1} \Sigma^{(2)} \hat \Sigma^{-1}  {X^{(2)}}^\top}  =  \sigma^2 \bigtr{ \Sigma^{(2)} \hat \Sigma^{-1}  }=L_{\vari},
\end{align*}
we conclude that conditioned on event $\Xi$, the following estimate holds true:
\begin{align}\label{large_devh1}
    \bigabs{L(\hat{\beta}_2^{\MTL}) - L_{\bias} - L_{\vari}} \prec \frac{L_{\bias} + L_{\vari}}{\sqrt p}.
\end{align}
This concludes the proof of equation \eqref{L_HPS_simple}.
\end{proof}

\section{Proofs for the Covariate Shift Setting}\label{proof_sec_cov}

Before presenting the proof of Theorem \ref{thm_main_RMT}, we first use its conclusion to complete the proofs of Propositions \ref{claim_dichotomy} and \ref{prop_covariate}. 

\subsection{Proof of Proposition \ref{claim_dichotomy}}\label{sec:claim_dich}

\begin{proof}
    By definition, we can expand $g(M)$ as
    $$g(M)=\frac{\sigma^2}{ n_1+n_2 }\sum_{i=1}^{\frac p 2}\left( \frac{1}{\lambda_i^{2} \alpha_1 + \alpha_2} + \frac1{\lambda_i^{-2} \alpha_1 + \alpha_2} \right).$$
    When $M=\id_{p\times p}$, we have
    $g(\id_{p\times p}) = \frac{\sigma^2 p}{n_1 + n_2 - p},$
		which can also be written as $\frac{\sigma^2}{ n_1+n_2 }\sum_{i=1}^{\frac p 2} \frac{2}{\alpha_1 + \alpha_2}$, since $\alpha_1 + \alpha_2 = 1 - \frac{p}{n_1 + n_2}$ by equation \eqref{eq_a12extra}.
    Then, we subtract $g(\id_{p\times p})$ from $g(M)$ to get
    \begin{align*}
        g(M) - g(\id_{p\times p})%
        &= \frac{\sigma^2 }{ n_1+n_2-p} \sum_{i=1}^{p/2} \frac{(\lambda_i^2-1)^2 \alpha_1 \left( \alpha_1 - \alpha_2\right) }{(\alpha_1 + \lambda_i^2 \alpha_2)(\lambda_i^2 \alpha_1 + \alpha_2)} .
    \end{align*}
    We claim that $\alpha_1 > \alpha_2$ if and only if $n_1 > n_2$, thus proving the dichotomy.
    When $\alpha_1 > \alpha_2$, the first part of equation \eqref{eq_a12extra} implies $\alpha_1 > \frac{1}{2}(1 - \frac{p}{n_1 + n_2})$.
    The second part of equation \eqref{eq_a12extra} gives
    \begin{align*}
			\frac{n_1}{n_1 + n_2} &= \alpha_1 + \frac{1}{n_1 + n_2} \sum_{i=1}^p \frac{\lambda_i^2 \alpha_1}{\lambda_i^2 \alpha_1 + \alpha_2} \\
      &> \frac1 2\Big(1 - \frac{p}{n_1 + n_2}\Big) + \frac{1}{n_1+n_2} \sum_{i=1}^{p/2}\left(\frac{\lambda_i^2}{\lambda_i^2+1}+\frac{\lambda_i^{-2}}{\lambda_i^{-2}+1}\right) \\
			&= \frac 1 2\Big(1 - \frac{p}{n_1 + n_2}\Big) +\frac{p}{2(n_1 + n_2)}=\frac{1}{2},
    \end{align*}
    implying $n_1 > n_2$.
		When $\alpha_1 < \alpha_2$, one can show $n_1 < n_2$ using similar steps.
	This completes the proof.
\end{proof}

\subsection{Proof of Proposition \ref{prop_covariate}}\label{appendix RMT0}

\begin{proof}%
From the second part of equation \eqref{eq_a12extra}, we get
    \begin{align*}
        \frac{n_1} {n_1 + n_2} = \alpha_1 + \frac 1 {n_1 + n_2}  {\sum_{i=1}^p \frac{\lambda_i^2\alpha_1}{\lambda_i^2\alpha_1 + \alpha_2}}
        < \alpha_1 + \frac p {n_1 + n_2},
    \end{align*}
    which gives that $\alpha_1 > \frac{n_1 - p}{n_1 + n_2}$. Combined with the first part of equation \eqref{eq_a12extra}, it yields that 
  $$\alpha_2 =1-\frac{p}{n_1+n_2} - \al_1< \frac{n_2}{n_1 + n_2}.$$
Hence, we have $\al_1/\al_2>(n_1-p)/n_2\ge c^{-1}$. 

Now, abbreviating $\bx_p=(x_1,\ldots, x_p)$ with $x_i=\lambda_i^2$, the trace of $(\alpha_1 M^{\top} M + \alpha_2)^{-1}$ can be written as  
    \begin{align*}
      f_p(\bx_p) := \bigtr{\frac{1}{\alpha_1 M^{\top} M + \alpha_2}} &= \sum_{i=1}^p \frac{1} {\alpha_1 x_i + \alpha_2} .
    \end{align*}
To conclude the proof, it suffices to prove the following claim: for any $k \in \{1,\ldots, p\}$, $c\le b\le 1$, and $\al_1, \al_2$ such that $\al_1/\al_2\ge c^{-1}$,
$f_k(\bx_k)$ is minimized at $\bx_k=(b,\ldots,b)\in \R^k$ subject to the constraints $c\le x_i \le c^{-1}$ and $\prod_{i=1}^k x_i=b^k$. Note that in contrast to the original setting, $\al_1$ and $\al_2$ are now taken as fixed parameters that do not depend on $\bx_k$, which will greatly simplify the optimization problem.

We prove the above claim by induction. First, the claim is trivial when $k=1$. For $k=2$, a similar argument as in the proof of Proposition \ref{claim_dichotomy} indeed gives that $f_2$ is minimized at $x_1=x_2=b$ as long as $b\al_1\ge \al_2$. Now, suppose the claim holds for $f_k$ for some $2\le k \le p-1$. Consider $\bx_{k+1}=(x_1\ldots, x_{k+1})$ with $c\le x_i \le c^{-1}$ and $\prod_{i=1}^{k+1}x_i=b^{k+1}$. Renaming the indices if necessary, we may assume that $x_{k+1}$ is the largest entry, i.e., $ x_{k+1}\ge \max_{i=1}^{k} x_i$. Then, we can write $x_{k+1}=ba$ for some $1\le a\le \min\{(bc)^{-1} , b^{k}c^{-k}\}$, where the condition $a\le b^{k}c^{-k}$ comes from the constraints $\prod_{i=1}^{k+1}x_i=b^{k+1}$ and $\prod_{i=1}^{k}x_i \ge c^k$. The remaining variables satisfy $c\le x_i \le c^{-1}$ and $\prod_{i=1}^k x_i=b^k/a$. Applying the induction hypothesis for $f_k$, we get that 
$$\sum_{i=1}^k \frac{1}{\al_1 x_i + \al_2}  \ge \frac{k}{\al_1 b a^{-1/k}  + \al_2}.$$
Hence, for $\bx_{k+1}=(x_1, \ldots, x_k, ba)$, we have
$$f_{k+1}(\bx_{k+1})\ge g_k(a):=\frac{k}{\al_1 b a^{-1/k}  + \al_2} + \frac{1}{\al_1 b a + \al_2}.$$
To conclude the proof, we only need to show that $g_k(a)$ achieves a minimum when $a=1$, i.e, 
$$g_k(a)\ge \frac{k+1}{\al_1 b +\al_2},$$
which is equivalent to 
\be\label{eq_wbk} \frac{k(a^{1/k}-1)}{w +a^{1/k}} \ge \frac{a-1}{w a +1} ,\ee
where we introduce the simplified notation $w=b \al_1/\al_2> bc^{-1} \ge 1$. We rewrite the inequality \eqref{eq_wbk} as
$ h_k(a) \ge 0$ with 
$$h_k(a):=k(a^{\frac{1}{k}}-1)(wa+1)-(w +a^{\frac{1}{k}})(a-1).$$
Its derivative satisfies
$$ h_k'(a)=w(k+1)(a^{\frac{1}{k}}-1)- \frac{k+1}{k} a^{\frac{1}{k}}(1-a^{-1})\ge 0,$$
where, in the second step, we apply that 
$k (1-a^{-\frac{1}{k}})\ge 1-a^{-1}.$ 
Hence, we have $h_k(a)\ge h_k(1)=0$ for $a\ge 1$, i.e., the inequality of \eqref{eq_wbk} indeed holds, which concludes the proof.
\end{proof}

\subsection{Proof of Theorem \ref{thm_main_RMT}}\label{appendix RMT}
For the rest of this section, we present the proof of Theorem \ref{thm_main_RMT}, which is one of the main results of this paper. The central quantity of interest is the matrix $\hat \Sigma^{-1}$ in equation \eqref{Sigma_a}. Assume that $M ={\Sigma^{(1)}}^{1/2}{\Sigma^{(2)}}^{-1/2}$ has a singular value decomposition
\begin{equation*}%
M = U\Lambda V^\top, \quad \text{where} \ \ \Lambda :=\text{diag}( \lambda_1 , \ldots, \lambda_p ).
\end{equation*}
Then, we can write equation \eqref{Lvar} as
\begin{align}\label{eigen2extra}
    L_{\vari} = \sigma^2 \tr[\Sigma^{(2)} \hat{\Sigma}^{-1}] = \frac{\sigma^2 }{n} \tr[W^{-1}],
\end{align}
where we denote $n:=n_1+n_2$ and
$$ W:=n^{-1}\left(  \Lambda  U^\top {Z^{(1)}}^\top Z^{(1)} U\Lambda   + V^\top {Z^{(2)}}^\top Z^{(2)}V\right).$$
The \emph{resolvent} or \emph{Green's function} of $W$ is defined as $(W - z\id_{p\times p})^{-1}$ for $z\in \C$.

In this section, we will prove a local convergence of this resolvent with a \emph{sharp convergence rate}, which is conventionally referred to as ``the local law''  \citep{isotropic,erdos2017dynamical,Anisotropic}.
In particular, the key technical challenge is to establish an (almost) sharp convergence rate of the bias and variance equations to their asymptotic limits.

\subsubsection{Resolvent and Local Law}\label{sec pf RMTlemma}

We can write $W$ as $W=\AF\AF^{\top}$ for a $p\times n$ matrix
	\be\label{defn AF} \AF := n^{-1/2} [ \Lambda U^\top {Z^{(1)}}^\top,V^\top {Z^{(2)}}^\top]. \ee
We introduce a convenient self-adjoint linearization trick to study $W$. It has been proved to be useful in studying random matrices of Gram type \cite{Anisotropic}. 

\begin{definition}[Self-adjoint linearization and resolvent]\label{defn_resolventH}
We define the following $(p+n)\times (p+n)$ symmetric block matrix %
 \begin{equation}\label{linearize_block}
    H \define \left( {\begin{array}{*{20}c}
   0 & \AF  \\
   \AF^{\top} & 0
   \end{array}} \right),
 \end{equation}
and its resolvent as
$$G(z) \equiv G(Z^{(1)}, Z^{(2)},z)\define \left[H - \begin{pmatrix}z\id_{p\times p}&0\\ 0 & \id_{(n_1+n_2)\times (n_1+n_2)} \end{pmatrix}\right]^{-1},\quad z\in \mathbb C , $$
as long as the inverse exists. Furthermore, we define the following (weighted) partial traces %
\be\label{defm}
\begin{split}
m(z) :=\frac1p\sum_{i\in \cal I_0} G_{ii}(z) ,\quad & m_0(z):=\frac{1}p\sum_{i\in \cal I_0} \lambda_i^2 G_{ii}(z),\\
 m_1(z):= \frac{1}{n_1}\sum_{\mu \in \cal I_1}G_{\mu\mu}(z) ,\quad & m_2(z):= \frac{1}{n_2}\sum_{\nu\in \cal I_2}G_{\nu\nu}(z),
\end{split}
\ee
where $\cal I_i$, $i=0,1,2$, are index sets defined as 
$$\cal I_0:=\llbracket 1,p\rrbracket, \quad  \cal I_1:=\llbracket p+1,p+n_1\rrbracket, \quad \cal I_2:=\llbracket p+n_1+1,p+n_1+n_2\rrbracket.$$
\end{definition}
We will consistently use latin letters $i,j\in\sI_{0}$ and greek letters $\mu,\nu\in\sI_{1}\cup \sI_{2}$.
Correspondingly, the indices $Z^{(1)}$ and $Z^{(2)}$ are labelled as
	\begin{equation*}%
 Z^{(1)}= \left[Z^{(1)}_{\mu i}:i\in \mathcal I_0, \mu \in \mathcal I_1\right], \quad Z^{(2)}= \left[Z^{(2)}_{\nu i}:i\in \mathcal I_0, \nu \in \mathcal I_2\right].\end{equation*}
Moreover, we define the set of all indices $\cal I:=\cal I_0\cup \cal I_1\cup \cal I_2$, and label the indices in $\cal I$ as $\fa, \ \fb,\ \mathfrak c$ and so on. 

Using the Schur complement formula for the inverse of a block matrix, we get that
	\begin{equation} \label{green2}
	  G(z) =  \left( {\begin{array}{*{20}c}
			(W- z\id)^{-1} & (W - z\id)^{-1} \AF  \\
      \AF^\top (W - z\id)^{-1} & z(\AF^\top \AF - z\id)^{-1}
		\end{array}} \right).%
  \end{equation}
In particular, the upper left block of $G$ is exactly the resolvent of $W$ we are interested in. Compared with $(W- z\id)^{-1}$, it turns out that $G(z)$ is more convenient to deal with because $H$ is a linear function of $Z^{(1)}$ and $Z^{(2)}$. 
This is why we have chosen to work with $G(z)$.

We define the matrix limit of $G(z)$ as 
\be \label{defn_piw}
	\Gi(z) \define \begin{pmatrix} [\al_{1}(z) \Lambda^2  +  (\al_{2}(z)- z)]^{-1} & 0 & 0 \\ 0 & - \frac{n}{n_1} \al_{1}(z)\id_{n_1\times n_1} & 0 \\ 0 & 0 & -\frac{n}{n_2}\al_{2}(z)\id_{n_2\times n_2}  \end{pmatrix},\ee
where $(\al_1(z),\al_2(z))$ is the unique solution to the following system of self-consistent equations:
\be\label{selfomega_a}
\begin{split}
	&\al_1(z) + \al_2(z) = 1 - \frac{1}{n} \sum_{i=1}^p \frac{\lambda_i^2 \al_1(z) + \al_2(z)}{\lambda_i^2  \al_1(z) + \al_2(z) - z}, \\ %
	&\al_1(z) + \frac{1}{n}\sum_{i=1}^p \frac{\lambda_i^2  \al_1(z)}{\lambda_i^2  \al_1(z) + \al_2(z) - z} = \frac{n_1}{n},
\end{split}
\ee
such that $\im \al_1(z)\le 0$ and $\im \al_2(z)\le 0$ whenever $\im z > 0$. The existence and uniqueness of solutions to the above system will be proved in Lemma \ref{lem_mbehaviorw}.

We now state the main result, Theorem \ref{LEM_SMALL}, of this section, which shows that for $z$ in a small neighborhood around $0$, $G(z)$ converges to the limit $\Gi(z)$ when $p$ goes to infinity. Moreover, it also gives an almost sharp convergence rate of $G(z)$. Such an estimate is conventionally called an {\it anisotropic local law} \citep{Anisotropic}. We define a domain of the spectral parameter $z$ as
\begin{equation*}
\mathbf D:= \left\{z=E+ \ii \eta \in \C_+: |z|\le (\log n)^{-1} \right\}. %
\end{equation*}

\begin{theorem} \label{LEM_SMALL} %
Suppose that $Z^{(1)}$, $Z^{(2)}$, $\rho_1$ and $\rho_2$ satisfy Assumption \ref{assm_big1}. Suppose that $Z^{(1)}$ and $Z^{(2)}$ satisfy the bounded support condition \eqref{eq_support} with $Q= n^{{2}/{\varphi}}$. Suppose that the singular values of $M$ satisfy that
\begin{equation}\label{assm3_app}
 \lambda_p \le\cdots\le\lambda_2 \le \lambda_1 \le \tau^{-1} .
\end{equation}
Then, the following local laws hold. 
\begin{itemize}
\item[(1)] {\bf Averaged local law}: We have that
\begin{align}
\sup_{z\in \mathbf D} \bigg\vert {p}^{-1}\sum_{i\in \cal I_0} [G_{ii}(z)- \Gi_{ii}(z)]\bigg\vert &\prec (np)^{-1/2} Q, \label{aver_in} \\ 
\sup_{z\in \mathbf D} \bigg\vert {p}^{-1}\sum_{i\in \cal I_0} \lambda_i^2 [G_{ii}(z)- \Gi_{ii}(z)]\bigg\vert &\prec (np)^{-1/2} Q. \label{aver_in1} %
\end{align}
\item[(2)] {\bf Anisotropic local law}: For any deterministic unit vectors $\mathbf u, \mathbf v \in \mathbb R^{p+n}$, we have that
\begin{equation}\label{aniso_law}
	\sup_{z\in \mathbf D}\left| \mathbf u^\top [G(z)-\Gi(z)] \mathbf v \right|  \prec  n^{-1/2}Q.
\end{equation}
\end{itemize}
\end{theorem}

With Theorem \ref{LEM_SMALL}, we can complete the proof of Theorem \ref{thm_main_RMT}  with a standard cutoff argument.

\begin{proof} %
Similar to equation \eqref{truncateZ}, we introduce the truncated matrices $\wt Z^{(1)}$ and $\wt Z^{(2)}$ with entries
\be\label{truncate1} 
\wt Z^{(1)}_{\mu i}:= \mathbf 1\left(  |Z^{(1)}_{\mu i}|\le Q \log n \right)\cdot Z^{(1)}_{\mu i}, \quad \wt Z^{(2)}_{\nu i}:= \mathbf 1\left(  |Z^{(2)}_{\nu i}|\le Q\log n \right)\cdot Z^{(2)}_{\nu i}, 
\ee
for $Q= n^{{2}/{\varphi}}$. From equation \eqref{Ptrunc},  we get that
\begin{equation*}%
\mathbb P(\wt Z^{(1)} = Z^{(1)},  \wt Z^{(2)} = Z^{(2)}) =1-\OO ( (\log n)^{-\varphi}).
\end{equation*}
By equations \eqref{meanshif} and \eqref{EZ norm}, we have that
\be \label{meanshif2}
\begin{split}
  |\mathbb E  \wt  Z^{(1)}_{\mu i}| =\OO(n^{-3/2}), \quad &|\mathbb E  \wt  Z^{(2)}_{\nu i}| =\OO(n^{-3/2}), \\ 
 \mathbb E |\wt  Z^{(1)}_{\mu i}|^2 =1+ \OO(n^{-1}),\quad &\mathbb E |\wt  Z^{(2)}_{\nu i}|^2 =1+ \OO(n^{-1}),
\end{split}
\ee
and %
\be  \label{EZ norm2}
\|\E \wt Z^{(1)}\|=\OO(n^{-1/2}),\quad \|\E \wt Z^{(2)}\|=\OO(n^{-1/2}) .
\ee
Then, we centralize and rescale $\wt Z^{(1)}$ and $\wt Z^{(2)}$ as
$$ \wh Z^{(1)} := (\E|\wt Z^{(1)}_{\mu i}|^2 )^{-1/2}(\wt Z^{(1)} - \E \wt Z^{(1)} ),\quad \wh Z^{(2)} := (\E|\wt Z^{(2)}_{\nu i}|^2)^{-1/2} (\wt Z^{(2)} - \E \wt Z^{(2)} ),$$
for arbitrary $\mu\in \cal I_1$, $\nu\in \cal I_2$ and $i\in \cal I_0$.  
Now, $\wh Z^{(1)}$ and $\wh Z^{(2)}$ satisfy the assumptions of Theorem \ref{LEM_SMALL}. %
Hence, $G(\wh Z^{(1)},\wh Z^{(2)},z)$ satisfies equation \eqref{aver_in}, where $G(\wh Z^{(1)},\wh Z^{(2)},z)$ is defined in the same way as $G(z)$ but with $(Z^{(1)}, Z^{(2)})$ replaced by $(\wh Z^{(1)},\wh Z^{(2)})$. 
By equations \eqref{meanshif2} and \eqref{EZ norm2}, we have that for $k=1,2,$
$$ \|\wh Z^{(k)} - \wt Z^{(k)}\|\lesssim n^{-1}\|\wh Z^{(k)}\| + \|\E \wt Z^{(k)}\|\prec n^{-1/2} ,$$
where we use Lemma \ref{SxxSyy} to bound the operator norm of $\wh Z^{(k)}$. Together with the estimate \eqref{priorim} below, this bound implies that 
$$\max_{i\in \cal I_0}\left|  G_{ii}(\wh Z^{(1)},\wh Z^{(2)},z)-G_{ii}(\wt Z^{(1)},\wt Z^{(2)},z)  \right|  \prec  n^{-1/2} \sum_{k=1}^2 \|\wh Z^{(k)} - \wt Z^{(k)}\| \prec n^{-1}.$$
Combining this estimate with the local law \eqref{aver_in} for $G(\wh Z^{(1)},\wh Z^{(2)},z)$, we obtain that \eqref{aver_in} also holds for $G(z)\equiv G( Z^{(1)},  Z^{(2)},z)$ on the event $\Xi_1:=\{\wt Z^{(1)} = Z^{(1)},  \wt Z^{(2)} = Z^{(2)}\}$. 

Now, we are ready to prove equation \eqref{lem_cov_shift_eq}. By estimate \eqref{eigen2extra}, we have that
$$L_{\vari}(a)=\frac{ \sigma^2 }{n}\sum_{i\in \cal I_0} G_{ii}(0).$$
We also notice that the equations in estimate \eqref{selfomega_a} reduce to the equations in equation \eqref{eq_a12extra} when $z=0$, which gives that $\al_1 = \al_1(0)$ and $\al_2 = \al_2(0)$. Hence, the partial trace of $\Gi$ in \eqref{defn_piw} over $\cal I_0$ is equal to
$$\sum_{i\in \cal I_0}\Gi_{ii}(0)= \bigtr{\frac1{\al_1    \Lambda^2 + \al_2}}=\bigtr{  \frac{1}{\al_1  M^\top M + \al_2  }  }.$$
Now, applying estimate \eqref{aver_in} to $G(z)$, we conclude that on the event $\Xi_1$,
$$ \bigabs{L_{\vari} - \frac{\sigma^2}{n_1+n_2}\bigtr{  \frac{1}{\al_1   M^\top M+ \al_2  }  }}
				\prec \frac{n^{2/\varphi}}{(np)^{1/2}}\cdot\frac{p \sigma^2}{n_1+ n_2}.  $$
This concludes the proof of Theorem \ref{thm_main_RMT}. 
 \end{proof}

\subsubsection{Self-consistent Equations}\label{sec contract}

\begin{proof}
The rest of this section is devoted to the proof of Theorem \ref{LEM_SMALL}. In this subsection, we show that the self-consistent equation \eqref{selfomega_a} has a unique solution $(\al_1(z), \al_2(z))$ for any $z\in \mathbf D$. Otherwise, Theorem \ref{LEM_SMALL} will be a vacuous result. For simplicity of notation, we define the following ratios 
\begin{equation*}%
 \gamma_n :=\frac{p}{n} ,\quad r_1 :=\frac{n_1}{n} ,\quad r_2 :=\frac{n_2}{n} .
\end{equation*}

When $z=0$, estimate \eqref{selfomega_a} reduces to the system of equations in \eqref{eq_a12extra}, from which we can derive an equation of $\al_1\equiv \al_1(0)$ only:
\be\label{fa1}f(\al_1)=r_1,\quad \text{with}\quad f(\al_1):=\al_1 +\frac{1}{n} \sum_{i=1}^p \frac{\lambda_i^2  \al_1}{ \lambda_i^2  \al_1+ (1- \gamma_n - \al_1) } .\ee
We can calculate that
$$f'(\al_1)=1+\frac{1}{n} \sum_{i=1}^p \frac{\lambda_i^2 (1-\gamma_n)}{ [\lambda_i^2  \al_1+ (1- \gamma_n - \al_1)]^2 }>0.$$
Hence, $f$ is strictly increasing on $[0,1-\gamma_n]$. Moreover, we have $f(0)=0<r_1$, $f(1-\gamma_n )=1>r_1$, and $f(r_1)>r_1$ if $r_1\le 1-\gamma_n$. Hence, there exists a unique solution $\al_1$ to \eqref{fa1} satisfying
$0< \al_1 <\min\left\{1-\gamma_n, r_1\right\}.$ Furthermore, using that $f'(x)=\OO(1)$ for any fixed $x\in (0, 1-\gamma_n)$, it is not hard to check that 
\be\label{a230}
r_1 \tau  \le   \al_1 \le \min\left\{ 1-\gamma_n,r_1\right\}  
\ee 
for a small constant $\tau>0$. From equation \eqref{eq_a12extra}, we can also derive a equation of $\al_2\equiv \al_2(0)$ only. With a similar argument as above, we get that 
\be\label{a231}
r_2 \tau   \le \al_2\le  \min\left\{ 1-\gamma_n,r_2\right\} 
\ee 
for a small constant $\tau>0$.

Next, we prove the existence and uniqueness of the solution to the self-consistent equation \eqref{selfomega_a} for a general $z\in \mathbf D$. For the proof of Theorem \ref{LEM_SMALL}, it is more convenient to use the following rescaled functions of $\al_1(z)$ and $\al_2(z)$:
\be\label{M1M2a1a2}
m_{1c}(z):= - r_1^{-1}\al_1(z),\quad m_{2c}(z):= - r_2^{-1}\al_2(z),
\ee
which, as we will see later, are the classical values (i.e., asymptotic limits) of $m_1(z)$ and $m_2(z)$, respectively. 
Moreover, it is not hard to check that equation \eqref{selfomega_a} is equivalent to the following system of self-consistent equations of $m_{1c}(z),m_{2c}(z)$:
\begin{equation}\label{selfomega}
\begin{split}
& \frac1{m_{1c}} = \frac{\gamma_n}p\sum_{i=1}^p \frac{\lambda_i^2}{  z+\lambda_i^2  r_1 m_{1c} +r_2 m_{2c}  } - 1 ,\\
& \frac1{m_{2c}} = \frac{\gamma_n}p\sum_{i=1}^p \frac{1 }{  z+\lambda_i^2 r_1 m_{1c} +  r_2 m_{2c}  }- 1 .
\end{split}
\end{equation}
When $z=0$, using equations \eqref{eq_a12extra}, \eqref{a230}, and \eqref{a231}, we get that
\be\label{a23}
\begin{split}
\tau  \le  - m_{1c}(0) \le 1,\quad \tau  \le  - m_{2c}(0) \le 1,\quad -r_1 m_{1c}(0)-r_2 m_{2c}(0)=1-\gamma_n.
\end{split}
\ee 

Now, we claim the following lemma, which gives the existence and uniqueness of the solutions $m_{1c}(z)$ and $m_{2c}(z)$ to the system of equations \eqref{selfomega}.

\begin{lemma} \label{lem_mbehaviorw}
There exist constants $c_0, C_0>0$ depending only on $\tau$ in Assumption \ref{assm_big1} and equation \eqref{a23} such that the following statements hold.
There exists a unique solution $m_{1c}(z), m_{2c}(z)$ to equation \eqref{selfomega} under the conditions
\be\label{prior1}
|z|\le c_0, \quad  |m_{1c}(z) - m_{1c}(0)| + |m_{2c}(z) - m_{2c}(0)|\le c_0.
\ee
Moreover, the solution satisfies
\be\label{Lipomega}
 |m_{1c}(z) - m_{1c}(0)| + |m_{2c}(z) - m_{2c}(0)| \le C_0|z|.
\ee
\end{lemma}

\begin{proof} %
The proof is based on a standard application of the contraction principle. 
First, it is easy to check that the system of equations in equation \eqref{selfomega} is equivalent to  
\begin{equation}\label{selfalter}
r_1m_{1c}=-(1-\gamma_n) - r_2m_{2c} - z\left( m_{2c}^{-1}+1\right),\quad g_z(m_{2c}(z))=1, 
\end{equation}
where
$$g_z(m_{2c}):= - m_{2c} +\frac{\gamma_n}p\sum_{i=1}^p \frac{m_{2c} }{  z -\lambda_i^2(1-\gamma_n)+ (1 - \lambda_i^2) r_2m_{2c} - \lambda_i^2 z\left(  m_{2c}^{-1}+1\right) }.$$
We first show that there exists a unique solution $m_{2c}(z)$ to the equation $g_z(m_{2c}(z))=1$ under the conditions in equation \eqref{prior1}.
We abbreviate $\delta(z):= m_{2c}(z) - m_{2c}(0)$. From equation \eqref{selfalter}, we obtain that 
\begin{equation} \nonumber
0=\left[g_z(m_{2c}(z)) -  g_0(m_{2c}(0)) -g_z'(m_{2c}(0))\delta(z)\right] + g_z'(m_{2c}(0))\delta(z),
\end{equation}
which gives that
\be\nonumber%
 \delta(z) =- \frac{ g_z(m_{2c}(0)) - g_0(m_{2c}(0)) }{g_z'(m_{2c}(0))}- \frac{ g_z(m_{2c}(0)+\delta(z)) -  g_z(m_{2c}(0))-g_z'(m_{2c}(0))\delta(z)}{g_z'(m_{2c}(0))}.
 \ee
Inspired by this equation, we define iteratively a sequence ${\delta}^{(k)}(z) \in \C$ such that ${\delta}^{(0)}=0$ and 
\begin{align}\label{selfomega2}
 \delta^{(k+1)} = - \frac{g_z(m_{2c}(0)) - g_0(m_{2c}(0))}{g_z'(m_{2c}(0))} 
 -\frac{g_z(m_{2c}(0)+ \delta^{(k)}) -  g_z(m_{2c}(0))-g_z'(m_{2c}(0))\delta^{(k)} }{g_z'(m_{2c}(0))}.
\end{align}
Then, equation \eqref{selfomega2} defines a mapping $h_z:\C\to \C$, which maps $\delta^{(k)}$ to $\delta^{(k+1)}=h_z(\delta^{(k)})$.
 
Through a straightforward calculation, we get that
$$g_z'(m_{2c}(0)) = -1 - \frac{\gamma_n}p\sum_{i=1}^p \frac{ \lambda_i^2(1-\gamma_n) - z\left[1- \lambda_i^2 \left(  2m_{2c}^{-1}(0)+1\right)\right]  }{  \left[z -\lambda_i^2(1-\gamma_n)+ (1 - \lambda_i^2) r_2m_{2c}(0) - \lambda_i^2 z\left( m_{2c}^{-1}(0)+1\right)\right]^2 }.$$
Using equations \eqref{assm3_app} and \eqref{a23}, it is easy to check that
\begin{align*}
&\left|z -\lambda_i^2(1-\gamma_n)+ (1 - \lambda_i^2) r_2m_{2c}(0) - \lambda_i^2 z\left( m_{2c}^{-1}(0)+1\right)\right| \ge c_\tau - c_\tau^{-1} |z| ,
\end{align*}
for a constant $c_\tau>0$ depending only on $\tau$. Hence, as long as we choose $c_0\le c_\tau^2/2$, we have 
\begin{align*}%
&\left|z -\lambda_i^2(1-\gamma_n)+ (1 -\lambda_i^2) r_2m_{2c}(0) - \lambda_i^2 z\left( m_{2c}^{-1}(0)+1\right)\right| \ge c_\tau /2 .
\end{align*}
With this estimate, we can check that there exist constants $\wt c_\tau, \wt C_\tau>0$ depending only on $\tau$ such that the following estimates hold: for all $z$, $\delta_1$, $\delta_2$ such that $\min\{|z|,|\delta_1|,|\delta_2| \}\le \wt c_\tau$,
\be\label{dust}
\left|\frac{1}{g_z'(m_{2c}(0))} \right|\le \wt C_\tau, \quad   \left|\frac{g_z(m_{2c}(0)) - g_0(m_{2c}(0))}{g_z'(m_{2c}(0))}\right|  \le \wt C_\tau |z| ,
\ee
and 
\be\label{dust222}
\left|\frac{g_z(m_{2c}(0)+ \delta_1) -  g_z(m_{2c}(0)+\delta_2)-g_z'(m_{2c}(0))(\delta_1-\delta_2) }{g_z'(m_{2c}(0))}\right|  \le \wt C_\tau |\delta_1-\delta_2|^2 .
\ee
Using equations \eqref{dust} and \eqref{dust222}, we find that there exists a sufficiently small constant $c_1>0$ depending only on $\wt C_\tau$ such that 
$h_z: B_{d}  \to B_{d}$ is a self-mapping on the ball $B_d:=\{\delta \in \C: |\delta| \le d \}$, as long as $d\le c_1$ and $|z| \le c_1$. 
Now, it suffices to prove that $h_z$ restricted to $B_d $ is a contraction, which then implies that ${\delta}:=\lim_{k\to\infty} { \delta}^{(k)}$ exists and $m_{2c}(0)+\delta(z)$ is a unique solution to the equation $g_z(m_{2c}(z))=1$ subject to  $|{\delta}| \le d$.

From the iteration relation \eqref{selfomega2}, using equation \eqref{dust222} we readily get that
\begin{equation*}%
{ \delta}^{(k+1)} - { \delta}^{(k)}= h_z({\delta}^{(k)}) - h_z({\delta}^{(k-1)}) \le \wt C_\tau | { \delta}^{(k)}-{ \delta}^{(k-1)}|^2.
\end{equation*}
Hence, $h$ is indeed a contraction mapping on $ B_d$ as long as $d$ is chosen sufficiently small such that $2d\wt C_\tau \le 1/2$.  
This proves both the existence and uniqueness of the solution $m_{2c}(z)=m_{2c}(0)+\delta(z)$ if we choose $c_0$ in equation \eqref{prior1} to be \smash{$c_0=\min\{c_\tau^2/2, c_1, (4\wt C_\tau)^{-1}\}$}. After obtaining $m_{2c}(z)$, we then solve $m_{1c}(z)$ using the first equation in \eqref{selfalter}. 

It remains to prove the estimate \eqref{Lipomega}. Using equation \eqref{dust} and ${\delta}^{(0)}= 0$, we can obtain from equation \eqref{selfomega2} that $ |{ \delta}^{(1)}(z)| \le \wt C_\tau |z| .$
Then, by the contraction mapping, we have the bound 
\begin{equation*}%
|{ \delta}| \le \sum_{k=0}^\infty |{  \delta}^{(k+1)}-{ \delta}^{(k)}| \le 2\wt C_\tau |z| \ \Rightarrow \ |m_{2c}(z)-m_{2c}(0)|\le 2\wt C_\tau |z| .\end{equation*}
Finally, using the first equation in \eqref{selfalter}, we also obtain the bound $|m_{1c}(z)-m_{1c}(0)| \lesssim |z|$. 
\end{proof}

As a byproduct of the above contraction mapping argument, we also obtain the following stability result that will be used in the proof of Theorem \ref{LEM_SMALL}. Roughly speaking, it states that if two analytic functions $m_1(z)$ and $m_2(z)$ satisfy the self-consistent equation \eqref{selfomega} approximately up to some small errors, then $m_1(z)$ and $m_2(z)$ will be close to the solutions $m_{1c}(z)$ and $m_{2c}(z)$.

\begin{lemma} \label{lem_stabw}
There exist constants $c_0, C_0>0$ depending only on $\tau$ in Assumption \ref{assm_big1} and \eqref{a23} such that the system of self-consistent equations \eqref{selfomega} is stable in the following sense. Suppose $|z|\le c_0$, and $m_{1}(z) $ and $m_{2}(z)$ are analytic functions of $z$ such that
\be  \label{prior12}
|m_{1}(z) - m_{1c}(0)| + |m_{2}(z) - m_{2c}(0)|\le c_0.
\ee
Moreover, assume that $m_1, m_2$ satisfies the system of equations
\begin{equation}\label{selfomegaerror}
\begin{split}
&\frac{1}{m_{1}} + 1 -\frac{\gamma_n}p\sum_{i=1}^p \frac{\lambda_i^2}{  z+\lambda_i^2  r_1m_{1} +r_2 m_{2}  } =\cal E_1,\\ &\frac{1}{m_{2}} + 1 -\frac{\gamma_n}p\sum_{i=1}^p \frac{1 }{  z+\lambda_i^2 r_1m_{1} +  r_2m_{2}  }=\cal E_2,
\end{split}
\end{equation}
for some (deterministic or random) errors such that $  |\mathcal E_1| +  |\mathcal E_2| \le (\log n)^{-1/2}$. 
Then, we have 
 \begin{equation*}
  \left|m_1(z)-m_{1c}(z)\right| +  \left|m_2(z)-m_{2c}(z)\right|\le C_0\left( |\mathcal E_1| +  |\mathcal E_2|\right). %
\end{equation*}
\end{lemma}

\begin{proof}%
Under condition \eqref{prior12}, we can obtain equation \eqref{selfalter} approximately: %
\be\label{selfalter2}r_1 m_{1}=-(1-\gamma_n) - r_2m_{2} - z\left(  {m_{2}^{-1}}+1\right) + \wt{\cal E}_1(z),\quad g_z(m_{2}(z))=1+ \wt{\cal E}_2(z),\ee
where the errors satisfy that $|\wt{\cal E}_1(z)|+ |\wt{\cal E}_2(z)|\lesssim |\mathcal E_1| +  |\mathcal E_2|$. Then, we subtract equation \eqref{selfalter} from equation \eqref{selfalter2}, and consider the contraction principle for the function $\delta (z):= m_{2}(z) - m_{2c}(z)$.  The rest of the proof is exactly the same as the one for Lemma \ref{lem_mbehaviorw}, so we omit the details.
\end{proof}

\subsubsection{Multivariate Gaussian Case}\label{sec entry}

One main difficulty for the proof is that the entries of $Z^{(1)} U\Lambda$ and ${Z^{(2)}V}$ are not independent. However, if the entries of $Z^{(1)}$ and $Z^{(2)}$ are i.i.d.~Gaussian, then by the rotational invariance of the multivariate Gaussian distribution, we have that
\begin{equation*}%
Z^{(1)}  U\Lambda \stackrel{d}{=}  Z^{(1)} \Lambda, \quad  Z^{(2)}  V \stackrel{d}{=} Z^{(2)}  .\end{equation*}
In this case, the problem is reduced to proving the local laws for $G(z)$ with $U=\id_{p\times p}$ and $V=\id_{p\times p}$, such that the entries of $ Z^{(1)} \Lambda  $ and ${Z^{(2)}}$ are independent.
In this case, we use the standard resolvent method, as in e.g., \citet{isotropic}, to prove the following proposition. Note that if the entries of $ Z^{(1)}$ and $ Z^{(2)}$ are i.i.d.~Gaussian, then %
$ Z^{(1)}$ and $ Z^{(2)}$ have bounded support $Q=1$ by Remark \ref{rem_stoch_add}.

\begin{proposition}\label{prop_diagonal}
    Under the setting of Theorem \ref{LEM_SMALL}, assume further that the entries of $ Z^{(1)}$ and $ Z^{(2)}$ are i.i.d.~Gaussian random variables, and $U=V=\id_{p\times p}$. Then, the estimates \eqref{aver_in}, \eqref{aver_in1}, and \eqref{aniso_law} hold  with $Q= 1$.
\end{proposition}

The proof of Proposition \ref{prop_diagonal} is based on the following entrywise local law. %
\begin{lemma}[Entrywise local law]\label{prop_entry}
Under the setting of Proposition \ref{prop_diagonal}, %
the averaged local laws \eqref{aver_in} and \eqref{aver_in1} and the following entrywise local law hold %
with $Q= 1$: %
\begin{equation}\label{entry_diagonal}
\sup_{z\in \mathbf D} \max_{\fa,\fb\in \cal I}\left| G_{\fa\fb}(z)  - \Gi_{\fa\fb} (z) \right| \prec n^{-1/2}.
\end{equation} 
\end{lemma}

With Lemma \ref{prop_entry}, we can complete the proof of Proposition \ref{prop_diagonal}. %

\begin{proof}%
With estimate \eqref{entry_diagonal}, we can use the polynomial method in Section 5 of \citet{isotropic} to get the anisotropic local law \eqref{aniso_law} with $Q=1$. The proof is exactly the same, except for some minor differences in notations. Hence, we omit the details.
\end{proof}

The rest of this subsection is devoted to the proof of Lemma \ref{prop_entry}, where the resolvent $G$ in Definition \ref{defn_resolventH} becomes
 \begin{equation} \label{resolv Gauss1}
   G(z)= \left( {\begin{array}{*{20}c}
   { -z\id_{p\times p} } & n^{-1/2}\Lambda {Z^{(1)}}^\top & n^{-1/2} {Z^{(2)}}^\top  \\
   {n^{-1/2} Z^{(1)} \Lambda  } & {-\id_{n_1\times n_1}} & 0 \\
   {n^{-1/2} Z^{(2)}} & 0 & {-\id_{n_2\times n_2}}
   \end{array}} \right)^{-1}.
 \end{equation}
To deal with the matrix inverse, we introduce resolvent minors.
\begin{definition}[Resolvent minors]\label{defn_Minor}
    Given a $(p+n)\times (p+n)$ matrix $\cal A$ and $\mathfrak c \in \cal I$, the minor of $\cal A$ after removing the $\mathfrak c$-th row and column is a $(p + n - 1)\times (p + n - 1)$ matrix denoted by $\cal A^{(\mathfrak c)} := [\cal A_{\fa \mathfrak b }:  \fa , \mathfrak b\in \mathcal I\setminus \{\mathfrak c\}]$. 
    We keep the names of indices when defining $\cal A^{(\mathfrak c)}$, i.e., $\cal A^{(\mathfrak c)}_{\fa\fb }= \cal A_{\fa\fb }$ for $ {\fa, \fb } \ne \mathfrak c$. Correspondingly, we define the resolvent minor of $G(z)$ by
    \begin{align*}
		G^{(\mathfrak c)}(z) := \left[ \left( {\begin{array}{*{20}c}
		  { -z\id_{p\times p} } & n^{-1/2}\Lambda {Z^{(1)}}^\top & n^{-1/2} {Z^{(2)}}^\top  \\
      {n^{-1/2} Z^{(1)} \Lambda  } & {-\id_{n_1\times n_1}} & 0 \\
			{n^{-1/2} Z^{(2)}} & 0 & {-\id_{n_2\times n_2}}
    \end{array}} \right)^{(\mathfrak c)}\right]^{-1}.
    \end{align*}
    We define the partial traces $m^{(\mathfrak c)}(z)$, $m_0^{(\mathfrak c)}(z)$, $m_1^{(\mathfrak c)}(z)$ and $m_2^{(\mathfrak c)}(z)$ by replacing $G(z)$ with $G^{(\mathfrak c)}(z)$ in equation \eqref{defm}. For convenience, we will adopt the convention that $G^{(\mathfrak c)}_{\fa\fb} = 0$ if $\fa = \mathfrak c$ or $\mathfrak b = \mathfrak c$. %
\end{definition}

The following resolvent identities are important tools for our proof. All of them can be proved directly using Schur's complement formula, cf. Lemma 4.4 of \citet{Anisotropic}. Recall that the matrix $\AF$ is defined in equation \eqref{defn AF}.
We do not assume $U$ and $V$ are identity matrices for Lemma \ref{lemm_resolvent} and Lemma \ref{lemm apri} below.
\begin{lemma}\label{lemm_resolvent}
We have the following resolvent identities.
\begin{itemize}
	\item[(i)] For $i\in \mathcal I_0$ and $\mu\in \mathcal I_1\cup \cal I_2$, we have
		\begin{equation}
			\frac{1}{G_{ii}} =  - z - \left( {\AF G^{\left( i \right)} \AF^\top} \right)_{ii} ,\quad  \frac{1}{{G_{\mu \mu } }} =  - 1  - \left( {\AF^\top  G^{\left( \mu  \right)} \AF} \right)_{\mu \mu }.\label{resolvent2}
		\end{equation}
	\item[(ii)] For $i\in \mathcal I_0$, $\mu \in \mathcal I_1\cup \cal I_2$, $\fa\in \cal I\setminus \{i\}$, and $\fb\in \cal I\setminus \{ \mu\}$, we have
		\begin{equation}
			G_{i\fa}   = -G_{ii}  \left( \AF G^{\left( {i} \right)} \right)_{i\fa},\quad  G_{\mu \fb }  = - G_{\mu \mu }  \left( \AF^\top  G^{\left( {\mu } \right)}  \right)_{\mu \fb }. \label{resolvent3}
		\end{equation}
 \item[(iii)] For $\mathfrak c \in \mathcal I$ and $\fa,\fb \in \mathcal I \setminus \{\mathfrak c\}$, we have
		\begin{equation}
			G_{\fa\fb}^{\left( \mathfrak c \right)}  = G_{\fa\fb}  - \frac{G_{\fa\mathfrak c} G_{\mathfrak c\fb}}{G_{\mathfrak c\mathfrak c}}.
			\label{resolvent8}
		\end{equation}
\end{itemize}
\end{lemma}

We claim the following a priori estimate on the resolvent $G(z)$ for $z\in \mathbf D$.

\begin{lemma}\label{lemm apri}
Under the setting of Theorem \ref{LEM_SMALL}, there exists a constant $C>0$ such that with overwhelming probability the following estimates hold uniformly in $z,z'\in \mathbf D$:
\begin{align}\label{priorim}
    &\|G(z)\| \le C, \text{ and } \\
    &\left\|G  (z) - G(z')\right\| \le C|z-z'|. \label{priordiff} 
\end{align}
\end{lemma}
\begin{proof}
 Our proof is a simple application of the spectral decomposition of $G$. %
 Let
\begin{equation*}%
\AF= \sum_{k = 1}^{p} {\sqrt {\mu_k} \xi_k } \zeta _{k}^\top ,\quad \mu_1\ge \mu_2 \ge \cdots \ge \mu_{p} \ge 0 =\mu_{p+1} = \ldots = \mu_{n},\end{equation*}
be a singular value decomposition of $\AF$, where
$\{\xi_{k}\}_{k=1}^{p}$ are the left-singular vectors and $\{\zeta_{k}\}_{k=1}^{n}$ are the right-singular vectors.
Then, using equation \eqref{green2}, we get that for $i,j\in \mathcal I_1$ and $\mu,\nu\in \mathcal I_1\cup \cal I_2$,
\be\label{spectral}
G_{ij} = \sum_{k = 1}^{p} \frac{\xi_k(i) \xi_k^\top(j)}{\mu_k-z}, \ \quad \ G_{\mu\nu} =
z\sum_{k = 1}^{n} \frac{\zeta_k(\mu) \zeta_k^\top(\nu)}{\mu_k-z} , \ee
\be \label{spectral2}
G_{i\mu} = G_{\mu i} = \sum_{k = 1}^{p} \frac{\sqrt{\mu_k}\xi_k(i) \zeta_k^\top(\mu)}{\mu_k-z}.
\ee
Using the fact $n^{-1}V^\top {Z^{(2)}}^\top Z^{(2)}V \preceq FF^\top$ and Lemma \ref{SxxSyy}, we obtain that 
$$\mu_p \ge \lambda_p\left(n^{-1}{Z^{(2)}}^\top Z^{(2)}\right) \ge c_\tau \quad \text{w.o.p.}, $$ 
for a constant $c_\tau>0$ depending only on $\tau$.  This further implies that
$$ \inf_{z\in \mathbf D}\min_{1\le k \le p}|\mu_k-z| \ge c_\tau - (\log n)^{-1}.$$
Combining this bound with estimates \eqref{spectral} and \eqref{spectral2}, we can conclude that estimates \eqref{priorim} and \eqref{priordiff} are both true.
\end{proof}

Now, we are ready to give the proof of Lemma \ref{prop_entry}.

\begin{proof}%
In the setting of Lemma \ref{prop_entry}, we can write equation \eqref{defn AF} as %
\begin{equation*}%
\AF \stackrel{d}{=} n^{-1/2}[\Lambda {Z^{(1)}}^{\top}, {Z^{(2)}}^\top].\end{equation*}
We use the resolvent in equation \eqref{resolv Gauss1} throughout the following proof. 
Our proof is divided into four steps. 

\medskip
\textit{Step 1: Large deviation estimates.} 
In this step, we prove some large deviation estimates on the off-diagonal $G$ entries and the following $\cal Z$ variables. In analogy to Section 3 of \citet{EKYY1} and Section 5 of \citet{Anisotropic}, we introduce the $\cal Z$ variables  
\begin{equation*}
 \cal  Z_{{\fa}} :=(1-\mathbb E_{{\fa}})\big[\big(G_{{\fa}{\fa}}\big)^{-1}\big], \quad \fa\in \cal I, %
\end{equation*}
where $\mathbb E_{{\fa}}[\cdot]:=\mathbb E[\cdot\mid H^{({\fa})}]$ denotes the partial expectation over the entries in the ${\fa}$-th row and column of $H$. Using equation (\ref{resolvent2}), we get that for $i \in \cal I_0$, 
\begin{align}
\cal Z_i  = &\  \frac{\lambda_i^2}{n} \sum_{\mu ,\nu\in \mathcal I_1}  G^{(i)}_{\mu\nu} \left(\delta_{\mu\nu} - Z^{(1)}_{\mu i}Z^{(1)}_{\nu i}\right)+\frac1n \sum_{\mu ,\nu\in \mathcal I_2}  G^{(i)}_{\mu\nu} \left( \delta_{\mu\nu} - Z^{(2)}_{\mu i}Z^{(2)}_{\nu i}\right) \nonumber\\
&\ - 2 \frac{\lambda_i}{n} \sum_{\mu\in \cal I_1,\nu\in \mathcal I_2} Z^{(1)}_{\mu i}Z^{(2)}_{\nu i}G^{\left( i \right)}_{\mu\nu},  \label{Zi}
\end{align}
and for $\mu\in \cal I_1$ and $\nu\in \cal I_2$, 
\begin{align}
&\cal  Z_\mu= \frac{1}{n} \sum_{i,j \in \mathcal I_0}  {\lambda_i \lambda_j}G^{(\mu)}_{ij} \left(\delta_{ij} - Z^{(1)}_{\mu i}Z^{(1)}_{\mu j}\right), \ \ \cal Z_\nu = \frac{1}{n} \sum_{i,j \in \mathcal I_0} G^{(\nu)}_{ij} \left( \delta_{ij} - Z^{(2)}_{\nu i}Z^{(2)}_{\nu j}\right).\label{Zmu} 
\end{align}
Moreover, we introduce the random error
\begin{equation}  \label{eqn_randomerror}
 \Lambda _o : = %
 \max_{{\fa} \ne {\fb} } \left|  G_{{\fa}{\fa}}^{-1}G_{{\fa}{\fb}}   \right| ,
\end{equation}
which controls the size of off-diagonal entries. %

\begin{lemma}\label{Z_lemma}
Under the assumptions of Proposition \ref{prop_diagonal}, the following estimate holds uniformly in $z\in \mathbf D$:
\begin{align}
\Lambda_o + \max_{{\fa}\in \cal I} |\cal Z_{{\fa}}|  \prec n^{-1/2}. \label{Zestimate1}
\end{align}
\end{lemma}
\begin{proof}
Note that for ${\fa}\in \cal I$, $H^{({\fa})}$ and $G^{({\fa})}$ also satisfy the assumptions of Lemma \ref{lemm apri}. Hence, the estimates \eqref{priorim} and \eqref{priordiff} hold for $G^{({\fa})}$ with overwhelming probability. For any $i\in \cal I_0$, since $G^{(i)}$ is independent of the entries in the $i$-th row and column of $H$, we can apply equations \eqref{eq largedev1}, \eqref{eq largedev2}, and \eqref{eq largedev3} to equation \eqref{Zi} to obtain that%
\begin{equation}\nonumber%
\begin{split}
\left| \cal Z_{i}\right|&\lesssim \frac{1}{n} \sum_{k=1}^2 \Big|\sum_{\mu ,\nu\in \mathcal I_k}  G^{(i)}_{\mu\nu} \left(\delta_{\mu\nu} - Z^{(k)}_{\mu i}Z^{(k)}_{\nu i}\right)\Big|+ \frac{1}{n} \Big|\sum_{\mu\in \cal I_1,\nu\in \mathcal I_2} Z^{(1)}_{\mu i}Z^{(2)}_{\nu i}G^{\left( i \right)}_{\mu\nu}\Big| \\
&\prec  \frac{1}{n} \Big( \sum_{\mu,\nu \in \cal I_1\cup \cal I_2 }  {| G_{\mu\nu}^{(i)}|^2 } \Big)^{1/2} \prec n^{-1/2} .
\end{split}
\end{equation}
Here, in the last step we used equation \eqref{priorim} to get that for $\mu\in \cal I_1\cup \cal I_2$,
\be\label{GG*}\sum_{\nu \in \cal I_1\cup \cal I_2 }  | G_{\mu\nu}^{(i)} |^2\le \sum_{{\fa} \in \cal I } | G_{\mu{\fa}}^{(i)} |^2 =\left(G^{(i)}{G^{(i)}}^* \right)_{\mu\mu} =\OO(1),\ee
 with overwhelming probability, where ${G^{(i)}}^*$ denotes the conjugate transpose of $G^{(i)}$. Similarly, applying equations \eqref{eq largedev1}, \eqref{eq largedev2}, and \eqref{eq largedev3} to $\cal Z_{\mu}$ and $\cal Z_\nu$ in equation \eqref{Zmu} and using equation \eqref{priorim}, we can obtain the same bound. This gives that $\max_{{\fa}\in \cal I} |\cal Z_{{\fa}}|  \prec n^{-1/2}$.

Next, we prove the off-diagonal estimate on $\Lambda_o$. For $i\in \mathcal I_1$ and ${\fa}\in \cal I\setminus \{i\}$, using equations \eqref{resolvent3}, \eqref{eq largedev1} and \eqref{priorim}, we can obtain that 
\begin{align*}
  \left|G_{ii}^{-1}G_{i{\fa}}\right| &\lesssim {n^{-1/2}}\Big|  \sum_{\mu \in \cal I_1} Z^{(1)}_{\mu i} G^{(i)}_{\mu \fa} \Big| + n^{-1/2}\Big|\sum_{\mu \in \cal I_2} Z^{(2)}_{\mu i} G^{(i)}_{\mu \fa}\Big| \\
& \prec  n^{-1/2}\Big( \sum_{\mu \in \cal I_1\cup \cal I_2}  {| G_{\mu {\fa}}^{(i)} |^2 } \Big)^{1/2} \prec n^{-1/2}. 
 \end{align*}
Using exactly the same argument, we can get a similar estimate on $\left|G_{\mu\mu}^{-1} G_{\mu{\fb}} \right|$ for $\mu \in \mathcal I_1\cup \cal I_2$ and ${\fb}\in \cal I\setminus \{ \mu\}$.  
This gives that $\Lambda_o\prec n^{-1/2}$. 
\end{proof}

Note that combining equation \eqref{Zestimate1} with the %
bound in equation \eqref{priorim}, we immediately conclude equation \eqref{entry_diagonal} for off-diagonal resolvent entries $G_{\fa\fb}$ with ${\fa}\ne {\fb}$.

\medskip
\textit{Step 2: Self-consistent equations.} 
In this step, we show that $(m_1(z),m_2(z))$ satisfies the system of approximate self-consistent equations in equation \eqref{selfomegaerror} for some small errors $\cal E_{1}$ and $\cal E_{2}$. Later in Step 3, we will apply Lemma \ref{lem_stabw} to show that $(m_1(z),m_2(z))$ is close to $(m_{1c}(z),m_{2c}(z))$.  

By equation \eqref{Lipomega}, the following estimates hold for $z\in \mathbf D$:
$$|m_{1c}(z)-m_{1c}(0)| \lesssim (\log n)^{-1},\quad |m_{2c}(z)-m_{2c}(0)| \lesssim (\log n)^{-1}.$$
Combining them with the estimates in \eqref{a23}, %
we obtain that uniformly in $z\in \mathbf D$, %
\be\label{Gsim1}
 |m_{1c}(z)| \sim |m_{2c}(z)| \sim 1, \quad |z+\lambda_i^2  r_1m_{1c}(z) + r_2 m_{2c}(z)|\sim 1. \ee
 Moreover, using equation \eqref{selfomega}, we get that uniformly in $z\in \mathbf D$,
 \be\label{Gsim0}
\left|1 + \gamma_n m_c (z)\right| = |m_{2c}^{-1}(z)| \sim 1, \quad  |1+\gamma_n m_{0c}(z)| = |m_{1c}^{-1}(z)| \sim 1  ,
\ee
 where we introduce the notations
 \begin{align}
 m_c(z)&:=-\frac1p\sum_{i=1}^p\frac{1}{z+\lambda_i^2 r_1m_{1c}(z) +r_2m_{2c}(z)},\label{defn mc1c}\\
  m_{0c}(z)&:=-\frac1p\sum_{i=1}^p\frac{ \lambda_i^2}{z+\lambda_i^2 r_1m_{1c}(z) +r_2m_{2c}(z)}. \label{defn mc0c}
 \end{align}
Later, we will see that $m_c(z)$ and $m_{0c}(z)$ are respectively the asymptotic limits of $m(z)$ and $m_0(z)$ defined in equation \eqref{defm}. Applying equation \eqref{Gsim1} to  equation \eqref{defn_piw} and using equation \eqref{M1M2a1a2}, we get that
\be\nonumber
|\Gi_{{\fa}{\fa}}(z)| \sim 1 \ \ \text{uniformly in } z\in \mathbf D \ \text{ and } \ {\fa}\in \cal I.
\ee 

Now, we define the following $z$-dependent event 
\be\label{Xiz}\Xi(z):=\left\{ |m_{1}(z)-m_{1c}(0)| + |m_{2}(z)-m_{2c}(0)| \le (\log n)^{-1/2}\right\}.\ee
With equation \eqref{Gsim1}, we get that on $\Xi(z)$,
\be\label{Gsim012} |m_{1}(z)| \sim |m_{2}(z)| \sim 1, \quad |z +  \lambda_i^2  r_1m_1(z)+r_2m_2(z)|\sim1.\ee 
We claim the following key lemma, which shows that $(m_1(z),m_2(z))$ satisfies equation \eqref{selfomegaerror} on $\Xi(z)$ for some small (random) errors $\cal E_{1}$ and $\cal E_{2}$.

\begin{lemma}\label{lemm_selfcons_weak}
Under the setting of Lemma \ref{prop_entry}, the following estimates hold uniformly in $z \in \mathbf D$: 
\begin{equation} \label{selfcons_lemm}
\mathbf 1(\Xi) \left|\frac{1}{m_{1}} + 1 -\frac{\gamma_n}p\sum_{i=1}^p \frac{ \lambda_i^2}{  z+\lambda_i^2 r_1m_{1} + r_2m_{2}  } \right|\prec n^{-1}+n^{-\frac{1}2}\Theta +|[\cal Z]_0|+  |[\cal Z]_1|,\end{equation}
and
\begin{equation} \label{selfcons_lemm2}
\mathbf 1(\Xi) \left|\frac{1}{m_{2}} + 1 -\frac{\gamma_n}p\sum_{i=1}^p \frac{1 }{  z+\lambda_i^2  r_1m_{1} +  r_2m_{2}  }\right|\prec n^{-1}+n^{-\frac12}\Theta +|[\cal Z]|+  |[\cal Z]_2|,
\end{equation}
where we introduce the notations
\be\label{defn_Theta}\Theta:= |m_1(z)-m_{1c}(z)|+|m_2(z)-m_{2c}(z)|,\ee
and
\begin{equation} \label{def_Zaver}
	\begin{split}
 &[\cal Z]:= \frac1p\sum_{i\in \cal I_0} \frac{\cal Z_i}{(z+ \lambda_i^2  r_1m_{1c}+ r_2m_{2c})^2} ,\\ 
 &[\cal Z]_0:= \frac1p\sum_{i\in \cal I_0} \frac{\lambda_i^2 \cal Z_i}{(z+ \lambda_i^2 r_1m_{1c}+ r_2m_{2c})^2} ,\\
 &[\cal Z]_1:=\frac{1}{n_1}\sum_{\mu\in \mathcal I_1}\cal Z_\mu, \quad [\cal Z]_2:=\frac{1}{n_2}\sum_{\nu \in \mathcal I_2}  \cal Z_\nu.
 \end{split}
\end{equation}
\end{lemma}

\begin{proof}
 Using equations (\ref{resolvent2}), (\ref{Zi}) and (\ref{Zmu}), we can write that 
\begin{align}
\frac{1}{{G_{ii} }}&=  - z - \frac{\lambda_i^2}{n} \sum_{\mu\in \mathcal I_1} G^{\left( i \right)}_{\mu\mu}- \frac{1}{n} \sum_{\mu\in \mathcal I_2} G^{\left( i \right)}_{\mu\mu} + \cal Z_i =  - z - \lambda_i^2  r_1m_1 - r_2m_2 + \cal E_i, \forall i \in \cal I_0, \label{self_Gii}\\
\frac{1}{{G_{\mu\mu} }}&=  - 1 - \frac{1}{n} \sum_{i\in \mathcal I_0}\lambda_i^2 G^{\left(\mu\right)}_{ii}+ \cal Z_{\mu} =  - 1  -  \gamma_n m_0 + \cal E_\mu,  \forall \mu \in \cal I_1,  \label{self_Gmu1}\\
\frac{1}{{G_{\nu\nu} }}&=  - 1 - \frac{1}{n} \sum_{i\in \mathcal I_0} G^{\left(\nu\right)}_{ii}+\cal Z_{\nu} =   - 1 - \gamma_n m + \cal E_\nu, \forall\nu \in \cal I_2, \label{self_Gmu2}
\end{align}
where we denote (recall the notations in equation \eqref{defm} and Definition \ref{defn_Minor}) %
\begin{align*}
    \cal E_i &:= \cal Z_i + \lambda_i^2 r_1(m_1 - m_1^{(i)}) + r_2(m_2-m_2^{(i)}) , \text{ and } \\
    \cal E_\mu &:= \cal Z_{\mu} +    \gamma_n(m_0-m_0^{(\mu)}),\quad \cal E_\nu:=\cal Z_{\nu} +\gamma_n(m-m^{(\nu)}).
\end{align*}
Using equations \eqref{resolvent8}, \eqref{eqn_randomerror} and \eqref{Zestimate1}, we can bound that  
\begin{equation}\label{m1i}
  |m_1 - m_1^{(i)}| \le   \frac1{n_1}\sum_{\mu\in \mathcal I_1}  \left|\frac{G_{\mu i} G_{i\mu}}{G_{ii}}\right| \le |\Lambda_o|^2|G_{ii}| \prec n^{-1}.
\end{equation}
Similarly, we also have that 
\be \label{higherr}  
 |m_2 - m_2^{(i)}| \prec n^{-1} , \quad |m_0 - m_0^{(\mu)}| \prec n^{-1},\quad  |m-m^{(\nu)}| \prec n^{-1},  \ee
for $i\in \cal I_0$, $\mu \in \cal I_{1}$ and $\nu\in \cal I_2$. Combining the above estimates with equation \eqref{Zestimate1}, we obtain that %
\begin{equation}\label{erri}
\max_{i\in \cal I_0} |\cal E_i | +\max_{\mu \in \cal I_1\cup \cal I_2} |\cal E_\mu|  \prec n^{-1/2}.
\end{equation}

From equation \eqref{self_Gii}, we obtain that on $\Xi$,
\begin{align}
  G_{ii}&= -\frac{1}{z + \lambda_i^2 r_1 m_1+r_2 m_2} - \frac{\cal E_i}{(z + \lambda_i^2  r_1 m_1+r_2 m_2)^2} +\OO_\prec\left(n^{-1}\right) \nonumber\\
&= -\frac{1}{z + \lambda_i^2  r_1 m_1+r_2 m_2} - \frac{\cal Z_i}{(z + \lambda_i^2 r_1 m_{1c}+r_2 m_{2c})^2} +\OO_\prec\left(n^{-1} + n^{-\frac12}\Theta\right) ,\label{Gmumu0}
\end{align}
where in the first step we use estimates \eqref{erri} and \eqref{Gsim012} on $\Xi$, and in the second step we use estimates \eqref{defn_Theta}, \eqref{m1i} and \eqref{erri}. Plugging equation \eqref{Gmumu0} into the definitions of $m$ and $m_0$ in equation \eqref{defm} and using equation \eqref{def_Zaver}, we get that on $\Xi$,
\begin{align}
 m&= -\frac1p\sum_{i\in \cal I_0}\frac{1}{z + \lambda_i^2 r_1 m_1+r_2 m_2} -[\cal Z] +\OO_\prec\left(n^{-1} + n^{-1/2}\Theta\right), \label{Gmumu} \\
 m_0&= -\frac1p\sum_{i\in \cal I_0}\frac{\lambda_i^2}{z + \lambda_i^2 r_1 m_1+r_2 m_2}  -[\cal Z]_0 +\OO_\prec\left(n^{-1} + n^{-1/2}\Theta\right). \label{Gmumu2}
\end{align}
Comparing these two equations with estimates \eqref{defn mc1c} and \eqref{defn mc0c}, we obtain that  
\begin{equation*}%
 |m(z)-m_c(z)| +|m_0(z)-m_{0c}(z)|  \lesssim (\log n)^{-1/2}, \quad \text{w.o.p.}\ \ \text{ on } \ \  \Xi. 
\end{equation*}
Together with equation \eqref{Gsim0}, this estimate implies that %
\be\label{Gsim2}
|1+\gamma_nm (z)|\sim 1, \quad |1+\gamma_nm_0(z)|\sim 1, \quad \text{w.o.p.}\ \ \text{ on } \ \  \Xi. 
\ee
With a similar argument as above, from equations \eqref{self_Gmu1} and \eqref{self_Gmu2} we obtain that on $\Xi$,%
\begin{align}
&G_{\mu\mu}=-\frac{1}{1 + \gamma_nm_0}  - \frac{\cal Z_\mu}{(1 + \gamma_nm_{0})^2}+\OO_\prec\left(n^{-1} + n^{-\frac 1 2}\Theta\right)   ,\quad \mu \in \cal I_1,\label{Gii0} \\
& G_{\nu\nu}=-\frac{1}{1 + \gamma_nm} - \frac{\cal Z_\nu}{(1 + \gamma_nm)^2}+\OO_\prec\left(n^{-1} + n^{-\frac 1 2}\Theta\right) ,\quad \nu \in \cal I_2,\label{Gii00} 
\end{align}
 where we used estimates \eqref{defn_Theta}, \eqref{higherr}, \eqref{erri} and  \eqref{Gsim2} in the derivation.
Taking averages of equations \eqref{Gii0} and \eqref{Gii00} over $\mu\in \cal I_1$ and $\nu\in \cal I_2$, we get that  on $\Xi$, %
\begin{align}
& m_1=-\frac{1}{1 + \gamma_n m_0} - \frac{[\cal Z]_1}{(1 + \gamma_nm_0)^2}+\OO_\prec\left(n^{-1} + n^{-\frac 1 2}\Theta\right)    ,\label{Gii000}\\
&m_2=-\frac{1}{1 +\gamma_n  m}- \frac{[\cal Z]_2}{(1 + \gamma_nm)^2}+\OO_\prec\left(n^{-1} + n^{-\frac 1 2}\Theta\right) ,\label{Gii001}
\end{align}
which further implies that  on $\Xi$,
\begin{align}
 &\frac{1}{m_1} + 1 + \gamma_nm_0  \prec  n^{-1} + n^{-\frac 1 2}\Theta + |[\cal Z]_1|,\label{Gii}\\
 & \frac{1}{m_2} + 1 + \gamma_nm \prec   n^{-1} + n^{-\frac 1 2}\Theta + |[\cal Z]_2|.\label{Gii111}
\end{align}
Finally, plugging estimates \eqref{Gmumu} and \eqref{Gmumu2} into equations \eqref{Gii} and \eqref{Gii111}, we get that equations \eqref{selfcons_lemm} and \eqref{selfcons_lemm2} are true.
\end{proof}

\medskip
\textit{Step 3: Entrywise local law.} In this step, we show that the event $\Xi(z)$ in \eqref{Xiz} actually holds with overwhelming probability for all $z\in \mathbf D$. Once we have proved this fact, applying Lemma \ref{lem_stabw} to equations \eqref{selfcons_lemm} and  \eqref{selfcons_lemm2} gives that $(m_1(z),m_2(z))$ is close to $(m_{1c}(z),m_{2c}(z))$ up to an error of order $\OO_\prec(n^{-1/2})$, with which we can conclude the entrywise local law \eqref{entry_diagonal}. 

First, we claim that it suffices to show 
\be\label{Xiz0}
|m_{1}(0)-m_{1c}(0)| + |m_{2}(0)-m_{2c}(0)| \prec n^{-\frac 1 2}.
\ee
In fact, by equation \eqref{priordiff} we have that uniformly in $z\in \mathbf D$,
$$  |m_{1}(z)-m_{1}(0)|+ |m_{2}(z)-m_{2}(0)|\lesssim (\log n)^{-1}  \quad \text{w.o.p.}$$
Thus, if equation \eqref{Xiz0} holds, using the triangle inequality, we obtain from the above estimate that 
\be\label{roughh2} 
\sup_{z\in \mathbf D} \left( |m_{1}(z)-m_{1c}(0)|+ |m_{2}(z)-m_{2c}(0)|\right) \lesssim (\log n)^{-1} \quad \text{w.o.p.}\ee
The equation \eqref{roughh2} shows that the event $\Xi(z)$ holds w.o.p., %
and it also verifies the condition \eqref{prior12} of Lemma \ref{lem_stabw}. Now, applying Lemma \ref{lem_stabw} to equations \eqref{selfcons_lemm} and \eqref{selfcons_lemm2}, we obtain
\begin{align*}
\Theta(z)&=|m_1(z)-m_{1c}(z)|+|m_2(z)-m_{2c}(z)| \\
&\prec n^{-1}+n^{-\frac12}\Theta(z) +|[\cal Z]|+  |[\cal Z]_0| +|[\cal Z]_1|+  |[\cal Z]_2| ,
\end{align*}
which implies that 
\be\label{Xizz}
\Theta(z) \prec n^{-1} +|[\cal Z]|+  |[\cal Z]_0| +|[\cal Z]_1|+  |[\cal Z]_2| \prec n^{-\frac 1 2},
\ee
uniformly in all $z\in \mathbf D$, where we used equation \eqref{Zestimate1} in the second step. On the other hand, with equations \eqref{Gii0} to \eqref{Gii001}, we obtain that 
$$\max_{\mu\in \cal I_1} |G_{\mu\mu}(z)-m_{1}(z)|+ \max_{\nu\in \cal I_2} |G_{\nu\nu}(z)-m_{2}(z)|\prec n^{-\frac 1 2}.
$$
Combining this estimate with equation \eqref{Xizz}, we get that
\be\label{Xizz2} \max_{\mu\in \cal I_1} |G_{\mu\mu}(z)-m_{1c}(z)|+ \max_{\nu\in \cal I_2} |G_{\nu\nu}(z)-m_{2c}(z)|\prec n^{-\frac 1 2}.
\ee
Next, plugging equation \eqref{Xizz} into equation \eqref{Gmumu0} and recalling equations \eqref{defn_piw} and \eqref{M1M2a1a2}, we obtain that 
$$\max_{{i}\in \cal I_1}|G_{ii}(z)-\Gi_{ii}(z)| \prec n^{-\frac 1 2}. 
$$
Together with equation \eqref{Xizz2}, it gives the diagonal estimate
\be\label{diagest}
\max_{{\fa}\in \cal I}|G_{{\fa}{\fa}}(z)-\Pi_{{\fa}{\fa}}(z)| \prec n^{-\frac 1 2}. 
\ee
Combining equation \eqref{diagest} with the off-diagonal estimate on $\Lambda_o$ in equation \eqref{Zestimate1}, we conclude the entrywise local law \eqref{entry_diagonal}.

It remains to prove equation \eqref{Xiz0}. Using equations \eqref{priorim} and \eqref{spectral}, w.o.p.,
$$1\gtrsim m(0)=\frac1p\sum_{i\in \cal I_0}G_{ii}(0) = \frac1p\sum_{i\in \cal I_0}\sum_{k = 1}^{p} \frac{|\xi_k(i)|^2 }{\mu_k} \ge \mu_1^{-1} \gtrsim 1,$$
where we used Lemma \ref{SxxSyy} to bound  $\mu_1$. Similarly, we can get that $m_0(0)>0$ and $m_0(0)\sim 1$. Hence, we have the estimates
\be\label{add_1+m}1+\gamma_n m(0)\sim 1,\quad 1+\gamma_n m_0(0)\sim 1.\ee
Combining these estimates with equations \eqref{self_Gmu1}, \eqref{self_Gmu2} and  \eqref{erri}, we obtain that equations \eqref{Gii000} and \eqref{Gii001} hold at $z=0$ without requiring the event $\Xi(0)$ to hold. This further gives that w.o.p.,
$$  \left|\lambda_i^2 r_1m_1(0)+r_2m_2(0)\right|=\left|\frac{ \lambda_i^2 r_1}{ 1+\gamma_n m_0(0)} +\frac{r_2}{1+\gamma_n m(0)}+ \OO_\prec (n^{-1/2})\right| \sim 1 .$$
Then, combining this estimate with (\ref{self_Gii}) and \eqref{erri}, we obtain that equations \eqref{Gmumu} and \eqref{Gmumu2} also hold at $z=0$ without requiring the event $\Xi(0)$ to hold. Finally, plugging equations \eqref{Gmumu} and \eqref{Gmumu2} into equations \eqref{Gii} and \eqref{Gii111}, we conclude that equations \eqref{selfcons_lemm} and  \eqref{selfcons_lemm2}  hold at $z=0$, that is,
\begin{equation} \label{selfcons_lemma222}
\begin{split}
& \left|\frac{1}{m_{1}(0)} + 1 -\frac1n\sum_{i=1}^p \frac{\lambda_i^2}{ \lambda_i^2 r_1m_{1}(0) + r_2m_{2}(0)  } \right|\prec n^{-\frac 1 2},\\ 
&\left|\frac{1}{m_{2}(0)} + 1 -\frac1n\sum_{i=1}^p \frac{1 }{ \lambda_i^2 r_1m_{1}(0) + r_2 m_{2}(0)  }\right|\prec n^{-\frac 1 2}.
\end{split}
\end{equation}

Denoting $y_{1}=-m_{1}(0)$ and $y_{2}=-m_{2}(0)$, by estimates \eqref{Gii000} and \eqref{Gii001} at $z=0$, we have that
$$y_1= \frac{1}{1+\gamma_n m_{0}(0)} +\OO_\prec(n^{-1/2}),\quad y_2= \frac{1}{1+\gamma_n m(0)}+\OO_\prec(n^{-\frac 1 2}).$$ 
Together with equation \eqref{add_1+m}, it implies that 
\be\label{omega12} c \le y_1 \le 1, \quad  c\le y_2\le 1, \quad \text{w.o.p.},\ee
for a small constant $c>0$. 
Also with equation \eqref{selfcons_lemma222}, we can check that $(r_1y_1,r_2y_2)$ satisfies approximately the same system of equations as equations \eqref{eq_a12extra} and \eqref{fa1}:
\be\label{selfcons_lemm000}
r_1y_1+r_2 y_2 = 1-\gamma_n + \OO_\prec (n^{-1/2}),\quad  f(r_1y_1)=r_1 + \OO_\prec (n^{-1/2}).
\ee
The first equation of \eqref{selfcons_lemm000} and equation \eqref{omega12} together imply that $r_1 y_1 \in [0,1-\gamma_n]$ with overwhelming probability. We also know that $r_1y_1=\al_1$ is a solution to the second equation of \eqref{selfcons_lemm000}. Moreover, we have checked that the function $f$ is strictly increasing and has bounded derivative on $[0,r_1^{-1}(1-\gamma_n)]$. 

With basic calculus, we can get that 
$$|m_1(0)-m_{1c}(0)|=|y_1-r_1^{-1}\al_1|\prec n^{-1/2}.$$ 
Plugging it into the first equation of \eqref{selfcons_lemm000}, we get 
$$|m_2(0)-m_{2c}(0)|=|y_2-r_2^{-1}\al_2|\prec n^{-1/2}.$$ 
The above two estimates conclude that equation \eqref{Xiz0} is true.
This concludes the proof of Lemma \ref{entry_diagonal}.
\end{proof}

\medskip
\textit{Step 4: Averaged local law.} Finally, we prove the averaged local laws \eqref{aver_in} and \eqref{aver_in1}. For this purpose, we need to use the following {\it{fluctuation averaging}} lemma, whose proof is omitted as it is a standard result in the literature.

\begin{lemma}[Fluctuation averaging] \label{abstractdecoupling}
Under the setting of Proposition \ref{prop_diagonal}, suppose the entrywise local law \eqref{entry_diagonal} holds uniformly in $z\in \mathbf D$. Then, we have that 
\begin{equation}\label{flucaver_ZZ}
|[\cal Z]|+|[\cal Z]_0|+|[\cal Z]_1|+|[\cal Z]_2| \prec (np)^{-1/2},
\end{equation}
uniformly in $z\in \mathbf D$.
\end{lemma}
\begin{proof}
The proof can be found in the works of \citet{EKYY1}.
\end{proof}

Now, plugging equations \eqref{Xizz} and \eqref{flucaver_ZZ} into equations \eqref{selfcons_lemm} and \eqref{selfcons_lemm2} and applying Lemma \ref{lem_stabw}, we get  
\be\label{eq_finerm1m2}
|m_1(z)-m_{1c}(z)|+|m_2(z)-m_{2c}(z)|  \prec (np)^{-1/2}.
\ee
Then, subtracting equation \eqref{defn mc1c} from equation \eqref{Gmumu}  and using equations \eqref{flucaver_ZZ} and \eqref{eq_finerm1m2}, we obtain that 
$$\left|m(z) - m_c(z)\right| \prec (np)^{-1/2}. $$
This is the averaged local law \eqref{aver_in} with $Q=1$. The proof of law \eqref{aver_in1} is similar.
\end{proof}

\subsubsection{Anisotropic Local Law}\label{sec_Gauss}

\begin{proof}
In this section, we prove the anisotropic local law in Theorem \ref{LEM_SMALL}  by extending from Gaussian random matrices to generally distributed random matrices.
With Proposition \ref{prop_diagonal}, it suffices to prove that for $Z^{(1)}$ and $Z^{(2)}$ satisfying the assumptions of Theorem \ref{LEM_SMALL}, we have
\begin{equation*}%
 \left|\mathbf u^\top  \left( G(Z,z) -  G(Z^{\text{Gauss}}, z)\right) \mathbf v \right| \prec n^{-1/2}Q,
\end{equation*}
for any deterministic unit vectors $\mathbf u,\mathbf v\in{\mathbb R}^{p+n}$ and $z\in \mathbf D$, where we abbreviate that 
$$Z:=\begin{pmatrix}Z^{(1)} \\ Z^{(2)}\end{pmatrix}\quad \text{and} \quad Z^{\text{Gauss}}:=\begin{pmatrix}(Z^{(1)})^{\text{Gauss}}\\ (Z^{(2)})^{\text{Gauss}}\end{pmatrix}.$$
Here, $(Z^{(1)})^{\text{Gauss}}$ and $(Z^{(2)})^{\text{Gauss}}$ are Gaussian random matrices satisfying the assumptions of Proposition \ref{prop_diagonal}.
We will prove the above statement using a continuous comparison argument developed in \citet{Anisotropic}. Since the arguments are similar to the ones in Sections 7 and 8 of \citet{Anisotropic}, we will not write down all the details. %

We define a continuous sequence of interpolating matrices between $Z^{\text{Gauss}}$ and $Z$. 

\begin{definition}[Interpolation]\label{defn_interp}
We denote $Z^0:=Z^{\text{Gauss}}$ and $Z^1:=Z$. Let $\rho_{\mu i}^0$ and $\rho_{\mu i}^1$ be respectively the laws of $Z_{\mu i}^0$ and $Z_{\mu i}^1$ for $i\in \cal I_0$ and $\mu \in \cal I_1\cup \cal I_2$. For any $\theta\in [0,1]$, we define the interpolated law
$\rho_{\mu i}^\theta := (1-\theta)\rho_{\mu i}^0+\theta\rho_{\mu i}^1.$ We work on the probability space consisting of triples $(Z^0,Z^\theta, Z^1)$ of independent $n\times p$ random matrices, where the matrix $Z^\theta=(Z_{\mu i}^\theta)$ has law
\begin{equation}\label{law_interpol}
\prod_{i\in \mathcal I_0}\prod_{\mu\in \mathcal I_1\cup \cal I_2} \rho_{\mu i}^\theta(\dd Z_{\mu i}^\theta).
\end{equation}
For $\lambda \in \mathbb R$, $i\in \mathcal I_0$ and $\mu\in \mathcal I_1\cup \cal I_2$, we define the matrix $Z_{(\mu i)}^{\theta,\lambda}$ through
\[\left(Z_{(\mu i)}^{\theta,\lambda}\right)_{\nu j}:=\begin{cases}Z_{\mu i}^{\theta}, &\text{ if } \ \ (j,\nu)\ne (i,\mu)\\ \lambda, &\text{ if }\ \ (j,\nu)=(i,\mu)\end{cases},\]
that is, it is obtained by replacing the $(\mu,i)$-th entry of $Z^\theta$ with $\lambda$. We abbreviate that
$$G^{\theta}(z):=G (Z^{\theta},z ),\quad G^{\theta, \lambda}_{(\mu i)}(z):=G (Z_{(\mu i)}^{\theta,\lambda},z ).$$
\end{definition}

We will prove the anisotropic local law \eqref{aniso_law} through interpolating matrices $Z^\theta$ between $Z^0$ and $Z^1$. We have seen that equation \eqref{aniso_law} holds for $G(Z^0,z)$ by Proposition \ref{prop_diagonal}. Using equation \eqref{law_interpol} and basic calculus, we readily get the following interpolation formula:
for any differentiable function $F:\mathbb R^{n \times p}\rightarrow \mathbb C$,
\begin{equation}\label{basic_interp}
    \frac{\dd}{\dd\theta}\mathbb E F(Z^\theta)=\sum_{i\in\mathcal I_0}\sum_{\mu\in\mathcal I_1\cup \cal I_2}\left[\mathbb E F\left(Z^{\theta,Z_{\mu i}^1}_{(\mu i)}\right)-\mathbb E F\left(Z^{\theta,Z_{\mu i}^0}_{(\mu i)}\right)\right],
\end{equation}
provided all expectations exist.
We will apply equation \eqref{basic_interp} to the function $F(Z):=F_{\mathbf u\mathbf v}^s(Z,z)$ for any fixed $s\in 2\N$, where %
\begin{equation}\label{eq_comp_F(X)}
    F_{\mathbf u\mathbf v}(Z,z):=\left|\mathbf u^\top \left[G (Z,z)-\Gi(z)\right]\mathbf v\right|.
\end{equation}
The main part of the proof is to show the following self-consistent estimate for the right-hand side of equation \eqref{basic_interp}: for any fixed $s\in 2\N$, constant $\e>0$, and $\theta\in[0,1]$,
 \begin{equation}\label{lemm_comp_4}
  \bigg|\sum_{i\in\mathcal I_0}\sum_{\mu\in\mathcal I_1\cup \cal I_2}\left[\mathbb EF_{\mathbf u\mathbf v}^s\left(Z^{\theta,Z_{\mu i}^1}_{(\mu i)},z\right)-\mathbb EF_{\mathbf u\mathbf v}^s\left(Z^{\theta,Z_{\mu i}^0}_{(\mu i)},z\right)\right]\bigg|\lesssim (n^\e q)^{s}+ \E F_{\mathbf u\mathbf v}^s (Z^{\theta},z ) ,
 \end{equation}
where we abbreviate that $q:=n^{-1/2}Q.$ 
If equation \eqref{lemm_comp_4} holds, then combining equation \eqref{basic_interp} with  Gr\"onwall's inequality, we obtain that for any fixed $s\in 2\N$ and constant $\e>0$, 
 \be\label{lemm_comp_4.4}\E\left|\bu^\top \left[G(Z^1,z)-\Pi(z)\right]\bv\right|^s  \lesssim (n^\e q)^{s}.\ee
Then, applying Markov's inequality and noticing that $\e$ can be arbitrarily small, we conclude that equation \eqref{aniso_law} is true.

In order to prove equation \eqref{lemm_comp_4}, we compare $Z^{\theta,Z_{\mu i}^0}_{(\mu i)}$ and $Z^{\theta,Z_{\mu i}^1}_{(\mu i)}$ via a common $Z^{\theta,0}_{(\mu i)}$, that is, we will prove that for any constant $c>0$, $\al\in \{0,1\}$, and $\theta\in[0,1]$,
\begin{align}\label{lemm_comp_5}
    \bigg|\sum_{i\in\mathcal I_0}\sum_{\mu\in\mathcal I_1\cup \cal I_2}\left[\mathbb EF_{\mathbf u\mathbf v}^s\left(Z^{\theta,Z_{\mu i}^\al}_{(\mu i)},z\right)-\mathbb EF_{\mathbf u\mathbf v}^s\left(Z^{\theta,0}_{(\mu i)},z\right)\right] - \cal A \bigg| 
    \lesssim (n^\e q)^{s}+\E F_{\mathbf u\mathbf v}^s(Z^{\theta},z),
\end{align}
for a quantity $\cal A$ that does not depend on $\al\in \{0,1\}$. Underlying the proof of equation \eqref{lemm_comp_5} is an expansion approach which we describe now. 
We define the $\mathcal I \times \mathcal I$ matrix $\Delta_{(\mu i)}^\lambda$ as
\begin{equation*}%
\Delta_{(\mu i)}^{\lambda} :=\lambda \left( {\begin{array}{*{20}c}
   { 0 } &   \mathbf u_i^{(\mu)} \mathbf e_\mu^\top     \\
   {\mathbf e_\mu {\mathbf u_i^{(\mu)}}^\top } & {0}  \\
   \end{array}} \right),%
\end{equation*}
where $\bu_i^{(\mu)}:=\Lambda U^\top \mathbf e_i$ if $\mu \in \cal I_1$, $\bu_i^{(\mu)}:=V^\top \mathbf e_i$ if $\mu \in \cal I_2$, and $\mathbf e_i$ and $\mathbf e_\mu$ denote the standard basis vectors along the $i$-th and $\mu$-th directions. Then, by the definition of $H$ in equation \eqref{linearize_block}, applying a Taylor expansion, we get that for any $\lambda,\lambda'\in \mathbb R$ and $K\in \mathbb N$,
\begin{equation}\label{eq_comp_expansion}
G_{(\mu i)}^{\theta,\lambda'} = G_{(\mu i)}^{\theta,\lambda}+n^{-\frac{k}2}\sum_{k=1}^{K}  G_{(\mu i)}^{\theta,\lambda}\left( \Delta_{(\mu i)}^{\lambda-\lambda'} G_{(\mu i)}^{\theta,\lambda}\right)^k+ n^{-\frac{K+1}2}G_{(\mu i)}^{\theta,\lambda'}\left(\Delta_{(\mu i)}^{\lambda-\lambda'} G_{(\mu i)}^{\theta,\lambda}\right)^{K+1}.
\end{equation}
Using this expansion and the bound in equation \eqref{priorim}, it is easy to prove the following estimate: if $y$ is a random variable satisfying $|y|\prec Q$ (note that the entries of any interpolating matrix $Z^\theta$ satisfy this bound), then {w.o.p.},
\begin{equation}\label{comp_eq_apriori}
    \max_{i\in\sI_0}\max_{\mu\in\sI_1 \cup \cal I_2}  G_{(\mu i)}^{\theta,y}=\OO (1). 
\end{equation}

For simplicity of notations, in the following proof, we denote
$$f_{(\mu i)}(\lambda):=F_{\mathbf u\mathbf v}^s\left(Z_{(\mu i)}^{\theta, \lambda}\right)=\left|\mathbf u^\top \left(G \left(Z_{(\mu i)}^{\theta, \lambda},z\right)-\Gi(z)\right)\mathbf v\right|^s.$$
We use \smash{$f_{(\mu i)}^{(r)}$} to denote the $r$-th order derivative of $f_{(\mu i)}$. By equation \eqref{comp_eq_apriori}, it is easy to see that for any fixed $r\in\bbN$, {$ f_{(\mu i)}^{(r)}(y) =\OO(1)$} w.o.p. for any random variable $y$ satisfying $|y|\prec Q$. 
Then, the Taylor expansion of $f_{(\mu i)}$ gives
\begin{equation*}%
f_{(\mu i)}(y)=\frac{1}{n^{r/2}}\sum_{r=0}^{s+4}\frac{y^r}{r!}f^{(r)}_{(\mu i)}(0)+\OO_\prec\left( q^{s+4}\right).
\end{equation*}
Therefore, we have that for $\al\in\{0,1\}$,
\begin{align}
&\mathbb EF_{\mathbf u\mathbf v}^s\left(Z^{\theta,Z_{\mu i}^\al}_{(\mu i)}\right)-\mathbb EF_{\mathbf u\mathbf v}^s\left(Z^{\theta,0}_{(\mu i)}\right)=\bbE\left[f_{(\mu i)}\left(Z_{\mu i}^\al\right)-f_{(\mu i)}(0)\right]\nonumber\\
=&\bbE f_{(\mu i)}(0)+\frac{1}{2n}\bbE f_{(\mu i)}^{(2)}(0)+\sum_{r=3}^{s+4}\frac{n^{-r/2}}{r!}\bbE f^{(r)}_{(\mu i)}(0)\cdot \bbE\left(Z_{\mu i}^\al\right)^r+\OO_\prec(q^{s+4}). \label{taylor1}
\end{align}
To illustrate the idea in a more concise way, we assume the extra condition 
\be\label{3moment}
\mathbb E (Z^1_{\mu i})^3=0, \quad 1\le \mu \le n,\ \  1\le i \le p.
\ee
Hence, the $r=3$ term in the Taylor expansion \eqref{taylor1} vanishes. However, this condition is not necessary as we will explain at the end of the proof.

Recall that the entries of $Z_{\mu i}^1$ have finite fourth moments by equation \eqref{conditionA2}. Together with the bounded support condition, it gives that
\begin{equation*}%
\left|\bbE\left(Z_{\mu i}^a\right)^r\right| \prec  Q^{r-4} , \quad r \ge 4.
\end{equation*}
We take
$$\cal A=\bbE f_{(\mu i)}(0)+\frac{1}{2n}\bbE f_{(\mu i)}^{(2)}(0).$$
Thus, to show equation \eqref{lemm_comp_5}) under equation \eqref{3moment}, we only need to prove that for $r=4, \cdots,s+4$,
\begin{equation}\label{eq_comp_est}
n^{-2}q^{r-4}\sum_{i\in\mathcal I_0}\sum_{\mu\in\mathcal I_1\cup \cal I_2}\left|\bbE f^{(r)}_{(\mu i)}(0)\right|\lesssim \left(n^\e q\right)^s+\mathbb EF_{\mathbf u\mathbf v}^s(Z^\theta,z) .\end{equation}
In order to get a self-consistent estimate in terms of the matrix $Z^\theta$ only on the right-hand side of equation \eqref{eq_comp_est}, we want to replace $Z^{\theta,0}_{(\mu i)}$ in $f_{(\mu i)}(0)=F_{\mathbf u\mathbf v}^s(Z_{(\mu i)}^{\theta, 0})$ with $Z^\theta \equiv Z_{(\mu i)}^{\theta, Z_{\mu i}^\theta}$. %
\begin{lemma}
Suppose that
\begin{equation}\label{eq_comp_selfest}
n^{-2}q^{r-4}\sum_{i\in\mathcal I_0}\sum_{\mu\in\mathcal I_1 \cup \cal I_2}\left|\bbE f^{(r)}_{(\mu i)}(Z_{\mu i}^\theta)\right|\lesssim \left(n^\e q\right)^s+\mathbb EF_{\mathbf u\mathbf v}^s(Z^\theta) 
\end{equation}
holds for $r=4,\cdots,s+4$. Then, equation \eqref{eq_comp_est} holds for $r=4,\cdots,s+4$.
\end{lemma}
\begin{proof}
The proof is the same as that for Lemma 7.16 of \citet{Anisotropic}.
\end{proof}

What remains now is the proof of equation \eqref{eq_comp_selfest}. For simplicity of notations, we will abbreviate $Z^\theta \equiv Z$ in the following proof. 
For any $k\in \N$, we denote
\be\label{Amui}
A_{\mu i}(k):= \left(\frac{\partial}{\partial Z_{\mu i}}\right)^k \mathbf u^\top\left( G-\Gi\right)\mathbf v.\ee
The derivative on the right-hand side can be calculated using the expansion \eqref{eq_comp_expansion}. In particular, it is easy to check the following bound 
 \begin{equation}\label{eq_comp_A2}
 |A_{\mu i}(k)|\prec \begin{cases}(\mathcal R_i^{(\mu)})^2+\mathcal R_\mu^2, \ & \text{if } k \ge 2, \\
 \mathcal R_i^{(\mu)}\mathcal R_\mu , \ & \text{if } k = 1, \end{cases}
  \end{equation}
where for $i\in \cal I_1$ and $\mu \in \cal I_1\cup \cal I_2$, we denote
\begin{equation}\label{eq_comp_Rs}
\mathcal R_i^{(\mu)}:=|\mathbf u^\top G{ \bu_i^{(\mu)}}|+| \mathbf v^\top G{ \bu_i^{(\mu)}}|,\quad \mathcal R_\mu:=|\mathbf u^\top G\mathbf e_{ \mu}|+|\mathbf v^\top G \mathbf e_{ \mu}|.
\end{equation}
Now, we can calculate the derivative
\begin{align*}
f^{(r)}_{(\mu i)}(Z_{\mu i})=\left(\frac{\partial}{\partial Z_{\mu i}}\right)^r F_{\mathbf u\mathbf v}^s(Z)= \sum_{k_1+\cdots+k_s=r}\prod_{t=1}^{s/2} \left(A_{\mu i}(k_t)\overline{A_{\mu i}(k_{t+s/2})}\right).
\end{align*}
Then, to prove equation \eqref{eq_comp_selfest}, it suffices to show that
\begin{equation}
n^{-2}q^{r-4}\sum_{i\in\mathcal I_0}\sum_{\mu\in\mathcal I_1\cup \cal I_2}\bigg|\bbE\prod_{t=1}^{s/2}A_{\mu i}(k_t)\overline{A_{\mu i}(k_{t+s/2})}\bigg|\lesssim \left(n^\e q\right)^s+\mathbb EF_{\mathbf u\mathbf v}^s(Z,z),\label{eq_comp_goal1}
\end{equation}
for $r=4,\cdots,s+4$ and any $(k_1,\cdots,k_s)\in \N^s$ satisfying $k_1 +\cdots+k_s=r$. 
Treating zero $k_t$'s separately (note that $A_{\mu i}(0)=G_{\mathbf u\mathbf v}-\Gi_{\mathbf u\mathbf v}$ by definition), we find that it suffices to prove
\begin{equation}
\label{eq_comp_goal3}n^{-2}q^{r-4}\sum_{i\in\mathcal I_0}\sum_{\mu\in\mathcal I_1\cup \cal I_2}\bbE|A_{\mu i}(0)|^{s-l}\prod_{t=1}^{l}\left|A_{\mu i}(k_t)\right|\lesssim \left(n^\e q\right)^s+\mathbb EF_{\mathbf u\mathbf v}^s(Z,z),
\end{equation}
for $r=4,\cdots,s+4$ and $ l =1,\cdots, s$. Here, without loss of generality, we have assumed that $k_t=0$ for $l+1\le t \le s$, $k_t \ge 1$ for $1\le t\le l$, and $\sum_{t=1}^l k_t=r$.

There is at least one non-zero $k_t$ in all cases, while in the case with $r\le 2l-2$, there exist at least two $k_t$'s with $k_t=1$ by the pigeonhole principle. Therefore, with equation \eqref{eq_comp_A2}, we get that
\begin{equation}\label{eq_comp_r1}
\prod_{t=1}^{l}\left|A_{\mu i}(k_t)\right| \prec  \one(r\ge 2l-1)\left[(\mathcal R_i^{(\mu)})^2+\mathcal R_\mu^2\right]+\one(r\le 2l-2)(\mathcal R_i^{(\mu)})^2\mathcal R_\mu^2 .
\end{equation}
Using equation \eqref{priorim} and a similar argument as in equation \eqref{GG*}, we can obtain that
\begin{align}
\sum_{i\in\sI_0}(\mathcal R_i^{(\mu)})^2 =\OO(1),\quad  \sum_{\mu\in\sI_1 \cup \cal I_2}\mathcal R_{\mu}^2 =\OO(1),\label{eq_comp_r2}
\end{align}
with overwhelming probability.
Using equations \eqref{eq_comp_r1} and \eqref{eq_comp_r2}, and $n^{-1/2}\le n^{-1/2}Q=q$, we get that %
\begin{align}
&\, n^{-2}q^{r-4}\sum_{i\in\mathcal I_0}\sum_{\mu\in\mathcal I_1\cup \cal I_2}|A_{\mu i}(0)|^{s-l}\prod_{t=1}^{l}\left|A_{\mu i}(k_t)\right| \nonumber\\
\prec &\, q^{r-4} F_{\mathbf u\mathbf v}^{s-l}(Z)\left[\one(r\ge 2l-1) n^{-1} +\one(r\le 2l-2)n^{-2}\right] \nonumber\\
\le &\, F_{\mathbf u\mathbf v}^{s-l}(Z)\left[\one(r\ge 2l-1)q^{r-2}+\one(r\le 2l-2)q^r\right]. \label{add_add}
\end{align}
If $r\le 2l-2$, then we have $q^r\le q^l$ by the trivial inequality $r\ge l$. On the other hand, if $r\ge 4$ and $r\ge 2l-1$, then $r\ge l+2$ and we have $q^{r-2}\le q^{l}$. Thus, with equation \eqref{add_add}, we conclude that  
\begin{align*} 
n^{-2}q^{r-4}\sum_{i\in\mathcal I_0}\sum_{\mu\in\mathcal I_1\cup \cal I_2} \E |A_{\mu i}(0)|^{s-l}\prod_{t=1}^{l}\left|A_{\mu i}(k_t)\right|  
\prec \E F_{\mathbf u\mathbf v}^{s-l}(Z) q^l &\le \left[\E F_{\mathbf u\mathbf v}^{s}(Z)\right]^{\frac{s-l}{s}} q^l \\
&\lesssim  \E F_{\mathbf u\mathbf v}^{s}(Z) + q^s,
\end{align*}
where we used H\"older's inequality in the second step and Young's inequality in the last step. This gives equation \eqref{eq_comp_goal3}, which concludes the proof of equation \eqref{eq_comp_selfest}, and hence of equation \eqref{lemm_comp_5}, and hence of equation \eqref{lemm_comp_4}, which concludes equation \eqref{lemm_comp_4.4} and completes the proof of the anisotropic local law of equation \eqref{aniso_law} under the condition of equation \eqref{3moment}.

Finally, if the condition \eqref{3moment} does not hold, then there is also an $r=3$ term in the Taylor expansion \eqref{taylor1}:
$$\frac{1}{6n^{3/2}}\bbE f^{(3)}_{(\mu i)}(0)\cdot \bbE\left(Z_{i\mu}^\al\right)^3.$$
But the sum over $i$ and $\mu$ in equation \eqref{lemm_comp_5} gives a factor $n^2$, which cannot be canceled by the $n^{-3/2}$ factor in the above equation. In fact, \smash{$\bbE f^{(3)}_{(\mu i)}(0)$} will provide an extra $n^{-1/2}$ factor to compensate the remaining $n^{1/2}$ factor. This follows from an improved self-consistent comparison argument for sample covariance matrices in Section 8 of \citet{Anisotropic}. The argument for our setting is almost the same except for some notational differences, so we omit the details. This completes the proof of equation \eqref{aniso_law} without the condition \eqref{3moment}.
\end{proof}

\subsubsection{Averaged Local Law}\label{section_averageTX}

\begin{proof}
Finally, in this subsection, we focus on the proof of the averaged local law \eqref{aver_in} in the setting of Theorem \ref{LEM_SMALL}, while the proof of the law \eqref{aver_in1} is exactly the same. Our proof is similar to that for the anisotropic local law \eqref{aniso_law} in the previous subsection, and we only explain the main differences. 
In analogy to equation \eqref{eq_comp_F(X)}), we define
\begin{align*}
\wt F(Z,z)  := \bigg|\frac{1}{p}\sum_{i\in\sI_0} \left[G_{ii}(Z,z)- \Gi_{ii}(z)\right]\bigg|.
\end{align*}
Under the notations of Definition \ref{defn_interp}, we have proved that $ \wt F(Z^0,z)\prec (np)^{-1/2}$ in Lemma \ref{prop_entry}. 
To illustrate the idea, we again assume that the condition \eqref{3moment} holds in the following proof. 
Using the arguments in Section \ref{sec_Gauss}, analogous to equation \eqref{eq_comp_selfest} we only need to prove that for $Z=Z^\theta$, $q=n^{-1/2}Q$, any small constant $c>0$, and fixed $s\in 2\N$,
\begin{equation}\label{eq_comp_selfestAvg}
n^{-2}q^{r-4}\sum_{i\in\mathcal I_0}\sum_{\mu\in\mathcal I_1\cup \cal I_2}\left|\bbE \left(\frac{\partial}{\partial Z_{\mu i}}\right)^r\wt F^s(Z)\right|\lesssim (p^{-1/2+c}q)^s+\mathbb E\wt F^s(Z) ,
\end{equation}
for $r=4,...,K$, where $K\in \N$ is a large enough constant. Similar to equation \eqref{Amui}, let
$$A_{j, \mu i}(k):= \left(\frac{\partial}{\partial Z_{\mu i}}\right)^k \left( G_{jj}-\Gi_{jj}\right).$$
Analogous to estimate \eqref{eq_comp_goal1}, it suffices to prove that 
\begin{equation}\nonumber
\begin{split}
& n^{-2}q^{r-4}\sum_{i\in\mathcal I_0}\sum_{\mu\in \cal I_1\cup \mathcal I_2}\bigg|\bbE\prod_{t=1}^{s/2}\bigg(\frac{1}{p}\sum_{j\in\sI_0}A_{j, \mu i}(k_t)\bigg)\bigg(\frac{1}{p}\sum_{j\in\sI_0}\overline{A_{j,\mu i}(k_{t+s/2})}\bigg)\bigg| \\
&  \lesssim (p^{-1/2+c}q)^s +\mathbb E\wt F^s(Z) ,
\end{split}
\end{equation}
for $r=4,...,K$ and any $(k_1,\cdots,k_s)\in \N^s$ satisfying $k_1 +\cdots+k_s=r$. Without loss of generality, we assume that $k_t=0$ for $l+1\le t \le s$, $k_t \ge 1$ for $1\le t\le l$, and \smash{$\sum_{t=1}^l k_t=r$}. Then, it suffices to prove that
\begin{equation*}%
n^{-2}q^{r-4}\sum_{i\in\mathcal I_0}\sum_{\mu\in\mathcal I_1\cup \cal I_2}\bbE \wt F^{s-l}(Z)\prod_{t=1}^{l}\bigg|\frac{1}{p}\sum_{j\in\sI_0} A_{j,\mu i}(k_{t})\bigg|\lesssim  (p^{-1/2+c}q )^s+\mathbb E\wt F^s(Z),
\end{equation*}
for $r=4,...,K$ and any $1\le l \le s$. 

Using equation \eqref{priorim} and a similar argument as in equation \eqref{GG*}, we obtain that for $1\le t \le l$,
\begin{equation}\label{average_bound}
\Big| \sum_{j\in\sI_0} A_{j,\mu i}(k_{t})\Big|\lesssim 1
\end{equation}
with overwhelming probability.
Moreover, taking $\mathbf u =\mathbf v= \mathbf e_j$ in equation \eqref{eq_comp_Rs}, we define
$$\mathcal R_{j,i}^{(\mu)}:=|\mathbf e_j^\top G{ \bu_i^{(\mu)}}|,\quad \mathcal R_{j,\mu}:=|\mathbf e_j^\top G\mathbf e_{ \mu}|.$$
Similar to equation \eqref{eq_comp_r2}, we have that
\begin{align}
\sum_{i\in\sI_0}(\mathcal R_{j,i}^{(\mu)})^2 =\OO(1),\quad  \sum_{\mu\in\sI_1 \cup \cal I_2}\mathcal R_{j,\mu}^2 =\OO(1), \quad \text{w.o.p.}\label{eq_comp_r2_add}
\end{align}
Using equation \eqref{average_bound} and applying equation \eqref{eq_comp_A2} to $A_{j, \mu i}(k_t)$, we obtain that 
\begin{align*}
\prod_{t=1}^{l}\bigg|\frac{1}{p}\sum_{j\in\sI_0} A_{j,\mu i}(k_{t})\bigg| &\prec \one(r\ge 2l-1)p^{-l} +\one(r\le 2l-2)p^{-l}\sum_{j_1,j_2\in \sI_0}\mathcal R_{j_1,i}^{(\mu)}\mathcal R_{j_2,i}^{(\mu)}\mathcal R_{j_1,\mu}\mathcal R_{j_2,\mu} ,
\end{align*}
where we use an argument that is similar to the one for estimate \eqref{eq_comp_r1}. Summing this equation over $i\in \cal I_0$, $\mu\in \cal I_1\cup \cal I_2$ and using estimate \eqref{eq_comp_r2_add}, we get that 
\begin{align*}
&\, n^{-2}q^{r-4}\sum_{i\in\mathcal I_0}\sum_{\mu\in\mathcal I_1\cup \cal I_2}\bbE \wt F^{s-l}(Z)\prod_{t=1}^{l}\bigg|\frac{1}{p}\sum_{j\in\sI_0} A_{j,\mu i}(k_{t})\bigg|\\
\prec &\, q^{r-4}\bbE \wt F^{s-l}(Z) \left[ \one(r\ge 2l-1)p^{-(l-1)} n^{-1} + \one(r\le 2l-2)n^{-2}p^{-(l-2)}\right].
\end{align*}
We consider the following cases (recall that $r\ge 4$ and $r\ge l$).
\begin{itemize}
\item If $r\ge 2l-1$ and $l\ge 2$, then we have $r\ge l+2$ and $l-1\ge l/2$, which gives 
$$p^{-(l-1)} q^{r-4}n^{-1}\le p^{-l/2} q^{r-2} \le (p^{-1/2}q)^l.$$
\item If $r\ge 2l-1$ and $l = 1$, then we have  
$$p^{-(l-1)} q^{r-4}n^{-1}\le n^{-1} \le (p^{-1/2}q)^l.$$
\item If $r\le 2l-2$ and $l\ge 4$, then we have $r\ge l$ and $l-2\ge l/2$, which gives 
$$p^{-(l-2)} q^{r-4}n^{-2}\le p^{-l/2} q^{r} \le (p^{-1/2}q)^l.$$
\item If $r\le 2l-2$ and $l < 4$, then we have $r= 4$ and $l=3$, which gives
$$p^{-(l-2)} q^{r-4}n^{-2}\le p^{-1} n^{-2} \le (p^{-1/2}q)^l.$$
\end{itemize}
Combining the above cases, we get that
\[n^{-2}q^{r-4}\sum_{i\in\mathcal I_0}\sum_{\mu\in\mathcal I_1\cup \cal I_2}\bbE \wt F^{s-l}(Z)\prod_{t=1}^{l}\bigg|\frac{1}{p}\sum_{j\in\sI_0} A_{j,\mu i}(k_{t})\bigg| \prec \bbE\wt F^{s-l}(X) (p^{-1/2}q)^{l}.\]
Applying Holder's inequality and Young's inequality, we can conclude that equation \eqref{eq_comp_selfestAvg} holds, which completes the proof of the averaged local law \eqref{aver_in} under the condition \eqref{3moment}. %

Finally, even if the condition \eqref{3moment} does not hold, using the self-consistent comparison argument in Section 9 of \citet{Anisotropic}, we can still prove equation \eqref{aver_in}. Since (almost) the same argument also works in our setting, we omit the details.
\end{proof}

\section{Proofs for the Model Shift Setting}\label{app_iso_cov}

In this section, we present the proofs of Theorems \ref{cor_MTL_loss}, \ref{prop_lb}, and \ref{thm_Sigma2Id}, and Proposition \ref{claim_model_shift}. 
At the end of this section, we will also derive an approximate estimate of the bias under arbitrary covariate and model shifts.

\subsection{Proof of Theorem \ref{cor_MTL_loss}}\label{subsec:HPSmodel}

The variance estimate \eqref{Lvar_samplesize} can be derived from Theorem \ref{thm_main_RMT} or equation \eqref{fact_tr}. Hence, we focus on the proof of the bias estimate \eqref{Lbias_samplesize}. In the setting of Theorem \ref{cor_MTL_loss}, we can write that
\begin{align*}
L_{\bias} &=\bv^\top {Z^{(1)}}^\top Z^{(1)} \left({Z^{(1)}}^\top Z^{(1)}+ {Z^{(2)}}^\top Z^{(2)}\right)^{-2 }  {Z^{(1)}}^\top Z^{(1)} \bv\\
&=\bv^\top \cal Q^{(1)} \left(\cal Q^{(1)}+ \frac{n_2}{n_1}\cal Q^{(2)}\right)^{-2 } \cal Q^{(1)} \bv,
\end{align*}
where we use the simplified notations
$$\bv :={\Sigma^{(1)}}^{1/2}\left(\beta^{(1)}- \beta^{(2)}\right), \quad \cal Q^{(1)}: = \frac1{n_1}{Z^{(1)}}^\top Z^{(1)},\quad \cal Q^{(2)}: = \frac1{n_2}{Z^{(2)}}^\top Z^{(2)}.$$
Note that $\cal Q^{(1)}$ and $\cal Q^{(2)}$ are both Wishart matrices. Using the rotational invariance of the laws of $\cal Q^{(1)}$ and $\cal Q^{(2)}$, we can simplify $L_{\bias} $ as follows. %

\begin{lemma}\label{claim_reduce_rota}
In the setting of Theorem \ref{cor_MTL_loss}, we have that
\be\label{red_bias1}L_{\bias}  =\left[ 1+\OO_\prec(p^{-1/2})\right] \frac{\|\mathbf v\|^2}{p}\bigtr{\left(\cal Q^{(1)}+ \frac{n_2}{n_1}\cal Q^{(2)}\right)^{-2} (\cal Q^{(1)})^2} .\ee
\end{lemma}
\begin{proof}
It is easy to see that the law of the matrix $\cal A:=\cal Q^{(1)} (\cal Q^{(1)}+ \frac{n_2}{n_1}\cal Q^{(2)})^{-2 } \cal Q^{(1)}$ is rotationally invariant.
Thus, we have that
\begin{align}\label{inserteq0}
\bv^\top \cal A \bv \stackrel{d}{=}  \|\mathbf v\|^2 \frac{\mathbf g^\top}{\|\mathbf g\|} \cal A \frac{\mathbf g}{{\|\mathbf g\|}},
\end{align}
where ``$\stackrel{d}{=}$" means ``equal in distribution", and $\mathbf g=(g_1,\cdots, g_p)$ is a random vector that is independent of $\cal A$ and has i.i.d. Gaussian entries of mean zero and variance one.
By Lemma \ref{SxxSyy}, we have that w.o.p.,
$$\|\cal A\|_F\le p^{1/2}\|\cal A\|\lesssim p^{-1/2}\tr [\cal A].$$
Using equation \eqref{vcalA2}, we get that
\be\label{inserteq1}\|\mathbf g\| = p + \OO_\prec (p^{1/2}),\quad |\mathbf g^\top \cal A \mathbf g- \tr \cal A|\prec \|\cal A\|_F \prec p^{-1/2}\tr \cal A.\ee
Plugging equation \eqref{inserteq1} into equation \eqref{inserteq0}, we conclude that equation \eqref{red_bias1} is true.
\end{proof}

Now, based on estimate \eqref{red_bias1}, we can write that
\begin{align}
    L_{\bias}  &=- \left[ 1+\OO_\prec(p^{-1/2})\right]  {\|\mathbf v\|^2} \cdot  \left. \frac{\dd  h_\al (t)}{\dd t}\right|_{t=0}, \text{ where } \al:= n_2/n_1  \text{ and } \label{dft}\\
h_{\al}(t) &\define \frac1p \bigtr{\frac{1}{\cal Q^{(1)} + t (\cal Q^{(1)})^2+ \al \cal Q^{(2)}}}.\nonumber
\end{align}
Hence, to obtain the asymptotic limit of $L_{\bias} $, we need to calculate the values of $h_\al(t)$ for $t$ around $0$.
Without loss of generality, in the following proof, we mainly focus on the case that $0 < \al\le 1$. That is, the ratio between $n_2$ and $n_1$ does not grow with $p$. Later, we will explain how to extend the proof to the $\al>1$ case.

We calculate $h_\al(t)$ using the Stieltjes transform method in random matrix theory and free additive convolution (or free addition) in free probability theory. We briefly describe the basic concepts that are needed for the proof and refer the interested readers to classical texts such as \citet{bai2009spectral} for a more thorough introduction. The Stieltjes transform of a probability measure $\mu$ supported on $\R$ is a complex function defined as
\be\label{def_stj}m_\mu(z):= \int_0^\infty \frac{\dd\mu(x)}{x-z}, \quad \text{ for } \ \ z\in \C\setminus \supp(\mu).\ee
Given any $p\times p$ symmetric matrix $M$, let $\mu_M:=p^{-1}\sum_{i} \delta_{\lambda_i(M)}$ denote the empirical spectral distribution (ESD) of $M$, where $\lambda_i(M)$ denotes the $i$-th eigenvalue of $M$ and $\delta_{\lambda_i(M)}$ is the point mass measure at $\lambda_i(M)$. Then, it is easy to see that the Stieltjes transform of $\mu_M$ is %
\[ m_{\mu_M}(z) \define \frac{1}{p}\sum_{i=1}^p \frac{1}{\lambda_i(M) - z}= \frac1{p}\tr\left[(M-z\id_{p\times p})^{-1}\right]. \]
Given two $p\times p$ matrices $A_p$ and $B_p$, suppose their ESDs $\mu_{A_p}$ and $\mu_{B_p}$ converge weakly to probability measures $\mu_{A}$ and $\mu_{ B}$, respectively. Let $U_p$ be a sequence of $p\times p$ Haar distributed orthogonal matrices. Then, it is known in free probability theory that the ESD of $ A_p + U_p B_p U_p^\top$ converges to the \emph{free addition} of $\mu_A$ and $\mu_B$, denoted by $\mu_A \boxplus \mu_B$.

It is well-known that the ESDs of Wishart matrices $\Qa$ and $\Qb$ converge weakly to the famous Marchenko-Pastur (MP) law \citep{MP}: $\mu_{\Qa}\Rightarrow \mu^{(1)}$ and $\mu_{\Qb}\Rightarrow \mu^{(2)},$ where $\mu^{(1)}$ and $\mu^{(2)}$ have densities
 $$\rho^{(i)}(x)= \frac1{2\pi \xi_i x}{\sqrt{(\lambda_+^{(i)}-x)(x-\lambda_-^{(i)})}}\mathbf 1_{x\in [\lambda_-^{(i)},\lambda_+^{(i)}]},\quad i=1,2,$$
where we denote $\xi_i={p}/{n_i}$ and the spectrum edges are given by $\lambda_{\pm}^{(i)}:=(1\pm \sqrt{\xi_i})^2.$
Moreover, the Stieltjes transforms of $\mu^{(1)}$ and $\mu^{(2)}$ satisfy the self-consistent equations
$$z\xi_i m^2_{\mu^{(i)}} - (1-\xi_i - z)m_{\mu^{(i)}} +1 =0,\quad i=1,2.$$
With this equation, we can check that $g_i(m_{\mu^{(i)}}(z))=z$, where the function $g_i$ is defined by
\be\label{g_i} g_i(m)= \frac{1}{1+\xi_i m}-\frac1m,\quad i=1,2.\ee
The sharp convergence rates of $\mu_{\Qa}$ and $\mu_{\Qb}$ have also been obtained in Theorem 3.3 of \citet{PY}, that is, %
\be\label{Kol_dist}
d_K\left( \mu_{\cal Q^{(i)}},\mu^{(i)}\right) \prec  p^{-1}, \quad i=1,2,
\ee
where $d_K$ denotes the Kolmogorov distance between two probability measures:
$$d_K\left( \mu_{\cal Q^{(i)}},\mu^{(i)}\right):=\sup_{x\in \R} \left| \mu_{\cal Q^{(i)}}\left((-\infty,x]\right)-\mu^{(i)}\left((-\infty,x]\right)\right|.$$

For any fixed $\al,t\ge 0$, the ESDs of $\al\Qb$ and $\Qa+t(\Qa)^2$ converge weakly to two measures $\mub_\al$ and $\mua_t$ defined through
$$\mub_\al((-\infty,x])= \int \mathbf 1_{\al y \in (-\infty, x]} \dd \mub(y),\quad \mua_t((-\infty,x])= \int \mathbf 1_{y+ ty^2 \in (-\infty, x]} \dd \mua(y).$$
Hence, their Stieltjes transforms are given by
\be\label{def_mt}
m_{\mub_\al}(z)= \frac1{\al}m_{\mu^{(2)}}\left(\frac{z}{\al}\right),\quad m_{\mua_t}(z)= \int \frac{\dd \mua(x)}{x+tx^2 - z}.
\ee
Note that the eigenmatrices of $\Qa+t(\Qa)^2$ and $\al \Qb$ are independent Haar-distributed orthogonal matrices. Hence, the ESD of $\Qa+t(\Qa)^2 + \al \Qb$ converges weakly to the free addition {$\mua_t\boxplus \mub_\al$}. In particular, we will use the following almost sharp estimate on the difference between the Stieltjes transforms of $\mu_{\Qa+t(\Qa)^2 + \al \Qb}$ and  $\mua_t\boxplus \mub_\al$.

\begin{lemma}\label{lem_distance_ab}
In the setting of Theorem \ref{cor_MTL_loss}, suppose $\al, t \in [0,C]$ for a constant $C>0$. Then, we have that
\begin{align}
& \left|\frac{1}{p}\bigtr{\frac{1}{\Qa+t(\Qa)^2 + \al \Qb}} - m_{\mua_t\boxplus \mub_\al}(z=0)\right| \nonumber\\
\prec& \frac1p+ d_K\left( \mu_{\Qa+t(\Qa)^2},\mu^{(1)}_t\right) + d_K\left( \mu_{\al\cal Q^{(2)}},\mu_\al^{(2)}\right) .\label{distance_ab}
 \end{align}
\end{lemma}
\begin{proof}
This lemma is a consequence of Theorem 2.5 of \citet{BES_free1}.
In fact, equation \eqref{distance_ab} is proved for $z=E+\ii\eta$ with {$E\in \supp(\mua_t\boxplus \mub_\al)$} and $\eta>0$ in a follow-up work of \citet{BES_free1}, but the proof there can be repeated almost verbatim in our setting with $z=0$.
\end{proof}

With the above lemma, in order to calculate the right-hand side of \eqref{dft}, we need to calculate $\partial_t m_{\mua_t\boxplus \mub_\al}(z=0)$ at $t=0$. This is given by the following lemma.

\begin{lemma}\label{lem_m'}
We have
\be\label{derv_freeadd}\left. \frac{\dd }{\dd t}m_{\mua_t\boxplus \mub_\al}(0)\right|_{t=0} =- \frac{1 - 2f_1(\al) f_3(\al)  +f_2(\al) f_3(\al) ^2  }{ 1 - \xi_2f_2(\al) f_3(\al)^2 } ,\ee
where the functions $f_1$, $f_2$ and $f_3$ are defined as
\begin{align}
 f_1(\al)&:= m_{\mua_0\boxplus \mub_\al}(0) \nonumber  \\
 &=\frac2{ \al(1-\xi_2) + (1-\xi_1)+ \sqrt{[\al(1-\xi_2) +(1-\xi_1)]^2 + 4\al (\xi_1+\xi_2 -\xi_1\xi_2)}},\label{defnf1}\\
 f_2(\al)&:= \left( \frac1{f_1(\al)^2} - \frac{\xi_1 }{(1+\xi_1f_1(\al) )^2}\right)^{-1},\label{defnf2}\\
 f_3(\al)&:= \frac{\al }{1+ \al \xi_2f_1(\al)}.\label{defnf3}
 \end{align}
\end{lemma}
\begin{proof}
We calculate the Stieltjes transform of the free addition $\mua_t\boxplus\mub_\al$ using the following lemma. The following lemma is known in the literature. %
\begin{lemma}
Given two probability measures, $\mu_1$ and $\mu_2$ on $\R$, there exist unique analytic functions $\omega_1,\omega_2:\C^+\to \C^+$, where $\C^+:=\{z\in \C: \im z>0\}$ is the upper half complex plane, such that the following equations hold: for any $z\in \C^+$,
 \be\label{free_eq0}
 m_{\mu_1}(\omega_2(z))= m_{\mu_2}(\omega_1(z)),\quad \omega_1(z)+\omega_2(z) - z= -\frac{1}{m_{\mu_1}(\omega_2(z))}.
 \ee
Moreover, $m_{\mu_1}(\omega_2(z))$ is the Stieltjes transform of $\mu_1\boxplus \mu_2$, that is,
$$ m_{\mu_1\boxplus \mu_2}(z)=m_{\mu_1}(\omega_2(z)).$$
\end{lemma}

We now solve equation \eqref{free_eq0} for $\mu_1=\mua_t$ and $\mu_2=\mub_\al$ when $z\to 0$:
 \be\label{free_eq}
 m_{\mua_t}(\omega_{2}(\al,t))= m_{\mub_\al}(\omega_1(\al, t)),\quad \omega_1(\al,t)+\omega_2(\al,t) = -\frac{1}{m_{\mua_t}(\omega_2(\al,t))},
 \ee
 where, for simplicity, we omit the argument $z=0$ from $\omega_1(z=0,\al,t)$ and $\omega_2(z=0,\al,t)$. Using the definition of $m_{\mub_\al}$ in equation \eqref{def_mt}, we can check that
\be\label{g2al}
\al g_2\big( \al m_{\mub_\al}(z)\big)=z,
\ee
where $g_2$ is defined in equation \eqref{g_i}. Applying equation \eqref{g2al} to the first equation of \eqref{free_eq}, we get that
\be\nonumber
\omega_1 = \frac{\al}{1+\xi_2  \al m_{\mua_t}(\omega_2)} - \frac1{m_{\mua_t}(\omega_2)}.
\ee
Plugging this equation into the second equation of \eqref{free_eq}, we get that
\be\label{eq_m12}
 \frac{\al}{1+\xi_2  \al m_{\mua_t}(\omega_2)} + \omega_2=0 \ \ \Leftrightarrow \ \
 \al + \omega_2 \left[1+  \al \xi_2   m_{\mua_t}(\omega_2)\right]=0.
\ee
This gives a self-consistent equation of $m_{\mua_t\boxplus \mub_\al}(z=0)= m_{\mua_t}(\omega_2(\al,t))$.

Now, we define the following quantities at $t=0$:
$$f_1(\al):=m_{\mua_0}(\omega_2(\al,0))= m_{\mua_0\boxplus \mub_\al}(z=0),$$
$$ f_2(\al) :=\left. \frac{\dd m_{\mua_0}(z)}{\dd z}\right|_{z=\omega_2(\al, 0)}= \int \frac{ \dd \mua(x)}{[x - \omega_2(\al,0)]^2},\quad f_3(\al):=-\omega_2(\al, 0) .$$
First, from equation \eqref{eq_m12} we can obtain that equation \eqref{defnf3} is true. Using the fact that $g_1$ in equation \eqref{g_i} is the inverse function of $m_{\mua_0}$, we can write equation \eqref{eq_m12} into an equation of $f_1$ only when $t=0$:
\be\nonumber
 \al + \left( \frac{1}{1+\xi_1 f_1} -\frac1{f_1}\right) \left(1+  \al \xi_2  f_1\right)=0 .
 \ee
This equation can be reduced to a quadratic equation:
\be\label{eq_m1234}
 \al \left( \xi_1+\xi_2-\xi_1\xi_2\right)f_1^2 + \left[ \al(1-\xi_2)+(1-\xi_1)\right]f_1 -1=0.
 \ee
By definition, $f_1$ is the Stieltjes transform of $\mua_0\boxplus \mub_\al$ at $z=0$. Since  $\mua_0\boxplus \mub_\al$ is supported on $(0,\infty)$, $f_1$ is positive by equation \eqref{def_stj}. Then, it is not hard to see that the only positive solution of equation \eqref{eq_m1234} is given by equation \eqref{defnf1}. Finally, calculating the derivative of $m_{\mua_0}$ using its inverse function, we obtain that \smash{$ f_2(\al)= [g_1'(f_1)]^{-1}$}, which implies that equation \eqref{defnf2} holds.

To conclude the proof, we still need to calculate $\partial_t m_{\mua_t}(\omega_2(\al,t))|_{t=0}$. Taking the derivative of equation \eqref{eq_m12} with respect to $t$ at $t=0$, we get that
\be\label{calc_partial}\partial_t \omega_2(\al, 0) \cdot \left[1+  \al \xi_2  f_1(\al)\right] - \al \xi_2 f_3(\al) \cdot \left.\partial_t  m_{\mua_t}(\omega_2(\al, t))\right|_{t=0}=0.\ee
Using equation \eqref{def_mt}, we can calculate that
\be\nonumber
 \partial_t  m_{\mua_t}(\omega_2(\al, t))= \partial_t \int \frac{\dd \mua(x)}{x+tx^2 - \omega_2(\al,t)}= -  \int \frac{[x^2 -\partial_t \omega(\al,t)]\dd \mua(x)}{[x+tx^2 - \omega_2(\al,t)]^2}.
\ee
Taking $t=0$ in the above equation, we get that
\begin{align*}
& \left.\partial_t  m_{\mua_t}(\omega_2(\al, t))\right|_{t=0}\\
 =& \partial_t \omega(\al,0)\cdot f_2(\al)-  \int \frac{[(x-\omega_2(\al,0))^2 +2 \omega_2(\al,0)(x-\omega_2(\al,0))+\omega_2(\al,0)^2] \dd \mua(x)}{[x - \omega_2(\al,0)]^2}\\
 =& \partial_t \omega(\al,0)\cdot f_2(\al)-1+ 2 f_1(\al) f_3(\al)- f_2(\al) f_3(\al)^2.
\end{align*}
We can solve from this equation that
\begin{align*}
 \partial_t \omega(\al,0)= \frac{1}{f_2(\al)}\left[\left.\partial_t  m_{\mua_t}(\omega_2(\al, t))\right|_{t=0} +1- 2 f_1(\al)f_3(\al)+ f_2(\al) f_3(\al)^2\right].
\end{align*}
Inserting it into equation \eqref{calc_partial}, we can solve that
$$ \left.\partial_t  m_{\mua_t}(\omega_2(\al, t))\right|_{t=0} =- \frac{1- 2f_1(\al) f_3(\al) + f_2(\al)f_3(\al)^2}{1- \frac{\al \xi_2 f_2(\al)f_3(\al)}{1+  \al \xi_2  f_1(\al) }}. $$
Using $(1+  \al \xi_2  f_1(\al))^{-1}=\al^{-1}f_3(\al)$ by equation \eqref{eq_m12}, we conclude Lemma \ref{lem_m'}.
\end{proof}

Now, we are ready to give the proof of Theorem \ref{cor_MTL_loss}.

\begin{proof}%
The estimate \eqref{Lvar_samplesize} is a special case of Theorem \ref{thm_main_RMT} with $M=\id_{p\times p}$. For a sanity check, we show that it is consistent with the result obtained from the free addition technique. 
We can write the variance term \eqref{Lvar} as
\begin{align}\label{Lvar_pf1}
L_{\vari} = \sigma^2 \bigtr{ \frac1{n_1 \Qa+n_2\Qb}}= \frac{p\sigma^2}{n_1 }\cdot \frac{1}{p}\bigtr{ \frac1{ \Qa+\al \Qb}}.
\end{align}
Combining Lemma \ref{lem_distance_ab} with estimate \eqref{Kol_dist}, we get that
$$\left|\frac{1}{p}\bigtr{\frac{1}{\Qa+  \al \Qb}} - m_{\mua_0 \boxplus \mub_\al}(0)\right|  \prec \frac1p.$$
Recall that $m_{\mua_0\boxplus \mub_\al}(0) =f_1(\al)$ by estimate \eqref{defnf1}. Plugging into $\al= n_2/n_1 $, we indeed get that  
$\frac{p}{n_1}f(\al)=L_1$.

Next, for the bias term, we first consider the case $\al = n_2/n_1\le 1$. Recall that the bias limit is given by estimate \eqref{dft}. Taking $t=p^{-1/2}$, we can use Lemma \ref{SxxSyy} to check that w.o.p.,
$$ \left| \left. \frac{\dd  h_\al (t)}{\dd t}\right|_{t=0} - \frac{h_\al(t)-h_\al(0)}{t}\right| \lesssim t.
$$
Similarly, we have w.o.p.,
 $$ \left| \left. \frac{\dd  m_{\mua_t\boxplus \mub_\al}(0)}{\dd t}\right|_{t=0} - \frac{m_{\mua_t\boxplus \mub_\al}(0)-m_{\mua_0\boxplus \mub_\al}(0)}{t}\right| \lesssim t.$$
On the other hand, using Lemma \ref{lem_distance_ab} and estimate \eqref{Kol_dist}, we get that
  $$\left|h_\al(0) - m_{\mua_0 \boxplus \mub_\al}(0)\right| + \left|h_\al(t) - m_{\mua_t \boxplus \mub_\al}(0)\right|  \prec \frac1p.$$
 Combining the above three estimates, we obtain that
 \begin{align}
 \left| \left. \frac{\dd  h_\al (t)}{\dd t}\right|_{t=0} -\left. \frac{\dd  m_{\mua_t\boxplus \mub_\al}(0)}{\dd t}\right|_{t=0} \right|&\prec t + \frac{\big|h_\al(0) - m_{\mua_0 \boxplus \mub_\al}(0)\big| + \big|h_\al(t) - m_{\mua_t \boxplus \mub_\al}(0)\big| }{t} \nonumber\\
 &\prec p^{-1/2}.\label{eq_derivative}
 \end{align}
Plugging this estimate into estimate \eqref{dft}, after a straightforward calculation estimate using equation \eqref{derv_freeadd}, we can get the estimate \eqref{Lbias_samplesize} when $\al\le 1$. 

Finally, we consider the $\al>1$ (recall that $\alpha = \frac{n_2}{n_1}$) case for the bias term. We rewrite equations \eqref{red_bias1} and \eqref{dft} as
\begin{align*}
L_{\bias} &=\left( 1+\OO_\prec(p^{-\frac 1 2})\right) \cdot \frac{\|\mathbf v\|^2}{\al^2 p}\bigtr{\left(\al^{-1}\cal Q^{(1)}+ \cal Q^{(2)}\right)^{-2} (\cal Q^{(1)})^2} \nonumber\\
&=- \left( 1+\OO_\prec(p^{-\frac 1 2})\right)   {\|\mathbf v\|^2} \cdot  {\wt\al^2}\left. \frac{\dd  \wt h_{\wt\al} (t)}{\dd t}\right|_{t=0}, %
\end{align*}
where $\wt\al:=\al^{-1}=n_1/n_2$ and $\wt h_{\wt\al}(t)$ is defined by
$$\wt h_{\wt\al}(t):=\frac1p \bigtr{\frac{1}{\wt\al \cal Q^{(1)} + t (\cal Q^{(1)})^2+  \cal Q^{(2)}}}.$$
We can estimate $\left. \frac{\dd  \wt h_{\wt\al} (t)}{\dd t}\right|_{t=0}$ in the same way as $\left. \frac{\dd  h_\al (t)}{\dd t}\right|_{t=0}$ for $\wt\al<1$, by establishing analogues of Lemma \ref{lem_distance_ab} and Lemma \ref{lem_m'}. Through these calculations, we can find that $\wt\al^2 \left. \frac{\dd  \wt h_{\wt\al} (t)}{\dd t}\right|_{t=0}$ converges to the same asymptotic limit as $\left. \frac{\dd h_\al (t)}{\dd t}\right|_{t=0}$ as expected. However, the error term becomes 
$$ \OO_\prec (\wt\al^2 p^{-\frac 1 2}) =\OO_\prec\left(\frac{p^{-\frac 1 2}n_1^2}{(n_1+n_2)^2}\right),$$
where we use that $\wt\al< \frac{2n_1}{n_1+n_2}$ when $n_1<n_2$. This concludes equation \eqref{Lbias_samplesize} for the case when $\al>1$. \end{proof}

\subsection{Proof of Proposition \ref{claim_model_shift}}\label{sec:claim_model_shift}
\begin{proof}
Let $a= {p}^{-1}\bigtr{\Sigma^{(1)}}$. Since $\|(\Sigma^{(1)})^{1/2}\left(\beta^{(1)}- \beta^{(2)}\right)\|^2 = (2+\OO(p^{-1/2+c})){\mu^2}a$ with high probability, the limit of $L(\hat{\beta}_2^{\MTL})$ is equal to
\begin{equation}\label{eq:ln1}
    \ell(n_1) \define \sigma^2 L_1 +   2\mu^2 a L_2= \frac{\sigma^2 p}{n_1+n_2-p} +  2 {\mu^2}a \cdot \frac{n_1^2 (n_1+n_2-p)+p n_1n_2}{(n_1+n_2)^2(n_1+n_2-p)}.
\end{equation}
By Theorem \ref{cor_MTL_loss}, %
we have that with high probability,
\begin{equation}\label{eq:Lbeta2}
    L(\hat{\beta}_2^{\MTL}) = \left(1+\OO(p^{-1/2+c})\right)\cdot \ell(n_1).
\end{equation} 
By equation \eqref{fact_tr}, with high probability over the randomness of $X^{(2)}$, the excess risk of the OLS estimator is
\begin{align}\label{eq:Lbeta_STL-add}
    L(\hat{\beta}_2^{\STL})
    = \sigma^2 \cdot \bigtr{\Sigma^{(2)}\big({X^{(2)}}^{\top} X^{(2)} \big)^{-1}}
    =  \left(1+\OO(p^{-1+c})\right)\cdot \frac{\sigma^2 p}{n_2 - p} .
\end{align}
Thus, whether or not $L(\hat{\beta}_2^{\MTL}) \le L(\hat{\beta}_2^{\STL})$ reduces to comparing $\ell(n_1)$ and $\frac{\sigma^2 p}{n_2 - p}$. Let $h(n_1)$ be their difference:
\[ h(n_1) = 2\mu^2 a \cdot \frac{n_1^2 (n_1 + n_2 - p) + p n_1 n_2}{(n_1 + n_2)^2 (n_1 + n_2 - p)} - \frac{\sigma^2 p n_1}{(n_1 + n_2 - p)(n_2 - p)}. \]
The sign of $h(n_1)$ is the same as the sign of the following second-order polynomial in $n_1$:
\begin{align*}
    &\tilde h(n_1)
    = 2\mu^2 a (n_2 - p) (n_1 (n_1 + n_2 - p) + p n_2) - \sigma^2 p (n_1 + n_2)^2 \\
    &= \big(2\mu^2 a (n_2 - p) - \sigma^2 p\big) n_1^2 + (2\mu^2 a (n_2 - p)^2 - 2\sigma^2 p n_2) n_1 + \big(2\mu^2 a (n_2 - p) p n_2 - \sigma^2 p n_2^2\big).
\end{align*}
Let $C_0, C_1, C_2$ be the coefficients of $n_1^0, n_1^1, n_1^2$ terms in $\tilde h(n_1)$.
We prove each claim as follows.
\begin{enumerate}%
    \item If $\mu^2a \le \frac{\sigma^2 p}{2(n_2 - p)}$, then $C_0, C_1, C_2$ are all non-positive, which gives $\tilde h(n_1) \le 0$.

    \item If $\frac{\sigma^2 p}{2(n_2 - p)}< \mu^2a < \frac{\sigma^2 n_2}{2(n_2 - p)}$, then $C_2 > 0$ and $C_0 < 0$.
    Thus, $\tilde{h}(n_1)$ has a positive root and a negative root.
    Let the positive root of $\tilde h(n_1)$ be $n_0\in \R$.
    Then, $\tilde h(n_1) \le 0$ if $n_1 \le n_0$, and $\tilde h(n_1) \ge 0$ otherwise.

    \item If $ \mu^2a\ge \frac{\sigma^2 n_2}{2(n_2 - p)}$, then $C_0$ and $C_2$ are both non-negative.
    Furthermore, we can check that $C_1\ge 0$ using the assumption $n_2 \ge 3p$. Hence, we have $\tilde h(n_1) \ge 0$ for all $n_1$.
\end{enumerate}
Combining these three cases with equations \eqref{eq:Lbeta2} and \eqref{eq:Lbeta_STL-add} concludes the proof.
\end{proof}

\subsubsection{Proof of the Dichotomy} %
\begin{proof}
   The derivative of $\ell(n_1)$ is equal to
   \begin{align*}
       \ell'(n_1)  = -\frac{p\sigma^2}{(n-p)^2} +  \frac{2\mu^2a n_2}{(n-p)^2} \left( 2\left( 1- \frac{n_2}{n}\right)\left( 1- \frac{2p}{n}\right)\left( 1- \frac{p}{n}\right) + \frac{p(n_2-p)}{n^2}\right),
   \end{align*}
   where we denote $n=n_1+n_2$. Hence, the sign of $\ell'(n_1)$ is given by the sign of the following function:
   $$ f(n)\define 2\left( 1- \frac{n_2}{n}\right)\left( 1- \frac{2p}{n}\right)\left( 1- \frac{p}{n}\right) + \frac{p(n_2-p)}{n^2} - \frac{p\sigma^2}{2\mu^2a n_2}.$$
Under the assumption $n\ge n_2\ge 3p$, we see that 
$$f'(n) > \frac{2n_2}{n^2} \left( 1- \frac{2p}{n}\right)\left( 1- \frac{p}{n}\right) - \frac{2p(n_2-p)}{n^3} \ge 0, $$
i.e., $f(n)$ is strictly increasing with respect to $n$. Hence, if $f(n_2)\ge 0$, then case 1 holds; otherwise, if $f(n_2)<0$, then case 2 holds. This concludes the proof.
\end{proof}

\subsection{Proof of Theorem \ref{prop_lb}}\label{proof_lb}

\begin{proof}%
Since $X^{(1)}$ and $ X^{(2)}$ are both isotropic Gaussian matrices, with high probability, they satisfy the following normalization conditions:
    there exists a constant $c>0$, such that for $i=1,2$ and any vector $\theta\in \R^p$,
    \begin{align}
        c \bignorm{\theta} \le \frac 1 {\sqrt {n_i}} \bignorm{X^{(i)} \theta} \le \frac{1}{c} \bignorm{\theta}, \label{eq_norm}
    \end{align}
    where we use $X_j^{(i)}$ to denote the $j$-th column of the matrix $X^{(i)}$.
    For reference, see, e.g., Section 3.2 of \citet{raskutti2011minimax}.
    These two estimates also follow directly from Lemma \ref{SxxSyy}.
    For the rest of the proof, we will condition on $X^{(1)}$ and $X^{(2)}$, which are regarded as fixed design matrices satisfying the normalization condition \eqref{eq_norm}.
    For convenience, denote $n = n_1 + n_2$ and $r_1=n_1/n$.
Our proof is divided into two cases.
    \begin{enumerate}
	   \item[i)] $\frac{n_1^2\mu^2}{(n_1+n_2)^2} \le \frac{\sigma^2 p}{n_1 + n_2}$. Notice that $\frac{\sigma^2 p}{n_1 + n_2}$ is the minimax lower bound for fixed design linear regression when $\mu = 0$, i.e., all of the $n_1 + n_2$ samples are drawn from the same linear regression model.
    Thus, by invoking some standard arguments (see e.g., Example 15.14 of \citet{wainwright2019high}), we can show equation \eqref{eq_lb}.
	   \item[ii)] $\frac{n_1^2\mu^2}{(n_1+n_2)^2} > \frac{\sigma^2 p}{n_1 + n_2}$. Then, we need to show that the minimax rate is at least of order $\min(\frac{n_1^2\mu^2}{(n_1+n_2)^2}, \frac{\sigma^2 p} {n_2} )$. This can be done by constructing an $\OO(r_1\mu )$-cover of the subset of parameter vectors.
    \end{enumerate}
    We next prove each case separately following the arguments in Example 15.14 of  \citet{wainwright2019high}. %
    
    For case i), 
    we design a covering set, denoted as $\cS=\set{\wt\theta^1, \wt\theta^2, \dots, \wt\theta^M}\cup \set{\theta^1, \theta^2, \dots, \theta^M}$, 
    such that the followings hold:
    \begin{itemize}
        \item $\set{\theta^1, \theta^2, \dots, \theta^M}$ is a $2\delta $-packing of the subset of $p$-dimensional vectors of Euclidean length at most $4\delta $.
        \item $\|\wt\theta^i - \theta^i\| = \mu$ for $i = 1, 2, \dots, M$. We achieve this by adding a fixed shift $a$ of length $\mu$ to $\theta^i$, i.e,  $\wt\theta^i = \theta^i + a$.
        \item The size of $\cal S$ can be as large as $2^{p+1}$.
    \end{itemize}
    For any $j=1,\ldots, M$, let $\mathbb P_j$ denote the law of the label vector  $y = \begin{pmatrix} Y^{(1)} \\ Y^{(2)} \end{pmatrix}$   of the two tasks when the true regression vectors are $\beta^{(1)} = \wt\theta^j$ and $\beta^{(2)} = \theta^j$.
    By definition of the linear model, conditioning on $X^{(1)}$ and $X^{(2)}$, 
    $\P_j$ follows a multivariate normal distribution $\cN\left(\begin{pmatrix} X^{(1)} \wt\theta^j \\ X^{(2)} \theta^j \end{pmatrix}, \sigma^2 \id_{n \times n}\right)$.
    Similarly, for any $k\ne j$, let $\mathbb P_k$ denote the law of the label vector when the true regression vectors are $\wt\theta^k$ and $ \theta^k$.
    By standard facts on the KL divergence between two normal distributions, we have
    \begin{align}
        D_{KL}(\mathbb P_j \| \mathbb P_k) &= \frac 1 {2\sigma^2} \left( \bignorm{X^{(1)} (\wt\theta^j - \wt\theta^k)}^2 + \bignorm{X^{(2)} (\theta^j - \theta^k)}^2 \right) \nonumber \\
        &\le \frac{1} {2\sigma^2}\Big(c^{-2} n_1 \| \wt\theta^j - \wt\theta^k \|^2 + c^{-2} n_2 \bignorm{\theta^j - \theta^k}^2 \Big) \nonumber\\ %
        &\le \frac{n}{2c^2\sigma^2}  \bignorm{\theta^j - \theta^k}^2 %
        \le \frac{32 n \delta^2} {c^2 \sigma^2}, \nonumber
    \end{align}
    where we use equation \eqref{eq_norm} in the second step. 
    We need the above bound to satisfy the following condition (see equation (15.35b), Section 15.3.3 of \citet{wainwright2019high}): %
    \begin{align}\frac{1}{2}\log  {|\cS|}  \ge \frac{32n \delta^2}{c^2\sigma^2} +\log 2,\label{eq_kl_condition}\end{align}
   which can be satisfied by setting
    $\delta^2 = \frac{c^2 \sigma^2 p}{100 n}$
    for large enough $p$. 
    Thus, we conclude that
    \begin{align}
        \inf_{\hat\beta}\sup_{\Theta(\mu)} \ex{\frac 1 {n_2}\bignorm{\wt X^{(2)} (\hat\beta - \beta^{(2)} )}^2} 
        &=\inf_{\hat\beta}\sup_{\Theta(\mu)} \ex{\bignorm{\hat\beta - \beta^{(2)}}^2} \label{eq_lb_35} \\
        &\ge \frac{\delta^2} 2 %
        \ge \frac{ c^2 \sigma^2 p} {200 n}
        \ge \frac{c^2} {400} \left(\frac{n_1^2\mu^2}{(n_1+n_2)^2} + \frac{\sigma^2 p} {n_1 + n_2} \right). \nonumber
    \end{align}

    Next, we consider case ii), which is further divided into two cases.
    If $r_1^2\mu^2 \ge \frac{\sigma^2 p}{n_2}$, then the source task samples are not helpful in improving the rate.
    In this case, we just need to follow a similar calculation as above, but use the label vector $y = Y^{(2)}$ instead with the KL divergence $\frac 1 {2\sigma^2} \bignorm{X^{(2)}(\theta^i - \theta^k)}^2$.
    Then, the lower bound in equation \eqref{eq_lb_35} becomes
    \[ \frac {c^2} {200} \frac{\sigma^2 p} {n_2} \ge \frac {c^2} {400} \left(\frac {\sigma^2 p} {n_2} + \frac{\sigma^2 p}{n_1 + n_2}\right). \]  %
    On the other hand, if $\frac{\sigma^2 p}{n_1 + n_2} < r_1^2\mu^2 < \frac{\sigma^2 p}{n_2}$, we again construct a $2\delta$-packing of the subset of $p$-dimensional vectors of Euclidean length at most $4\delta$, but we choose the label vector as $y = Y^{(2)}$. %
    Following a similar argument as above, we reach the conclusion that
    \[ D_{KL}(\mathbb P_j \| \mathbb P_k) \le  \frac {32 n_2 \delta^2} {c^2 \sigma^2}. \] 
    By setting { $\delta^2 = \frac {c^2r_1^2\mu^2} {100}$}, we know that the above bound is less than $\frac{1}{2}(p-1)\log 2$. 
    Thus, the condition for the KL bound \eqref{eq_kl_condition} is satisfied.
    Hence, the lower bound in equation \eqref{eq_lb_35} becomes
    \[ \frac{\delta^2 } 2 =  \frac { c^2r_1^2\mu^2} {200} \ge \frac { c^2} {400} \left(\frac{n_1^2\mu^2}{(n_1+n_2)^2} + \frac{\sigma^2 p}{n_1 + n_2}\right). \] %
    Combining all three cases together and renaming the constant $c$, we complete the proof of estimate \eqref{eq_lb}.
\end{proof}

\subsection{Proof of Theorem \ref{thm_Sigma2Id}}\label{sec_add_modelcov}

Theorem \ref{thm_Sigma2Id} can be proved in a similar way to Theorem \ref{cor_MTL_loss} by using free additions.
\begin{proof}%
In the current setting with $\Sigma^{(2)}=\id$, we denote $\al=n_2/n_1$ and
$$\cal Q^{(1)}: = \frac1{n_1}{X^{(1)}}^\top X^{(1)}= \frac1{n_1}{\Sigma^{(1)}}^{1/2}{Z^{(1)}}^\top Z^{(1)}{\Sigma^{(1)}}^{1/2},\quad \cal Q^{(2)}: = \frac1{n_2}{Z^{(2)}}^\top Z^{(2)}.$$
In terms of these notations, $L_{\vari}$ still satisfies equation \eqref{Lvar_pf1}, and $L_{\bias}$ satisfies a similar equality as estimate \eqref{red_bias1}:
\be\label{red_bias_new}
\begin{split}
L_{\bias} &=\left( 1+\OO_\prec(p^{-1/2})\right) \cdot \frac{2\mu^2}{p} \bigtr{\left(\cal Q^{(1)}+ \al \cal Q^{(2)}\right)^{-2} (\cal Q^{(1)})^2} \\
&=- \left( 1+\OO_\prec(p^{-1/2})\right) \cdot  {2\mu^2} \left. \frac{\dd  h_\al (t)}{\dd t}\right|_{t=0},
\end{split}
\ee
where we use that under the random-effect model, $\beta^{(1)}-\beta^{(2)}$ has i.i.d.~Gaussian entries of mean zero and variance $2\mu^2 / p$.

The ESD of $\Qa$ converges weakly to a deformed MP law $\mu^{(1)}$ determined by the eigenvalues of $\Sigma^{(1)}$ \cite{MP} as described below. Let $\mu^{(0)}$ be the asymptotic ESD of $\cal Q^{(0)}:=X^{(1)}{X^{(1)}}^\top/n_1$. Since $\cal Q^{(0)}$ has the same non-zero eigenvalues as $\Qa$ and $n_1-p$ more zero eigenvalues, we have that
$$\mu^{(0)} = \frac{p}{n_1}\mu^{(1)} + \frac{n_1-p}{n_1 }\delta_0. $$
Hence, the Stieltjes transforms of $\mu^{(0)}$ and $\mu^{(1)}$ satisfy the following relation:
\begin{equation*}%
m_{\mu^{(1)}} =\frac{n_1}{p} m_{\mu^{(0)}} + \frac{n_1-p}{p z}.
\end{equation*}
Moreover, similar to \eqref{eq_self_at}, $m_{\mu^{(0)}}$ satisfies the deformed MP equation
$$\frac{1}{m_{\mu^{(0)}}(z)}=-z+ \frac{1}{n_1}\sum_{i=1}^p \frac{\sigma_i^{(1)}}{1+ \sigma_i^{(1)} m_{\mu^{(0)}}(z)}, $$
where $\sigma_i^{(1)}$, $i=1,\cdots p$, are the eigenvalues of $\Sigma^{(1)}$. Thus, we have that $g_0(m_{\mu^{(0)}}(z))=z$, where recall that the function $g_0$ is defined in \eqref{g1new}.
The sharp convergence rate of $\mu_{\cal Q^{(0)}}$ follows from Theorem 3.12 of \citet{Anisotropic}:
\begin{equation*}%
d_K\left( \mu_{\cal Q^{(0)}},\mu^{(0)}\right) \prec  p^{-1}.
\end{equation*}

Similar to the proof of Theorem \ref{cor_MTL_loss} in Appendix \ref{subsec:HPSmodel}, we need to calculate $m_{\mua_t\boxplus \mub_\al}(0)$ and its derivative $\partial_t m_{\mua_t\boxplus \mub_\al}(0)$ at $t=0$.
First, we follow exactly the same argument as in the proof of Lemma \ref{lem_m'} and derive the self-consistent equation \eqref{eq_m12} for $\omega_2(\al,t)$. Then, we again introduce the following functions:
$$f_1(\al):=m_{\mua_0}(\omega_2(\al,0))= \frac{n_1}{p} m_{\mu^{(0)}}(\omega_2(\al,0)) + \frac{n_1-p}{p \omega_2(\al,0)},$$
$$  f_2(\al) :=\left. \frac{\dd m_{\mua_0}(z)}{\dd z}\right|_{z=\omega_2(\al, 0)}, \quad f_3(\al):=-\omega_2(\al, 0) .$$
We rewrite equation \eqref{eq_m12} at $t=0$ and $\al=n_2/n_1$ as
\be\label{eq_m12_new}
\frac{n_2}{n_1}\left(1+ \frac{\xi_2}{\xi_1} -\xi_2 \right) + \omega_2 \cdot \left[1+   m_{\mu^{(0)}}(\omega_2(\al,0))\right]=0 .
\ee
Using the fact that $g_0$ is the inverse function of $m_{\mu^{(0)}}$, we can rewrite \eqref{eq_m12_new} into an equation of $y_0:=m_{\mu^{(0)}}(\omega_2(\al,0))$ at $t=0$:
\be\nonumber
\frac{n_1+n_2-p}{n_1} + g_0(y_0)  \cdot \left( 1+    y_0\right)=0 ,
\ee
which is the equation \eqref{y0_simple_eq}. We can express $f_1(\al)$ and $f_3(\al)$ as in equation \eqref{f1f2f3_new_eq}. Regarding $f_2$, we have that
$$  f_2(\al)= \frac{n_1}{p} m_{\mu^{(0)}}'(\omega_2(\al,0)) - \frac{n_1-p}{p [\omega_2(\al,0)]^2}= \frac{n_1}{p\cdot g_0'(y_0)} - \frac{n_1-p}{p\cdot [g_0(y_0)]^2} .$$
Finally, following the argument below line \eqref{calc_partial}, we can obtain line \eqref{derv_freeadd}.

From these calculations, we obtain $m_{\mua_0\boxplus \mub_\al}(0)=f_1(\al)$ and $ \partial_t m_{\mua_t\boxplus \mub_\al}(0)|_{t=0}$ given by equation \eqref{derv_freeadd}.
Plugging them into equations \eqref{Lvar_pf1} and \eqref{red_bias_new}, we conclude the proof.
\end{proof}

\subsection{An Approximate Estimate under Covariate and Model Shifts}\label{sec_approx_cm}

With the local convergence of the resolvent, we can also derive an approximate estimate of the bias under arbitrary covariate and model shifts, which may be of independent interest.
The accuracy of this estimate increases as $n_1/p$ increases.

\begin{theorem}[Anisotropic covariance]\label{prop_main_RMT}
Under Assumption \ref{assm_big1}, if $n_1\gg p$, then for any small constant $c>0$, 
with high probability over the randomness of the training samples, the following holds: %
	\begin{align}
		& \bigabs{ L_{\bias} -   \bignorm{\Pi^{1/2} {\Sigma^{(1)}}^{1/2}  \big(\beta^{(1)}-   \beta^{(2)}\big)}^2  } \notag\\
		=& \OO\left(\left(\sqrt{\frac p {n_1}} + n_1^{{2}/{\varphi} - 1/2 + c}\right) \frac{\lambda_1^2}{\lambda_p^2}\right)
        \bignorm{{\Sigma^{(1)}}^{1/2} \left(\beta^{(1)}-  \beta^{(2)}\right) }^2.
        \label{lem_cov_derv_eq}
	\end{align}
				 Above, $\lambda_1$ and $\lambda_p$ are the largest and smallest singular values of $M$, respectively.
				$\Pi $ is a $p\times p$ matrix defined as
				$$\Pi =\frac{n_1^2 }{(n_1+n_2)^2}  M \frac{\al_3   M^\top M+ \al_4 + 1 }{\left(\al_1   M^\top M + \al_2 \right)^2} M^\top,$$
    where $\al_1$ and $\al_2$ are defined in equation \eqref{eq_a12extra}, and 			  $\al_{3}$ and $\al_4$ are the unique positive solutions of the following system of equations:
		\be  \label{eq_a34extra}
		\begin{split}
				& \al_3 + \al_4 = \frac{1}{n_1 + n_2}\sum_{i=1}^p \frac{1}{\lambda_i^2 \al_1 + \al_2}, \\
				& \al_3 + \frac{1}{n_1 + n_2} \sum_{i=1}^p \frac{\lambda_i^2 (\al_2 \al_3-\al_1 \al_4 )}{\left(\lambda_i^2 \al_1 + \al_2\right)^2} = \frac{1}{n_1 + n_2} \sum_{i=1}^p \frac{\lambda_i^2 \al_1}{\left(\lambda_i^2 \al_1 + \al_2\right)^{2}}.
			\end{split}
			\ee
\end{theorem}

\begin{proof}
We first show an estimate on 
 $$\wt L_{\bias} :=n_1^2 \left\| {\Sigma^{(2)}}^{1/2}\hat \Sigma^{-1} \Sigma^{(1)}\left(\beta^{(1)}- \beta^{(2)}\right) \right\|^2.$$
We claim that for any small constant $c>0$, there exists a high probability event on which the following estimate holds:
\begin{align}
				& \bigabs{ \wt L_{\bias}  -  \bignorm{  \Pi^{1/2} {\Sigma^{(1)}}^{1/2} (\beta^{(1)}- \beta^{(2)})}^2   }   \prec  n^{-1/2}Q  \left\|{\Sigma^{(1)}}^{1/2} \left(\beta^{(1)}- \beta^{(2)}\right) \right\|^2  . \label{lem_cov_derv_app1}
			\end{align}

Define the vector $\bv:=V^\top  {\Sigma^{(2)}}^{-1/2} \Sigma^{(1)}\left( \beta^{(1)}-  \beta^{(2)}\right)\in \R^{p}$, and its embedding in $\R^{p+n}$, $\mathbf w =(\bv^\top, \mathbf 0_n^\top)^\top$, where $\mathbf 0_n$ is an $n$-dimensional zero vector. Then, we can write that 
$$ \wt L_{\bias} = \bw^\top \frac{n_1^2}{( \Lambda U^\top {Z^{(1)}}^\top Z^{(1)} U\Lambda  + V^\top {Z^{(2)}}^\top Z^{(2)}V )^2} \bw= \frac{n^2_1}{n^2}\bw^\top G'(0) \bw,$$
where $G'(0)$ denotes the derivative of $G(z)$ with respect to $z$ at $z=0$. Again, we introduce the truncated matrices $\wt Z^{(1)}$ and $\wt Z^{(2)}$ in equation \eqref{truncate1}. With a similar argument as in the above proof of Theorem \ref{thm_main_RMT}, we can show that \eqref{aniso_law} holds for $\bw^\top G(z)\bw$ on the event $\Xi_1=\{\wt Z^{(1)} = Z^{(1)},  \wt Z^{(2)} = Z^{(2)}\}$. 
Combining equation \eqref{aniso_law} with Cauchy's integral formula, we get that on $\Xi_1$, %
\be\label{apply derivlocal}
\begin{split}
  \bw^\top  G'(0)\bw  = \frac{1}{2\pi \ii}\oint_{\cal C} \frac{ \bw^\top G(z)\bw }{z^2}\dd z &=  \frac{1}{2\pi \ii}\oint_{\cal C} \frac{ \bw^\top\Gi(z)\bw}{z^2}\dd z +\OO_\prec(n^{-\frac12}Q\|\mathbf w\|^2) \\
  &=  \bw^\top \Gi'(0)\bw + \OO_\prec(n^{-\frac12}Q\|\mathbf w\|^2),
\end{split}
\ee
where $\cal C$ is the contour $\{z\in \C: |z| = (\log n)^{-1} \}$. With equation \eqref{defn_piw}, we can calculate the derivative $\bw^\top \Gi'(0)\bw$ as
\begin{equation*}%
\bw^\top \Gi'(0)\bw = \bv^\top  \frac{\al_3  \Lambda^2+(1+\al_4)\id_{p\times p}}{(\al_{1}    \Lambda^2 + \al_{2}\id_{p\times p})^2}\bv, 
\end{equation*}
where
$$ \al_3 := - \left. \frac{\dd \al_1(z)}{\dd z}\right|_{z=0}, \quad \al_4: = -\left. \frac{\dd \al_2(z)}{\dd z}\right|_{z=0}.$$
Taking derivatives of the system of equations in equation \eqref{selfomega_a} with respect to $z$ at $z=0$, we can derive equation \eqref{eq_a34extra}. Together with equation \eqref{apply derivlocal}, this concludes that equation \eqref{lem_cov_derv_app1} is true.

Now, with equation \eqref{lem_cov_derv_app1}, to conclude equation \eqref{lem_cov_derv_eq} it remains to bound %
\begin{align*}
L_{\bias}-\wt L_{\bias} 
&= {2 n_1}(\beta^{(1)} - \beta^{(2)})^{\top}{\Sigma^{(1)}}^{1/2}\Delta \left[{\Sigma^{(1)}}^{1/2}\hat \Sigma^{-1} \Sigma^{(2)} \hat \Sigma^{-1} {\Sigma^{(1)}}^{1/2}\right] {\Sigma^{(1)}}^{1/2} (\beta^{(1)} - \beta^{(2)}) \nonumber
		\\
		&\quad + \bignorm{{\Sigma^{(2)}}^{1/2} \hat \Sigma^{-1} {\Sigma^{(1)}}^{1/2}\Delta  {\Sigma^{(1)}}^{1/2}(\beta^{(1)} - \beta^{(2)})}^2,
		\end{align*}
		where we abbreviate $\Delta = {Z^{(1)}}^{\top}Z^{(1)} - {n_1}\id_{p\times p}$. From this equation, we get that 
\begin{align*}
		 \left|L_{\bias}-\wt L_{\bias}\right| 
		 \le \left(\left( n_1 + \|\Delta\|\right)^2 -n_1^2 \right) \bignorm{{\Sigma^{(1)}}^{\frac 1 2} \hat \Sigma^{-1}  \Sigma^{(2)}  \hat \Sigma^{-1} {\Sigma^{(1)}}^{\frac 1 2}} 
		  \cdot \bignorm{{\Sigma^{(1)}}^{\frac 1 2} (\beta^{(1)} - \beta^{(2)})}^2.
	\end{align*}
 By Corollary \ref{fact_minv}, the following estimates holds with high probability for any constant $c>0$:
	\begin{align*}
	\left( n_1 + \|\Delta\|\right)^2 -n_1^2 & \le n_1^2\left[\left( 1+\sqrt{\frac{p}{n_1}}\right)^4 - 1 + n_1^{-1/2+2/\varphi + c}\right], \\
	\left\|{\Sigma^{(2)}}^{1/2}\hat \Sigma^{-1}  {\Sigma^{(2)}}^{1/2}\right\| &\le \frac{1+ n_1^{-1/2+2/\varphi + c}}{ \lambda_p^2 (\sqrt{n_1}-\sqrt{p})^2 + (\sqrt{n_2}-\sqrt{p})^2} .
	\end{align*}
	Combining the above three estimates and $\bignorms{M} \le \lambda_1$, we obtain that with high probability,
	\begin{align*}
	\left|L_{\bias}-\wt L_{\bias}\right| 
	\le \left(\left( 1+\sqrt{\frac{p}{n_1}}\right)^4 - 1 + n_1^{-1/2+2/\varphi + c}\right) \frac{ \lambda_1^2  n_1^2 \left\|{\Sigma^{(1)}}^{\frac 1 2} \left(\beta^{(1)}-  \beta^{(2)}\right) \right\|^2}{  ( \lambda_p^2(\sqrt{n_1}-\sqrt{p})^2 + (\sqrt{n_2}-\sqrt{p})^2)^2} .
	\end{align*}
	Combining this estimate with equation \eqref{lem_cov_derv_app1}, we conclude the proof of Theorem \ref{prop_main_RMT}. 
\end{proof}

\section{Proofs for the Multiple Tasks Setting}\label{sec_multiproof}

We first derive the HPS estimator. For the optimization objective $\ell(A, B)$ in equation \eqref{eq_mtl_same_cov}, using the local optimality condition $\frac {\partial f} {\partial B} = 0$, we can obtain $\hat{B}$ as a function of $A$:
	\begin{align*}
		\hat{B}(A) 
		&= (X^{\top} X)^{-1} X^{\top} Y A^{\top} (AA^{\top})^{+}, %
	\end{align*}
	where $Y := [Y^{(1)}, Y^{(2)}, \dots, Y^{(t)}]$ and $(AA^{\top})^{+}$ denotes the pseudoinverse of $AA^{\top}$.
	Plugging $\hat{B}(A)$ into equation \eqref{eq_mtl_same_cov}, we obtain the following objective that depends only on $A$ (in matrix notation):
	\begin{align}\label{eq_mtl_output_layer}
		g(A) = \bignormFro{X (X^{\top}X)^{-1}X^{\top} Y A^{\top} (AA^{\top})^{+} A - Y}^2.
	\end{align}
	Note that $A^{\top} (AA^{\top})^{+} A$ is a projection onto the subspace spanned by the rows of $A$. For simplicity, we write it into the form
	$$A^{\top} (AA^{\top})^{+} A= U_A U_A^\top,$$
	where $U_A \in \R^{t\times r}$ is a $t\times r$ partial orthonormal matrix (i.e., $U_A^\top U_A=\id_{r\times r}$). Hence, we also denote the function $g(A)$ by $g(U_A)$.

    Now, to find the optimal solution for $U_A$, we write $X (X^{\top} X)^{-1} X^{\top} = P_X$ (which is an $n$ by $n$ projection matrix of rank $p$), and expand $g(A)$ as:
    \begin{align*}
        g(A) = \bignormFro{P_X Y U_A U_A^{\top} - Y}^2
        &= \bignormFro{P_X Y U_A U_A^{\top}}^2 + \bignormFro{Y}^2 - 2\inner{P_X Y U_A U_A^{\top}}{Y} \\
        & = \bignormFro{Y}^2 - \inner{P_X Y U_A U_A^{\top}}{Y} = \bignormFro{Y}^2 - \inner{U_A U_A^{\top}}{Y^{\top} P_X Y},
    \end{align*}
    where we use the fact that $P_X^2 = P_X$ since it is a projection matrix and $U_A^{\top} U_A = \id_{r\times r}$.
    Thus, to minimize $g(A)$, we just need to maximize the inner project between $Y^{\top} P_X Y$ and $U_A U_A^{\top}$---this can be achieved by finding the best rank-$r$ SVD of $Y^{\top} P_X Y$ and taking the top-$r$ singular vectors, which form the partial orthonormal matrix $U_r = U_{\hat A}$.
    Lastly, we can insert $U_r$ as $A$ back to $\hat B A$ to obtain $\hat B \hat A$ as the HPS estimators for all tasks from $1$ to $t$.

    \subsection{Proof of Theorem \ref{thm_many_tasks}}\label{proof_mtl}
    
	First, we can use the concentration estimates in Appendix \ref{app_tool} to simplify the expression of $g(U_A)$.
	In the following proof, we always let $Q=n^{2/\varphi}$.

	\begin{lemma}\label{lem_exp_opt}
	In the setting of Theorem \ref{thm_many_tasks}, for any small constant $c>0$ and large constant $C>0$, there exists a high probability event $\Xi$, on which the following estimate holds:
		  \begin{align}
		\bigabs{g(U_A) - h(U_A)} \le &\ Q n^{\frac 1 2+c} \bignorm{\Sigma^{\frac 1 2} B^{\star} (U_AU_A^{\top} - \id_{t\times t})}_F^2  +  \sigma^2 n^{\frac 1 2 + c} + p^{-C} \bignorm{\Sigma^{\frac 1 2} B^{\star}}_F^2,\label{eq_gA_err}
	\end{align}
	uniformly in all rank-$r$ partial orthonormal matrices $U_A\in \R^{t\times r}$. Here, $h$ is defined by
	\be\label{same_hA}h(U_A):= n \bignormFro{\Sigma^{1/2} B^{\star} \bigbrace{U_{A} U_{A}^{\top} - \id_{t\times t}}}^2 + \sigma^2 (n t - p r).\ee
	\end{lemma}
	\begin{proof} %
	With Corollary \ref{fact_minv} and Corollary \ref{cor_largedeviation}, we can choose a high probability event $\Xi_1$ on which the estimate \eqref{Zv_cor} holds and
    \be\label{Op_norm3}
        C^{-1} n \le \lambda_p ( Z^\top Z)  \le  \lambda_1( Z^\top  Z) \le Cn
    \ee
    for a large constant $C>0$. Throughout the following proof, we assume that event $\Xi_1$ holds.

    To facilitate the analysis, we introduce the following matrix notations. Denote
    \[\cE  :=[\epsilon^{(1)},\epsilon^{(2)},\cdots, \epsilon^{(t)}],  \quad \text{ and } \quad \cW \define X(X^{\top} X)^{-1} X^{\top} \cE U_AU_A^\top. \]
	For any $j = 1,2,\dots, t$, denote
	\begin{align}\label{defnHj}
		H_j &\define  B^{\star} \bigbrace{U_{A} U_{A}^{\top} - \id_{t\times t}}e_j ,  \quad E_j \define (\cW -\cal E)e_j ,
	\end{align}
	where $e_j$ is the standard basis unit vector along the $j$-th direction.
	Then, plugging $Y=XB^\star + \cal E$ into equation \eqref{eq_mtl_output_layer}, we can write the function $g(U_A)$   as
	\[ g(U_A) = \sum_{j=1}^t \bignorm{X H_j + E_j}^2. \]
	We will divide $g(A)$ into three parts.

The first part is %
	\begin{align*} %
		\sum_{j=1}^t \bignorm{X H_j}^2 = \bignormFro{X B^{\star} \bigbrace{U_{A} U_{A}^{\top} - \id_{t\times t}}}^2.
	\end{align*}
	Applying estimate \eqref{Zv_cor} to $X H_j = Z \Sigma^{\frac 1 2} H_j$, we obtain that on $\Xi_1$,
	\begin{align*}
	    \bignorm{X H_j}^2
        &= { n \|\Sigma^{\frac 1 2} H_j\|^2} \cdot \left(1 + \OO_\prec (n^{-\frac 1 2}Q)\right) \\
        &= n \bignorm{\Sigma^{\frac 1 2} B^{\star} \bigbrace{U_{A} U_{A}^{\top} - \id_{t\times t}}e_j}^2\cdot \left(1 + \OO_\prec (n^{-\frac 1 2}Q)\right),
    \end{align*}
	which gives that %
	\begin{align*}
	\bigg|\sum_{j=1}^t \|X H_j\|^2 -  n \bignormFro{\Sigma^{1/2} B^{\star} \bigbrace{U_{A} U_{A}^{\top} - \id_{t\times t}}}^2 \bigg| \prec Q   n^{1/2} \bignorm{\Sigma^{1/2} B^{\star} (U_{A}U_A^{\top} - \id_{t\times t})}_F^2. %
	\end{align*}

The second part is the cross term %
	\begin{align*}%
		2\sum_{j=1}^t\inner{XH_j}{E_j} = 2\inner{XB^{\star} (U_A U_A^{\top} - \id_{t\times t})}{\cW  - \cE}
		= - 2\inner{X B^{\star} (U_A U_A^{\top} - \id_{t\times t})}{\cE}.
	\end{align*}
	Using estimate \eqref{vcalE}, we obtain that %
	\begin{align*}
		  |\inner{XB^{\star} (U_AU_A^{\top} - \id_{t\times t})}{\cE}|  & \prec \sigma   \bignormFro{XB^{\star} (U_AU_A^{\top} - \id_{t\times t})} \nonumber \\
		&\lesssim \sigma n^{ 1/2}  \bignormFro{\Sigma^{1/2} B^{\star}(U_AU_A^{\top} -\id_{t\times t})}  \\
		& \le \sigma^2 n^{1/2} + n^{ 1/2}  \bignormFro{\Sigma^{1/2} B^{\star}(U_AU_A^{\top} -\id_{t\times t})}^2, %
	\end{align*}
where we use $X=Z\Sigma^{\frac 1 2}$ and equation \eqref{Op_norm3} in the second step, and use the AM-GM inequality in the third step. %

The last part is 
	\begin{align*}%
		\sum_{j=1}^t \norm{E_j}^2 &= \bignormFro{\cW  - \cE}^2
		= \bignormFro{\cE}^2 - \inner{\cW }{\cE},
	\end{align*}
	where in the second step we use $\norm{\cW}_F^2 = \inner{\cW }{\cE}$ by algebraic calculation.
	Using estimate \eqref{vcalA2}, we get that
	\be\label{eq_gA_err2} \left| \| \ve^{(i)}\|^2 - \sigma^2 n\right\| \prec \sigma^2 \sqrt n,\ee
	and
	\begin{align}
	   & \ \left|{\epsilon^{(i)}}^\top X(X^{\top} X)^{-1} X^{\top} \epsilon^{(j)} - \delta_{ij}\cdot p\sigma^2 \right| \nonumber\\
	   =& \ \left|{\epsilon^{(i)}}^\top X(X^{\top} X)^{-1} X^{\top} \epsilon^{(j)} - \delta_{ij}\cdot \sigma^2 \tr[X(X^{\top} X)^{-1} X^{\top} ]\right| \nonumber\\
	   \prec &\ \sigma^2 \left\|X(X^{\top} X)^{-1} X^{\top}\right\|_F = \sigma^2\sqrt p. \label{eq_gA_err3.5}
	\end{align}
	Summing up equation \eqref{eq_gA_err2} over $i$, we obtain that %
	\begin{align}
		\abs{\normFro{\cE}^2 - \sigma^2 n t}\le \sum_{i=1}^t \left| \|\ve^{(i)}\|^2 - \sigma^2 n \right|\prec  \sigma^2 \sqrt n. \label{eq_gA_err3}
	\end{align}
	Using line \eqref{eq_gA_err3.5}, we can estimate the inner product between $\cW$ and $\cE$ as %
	\begin{align}
		\bigabs{\inner{\cW }{\cE} - \sigma^2 pr} &= \bigabs{\bigtr{\left(\cE^\top X(X^{\top} X)^{-1} X^{\top} \cE  - p\sigma^2 \cdot  \id_{t\times t} \right)  U_AU_A^{\top} }} \nonumber  \\
		&\le r\bignorm{\cE^\top X(X^{\top} X)^{-1} X^{\top} \cE - p \sigma^2 \cdot  \id_{t\times t}} \prec \sigma^2 \sqrt{n} .\label{eq_gA_err4}
	\end{align}
	Combining equations \eqref{eq_gA_err3} and \eqref{eq_gA_err4}, we obtain that
	\begin{align*}%
		\sum_{j=1}^t \norm{E_j}^2 &=\sigma^2 (nt-pr)+  \OO_\prec(\sigma^2 \sqrt{n} ).
	\end{align*}
	
    Combining the concentration estimates for all three parts, we obtain that on event $\Xi_1$,
	\begin{align*}
	\bigabs{g(U_A) - h(U_A)} &\prec  Q \sqrt{n} \bignorm{\Sigma^{1/2} B^{\star} (U_AU_A^{\top} - \id_{t\times t})}_F^2  +  \sigma^2 \sqrt{n} 
		\end{align*}
	for any fixed $U_A$. Then, using a standard $\ve$-net argument with $\ve=p^{-C}$, we can conclude that equation \eqref{eq_gA_err} is true uniformly in all rank-$r$ partial orthonormal matrices $U_A$.%
	\end{proof}
	From equation \eqref{same_hA}, we see that the global minimizer of $h(U_A)$ is the best rank-$r$ approximation $A^\star$ of ${B^{\star}}^{\top}\Sigma B^{\star}$ defined in equation \eqref{eq_A_star}. On the other hand, let $U_{\hat A}$ be a global minimizer of $g(U_A)$. We have the following characterization of $U_{\hat A}U_{\hat A}^\top$, which is a consequence of Lemma \ref{lem_exp_opt}.

	\begin{lemma}\label{claim_opt_dist}
		In the setting of Theorem \ref{thm_many_tasks}, we have that %
		\[ \bignormFro{U_{\hat{A}} U_{\hat{A}}^{\top} - A^{\star} {A^{\star}}^{\top}}^2
				\lesssim  n^{-1/2+c}  \frac{ Q  \|{B^{\star}}^\top\Sigma B^{\star}\|  +  \sigma^2 }{\lambda_r - \lambda_{r+1} } ,\]
		on the high probability event $\Xi$ in Lemma \ref{lem_exp_opt}.
	\end{lemma}

\begin{proof}%
	Using triangle inequality, we upper bound $h(U_{\hat{A}})-h(A^{\star})$ as
	\begin{align}
		h(U_{\hat{A}})- h(A^{\star})   &\le (g(U_{\hat{A}}) - g({A}^{\star})) + \bigabs{g(A^{\star}) - h(A^{\star})} +  \big|g(U_{\hat{A}}) - h(U_{\hat{A}})\big| \label{eq_g_gap} \\
		&\le \bigabs{g(A^{\star}) - h(A^{\star})}  + \big|g(U_{\hat{A}}) - h(U_{\hat{A}})\big|   \lesssim n^{1/2+c}\left(Q  \|\Sigma^{1/2} B^{\star}\|_F^2 +  \sigma^2 \right), \nonumber
	\end{align}
	on the event $\Xi$.	Here, in the second step, we used the fact that $U_{\hat A}$ is the global minimizer of $g(\cdot)$, so that $g(U_{\hat{A}}) \le g({A}^\star)$, and in the third step we used equation \eqref{eq_gA_err},  $\|U_A U_A^{\top} - \id_{t\times t}\|\le 1$ and $p^{-C}\le Qn^{1/2+c}$ for large enough $C>0$.
	Using the definition of $h(U_A)$ in equation \eqref{same_hA}, we can check that
	$$h(U_{\hat{A}})-h(A^\star) = n \cdot \bigtr{{B^\star}^\top\Sigma B^\star ( A^\star {A^\star}^\top -U_{\hat A}U^\top_{\hat A})} .$$
	For $1\le i \le t$, let $\lambda_i $ be the $i$-th largest eigenvalue of ${B^\star}^\top\Sigma B^\star$ and $v_i$ be the corresponding eigenvector.
	Then, we have $A^\star {A^\star}^\top =\sum_{i=1}^r v_i v_i^\top$, and
	\begin{align}
	 h(U_{\hat{A}})-h(A^\star) & = n \sum_{i=1}^r \lambda_i - n\sum_{i=1}^t \lambda_i \| U^\top_{\hat A} v_i\|^2  \nonumber\\
	&= n\sum_{i=1}^r \lambda_i\left(1 -  \| U^\top_{\hat A} v_i\|^2\right)-n\sum_{i=r+1}^t \lambda_i \| U^\top_{\hat A} v_i\|^2 \nonumber\\
	& \ge  n(\lambda_r-\lambda_{r+1}) \sum_{i=r+1}^t \| U^\top_{\hat A} v_i\|^2 , \label{bdd_A-A}
	\end{align}
	where in the last step we use that
	$$\sum_{i=1}^r \left(1 -  \| U^\top_{\hat A} v_i\|^2\right) = r- \sum_{i=1}^r  \| U^\top_{\hat A} v_i\|^2 =\sum_{i=r+1}^t \| U^\top_{\hat A} v_i\|^2  .$$
	From equation \eqref{eq_g_gap} and equation \eqref{bdd_A-A}, we obtain that
	\be\label{bdd_A-A2} \sum_{i=r+1}^t \| U^\top_{\hat A} v_i\|^2 \lesssim n^{-1/2+c}\frac{Q \bignormFro{\Sigma^{1/2} B^{\star}}^2 + \sigma^2 }{\lambda_r- \lambda_{r+1}}  .\ee
	On the other hand, we have
	\begin{align*}
	\left\| A^\star {A^\star}^\top -U_{\hat A}U^\top_{\hat A}\right\|_F^2
	&= 2r - 2\left\langle A^{\star}{A^{\star}}^{\top}, U_{\hat{A}} {U_{\hat{A}}}^{\top}\right\rangle = 2 \sum_{i=r+1}^t \| U^\top_{\hat A} v_i\|^2.
	\end{align*}
	Combined with equation \eqref{bdd_A-A2} and $\|\Sigma^{1/2} B^{\star}\|_F^2\lesssim  \|{B^{\star}}^\top\Sigma B^{\star}\| $, we conclude the proof.
	\end{proof}

	The last piece of the proof of Theorem \ref{thm_many_tasks} is the following concentration estimate on the prediction loss of $\hat{\beta}_i^{\MTL}(A)$. %
	\begin{lemma}\label{claim_pred_err}
	In the setting of Theorem \ref{thm_many_tasks}, denote %
	$$ \hat{\beta}_i^{\MTL}(U_A):= \hat{B}(A)A=(X^{\top} X)^{-1} X^{\top} Y U_A U_A^{\top} ,$$
	and ${a}_i := U_A U_A^{\top} e_i.$ Then, for any small constant $c>0$ and large constant $C>0$, there exists a high probability event $\Xi$, on which the following estimate holds:
		\begin{align}
			&\bigabs{L_i(\hat{\beta}_i^{\MTL}(U_A)) - L_i(B^{\star} {a}_i) - \sigma^2 \norm{{a}_i}^2 \cdot \bigtr{\Sigma (X^{\top}X)^{-1}}}\nonumber \\
			\le&  n^{-1/2+c} \left(L_i(B^{\star} {a}_i) + \sigma^2 \| a_i\|^2\right) + p^{-C} \bignorm{\Sigma^{1/2} B^{\star}}_F^2 ,\label{eq_Li_err}
		\end{align}
		uniformly in all rank-$r$ partial orthonormal matrices $U_A\in \R^{t\times r}$, where $L_i$ is defined
in \eqref{ith_loss}.
\end{lemma}

\begin{proof}%
    The following proof is restricted to a high probability event $\Xi_1$ on which equation \eqref{Op_norm3} holds.
    For any fixed $U_A$, the prediction loss of $\hat{\beta}_i^{\MTL}(U_A)$ is %
	\begin{align*}
		L_i(\hat{\beta}_i^{\MTL}(U_A)) %
		&= \bignorm{\Sigma^{1/2} \left((X^{\top} X)^{-1} X^{\top} Y a_i - \beta^{(i)}\right)}^2  = \bignorm{\Sigma^{1/2} (H_i + R_i)}^2,
	\end{align*}
    where we denote $R_i = (X^{\top} X)^{-1} X^{\top} \cE  {a}_i$ and $H_i$ is defined in equation \eqref{defnHj}. The rest of the proof is similar to that of Lemma \ref{lem_exp_opt}. We divide the prediction loss into three parts.
	
	{\it Part 1:} The first part is the bias term $\norm{\Sigma^{1/2}H_i}^2 = L_i(B^{\star}  {a}_i).$
        
	{\it Part 2:} The second part is the cross term
	$2\inner{\Sigma^{1/2}H_i}{\Sigma^{1/2}R_i}$.
	We can bound it as %
	\begin{align*}
		 \bigabs{\inner{\Sigma^{1/2}H_i}{\Sigma^{1/2}R_i}} &= \bigabs{\inner{X (X^{\top}X)^{-1} \Sigma H_i  }{\cE{a}_i}}  \le \sum_{j=1}^t \abs{ {a}_i(j)} \cdot \bigabs{\inner{X (X^{\top} X)^{-1} \Sigma H_i}{\varepsilon^{(j)}}} \\
		&\prec \sum_{j=1}^t \abs{ {a}_i(j)} \cdot  \sigma \bignorm{X (X^{\top}X)^{-1} \Sigma H_i}  \lesssim   \frac{\norm{ {a}_i}   \sigma}{ n^{1/2} }\bignorm{ \Sigma^{1/2}H_i} \\
		&\le  n^{-1/2}\sigma^2  \norm{ {a}_i}^2 + n^{-1/2}L_i(B^\star a_i).
	\end{align*}
	Here, in the second step, we use $ {a}_i(j)$ to denote the $j$-th coordinate of $ {a}_i$; in the third step, we use equation \eqref{vcalE}; in the fourth step, we use equations \eqref{Op_norm3}, \eqref{assm3} and $\sum_j |a_i(j)|\le \sqrt{t}\|a_i\|$ by Cauchy-Schwarz inequality; in the last step, we use AM-GM inequality and $\|\Sigma^{1/2}H_i\|^2= L_i(B^\star a_i)$.

	{\it Part 3:}  The last part is %
	\begin{align}
	\|\Sigma^{1/2}R_i\|^2 &= \bigg\|\sum_{j=1}^t  {a}_i(j) \Sigma^{1/2}(X^{\top} X)^{-1} X^{\top} \epsilon^{(j)}\bigg\|^2 \nonumber\\
	& = \sum_{1\le j , k \le t} a_i(j)  a_i(k)  {\epsilon^{(j)}}^\top X(X^{\top} X)^{-1}\Sigma (X^{\top} X)^{-1} X^{\top} {\epsilon^{(k)}}.\label{E_i^2}
	\end{align}
	Using estimate \eqref{vcalA2} and $\tr[X(X^{\top} X)^{-1}\Sigma (X^{\top} X)^{-1} X^{\top}]=\tr[ \Sigma (X^{\top} X)^{-1}] $, we obtain that %
	\begin{align}
	& \left|{\epsilon^{(j)}}^\top X(X^{\top} X)^{-1}\Sigma (X^{\top} X)^{-1} X^{\top} {\epsilon^{(k)}} -\delta_{jk}\cdot \sigma^2\tr\left[\Sigma(X^\top X)^{-1}\right]\right|\nonumber \\
	 \prec & \ \sigma^2  \bignormFro{X(X^{\top} X)^{-1}\Sigma (X^{\top} X)^{-1} X^{\top} } \lesssim  \sigma^2 p^{1/2} n^{-1} \le  \sigma^2  n^{-1/2 }, \label{jeqk}
	\end{align}
	where we use estimate \eqref{Op_norm3} in the second step.
	Plugging  estimate \eqref{jeqk} into estimate \eqref{E_i^2}, we obtain that
	\begin{align*}
		\bigabs{\bignorm{\Sigma^{1/2} R_i}^2 -  \sigma^2 \| a_i\|^2 \cdot\tr\left[\Sigma(X^\top X)^{-1}\right]} &\prec  \sigma^2  n^{-1/2 }  \sum_{1\le j,k\le t}|a_i(j)|| a_i(k)| \\
		&\lesssim   n^{-1/2} \sigma^2 \|a_i\|^2.
	\end{align*}

	Combining the concentration error estimates for all three parts, we obtain that on event $\Xi_1$,
	\begin{align}
			\bigabs{L_i(\hat{\beta}_i^{\MTL}) - L_i(B^{\star} {a}_i) - \sigma^2 \norm{{a}_i}^2 \bigtr{\Sigma (X^{\top}X)^{-1}}} 
            \prec  n^{-\frac 1 2} \left(L_i(B^{\star} {a}_i) + \sigma^2 \| a_i\|^2\right), \label{eq_mtl_bv}
		\end{align}
	for any fixed $U_A$. Then, using a standard $\ve$-net argument with $\ve=p^{-C}$, we can conclude equation \eqref{eq_Li_err} holds uniformly in all rank-$r$ partial orthonormal matrices $U_A$. 
\end{proof}

	Combining Lemma \ref{claim_opt_dist} and Lemma \ref{claim_pred_err}, we are ready to prove Theorem \ref{thm_many_tasks}. %
	\begin{proof} %
		Using Lemma \ref{claim_pred_err} and applying Theorem \ref{thm_main_RMT} (with $n_1=0$, $n_2=n$ and $\Sigma^{(2)}=\Sigma$) %
		to $\bigtr{\Sigma (X^{\top}X)^{-1}}$, we get that for $\hat a_i := U_{\hat{A}} U_{\hat{A}}^{\top}e_i$,
		\begin{align}
			&\bigabs{L_i(\hat{\beta}_i^{\MTL}(U_{\hat A})) - L_i(B^{\star} \hat {a}_i) - \frac{p\sigma^2 }{n-p} \norm{\hat {a}_i}^2}\nonumber \\
			\le&  n^{-\frac 1 2+c}  L_i(B^{\star} \hat {a}_i) +  p^{-\frac 1 2}n^{ \frac {2} {\varphi} - \frac 1 2 + c} \frac{p\sigma^2 \| \hat a_i\|^2}{n} + p^{-C}  \|{\Sigma^{1/2} B^{\star}}\|_F^2  \label{eq_Li_err_add0}
		\end{align}
		with high probability. From this equation, using $\varphi>4$, we obtain that
		\begin{align}
			& \bigabs{L_i(\hat{\beta}_i^{\MTL}(U_{\hat A})) - L_i(B^{\star}  {a}^\star_i) - \frac{p\sigma^2  }{n-p}\norm{ {a}^\star_i}^2} \nonumber\\
			 \le&  n^{-1/2+c} \left(\bignorms{\Sigma^{1/2} B^{\star}}^2 + \sigma^2 \|  a^\star_i\|^2\right) +\left(\bignorms{\Sigma^{1/2} B^{\star}}^2  +\sigma^2\right)  \bignorm{\hat a_i-a_i^\star }, \label{eq_Li_err_add}
		\end{align}
		where we also use that $L_i(B^{\star} {a}^\star_i)\le \bignorms{\Sigma^{1/2} B^{\star}}^2$. On the other hand, using Lemma \ref{claim_opt_dist}, we can bound that with high probability,
		$$\bignorm{\hat a_i-a_i^\star }^2
				\lesssim  n^{-1/2+c} \frac{ Q  \bignorms{{B^{\star}}^\top\Sigma B^{\star}} +  \sigma^2 }{\lambda_r - \lambda_{r+1} }.$$
	Plugging it into equation \eqref{eq_Li_err_add}, we conclude that equation \eqref{Li_multi1} is true.
\end{proof}

\subsection{Proof of Proposition \ref{claim_re_multi}}\label{sec:pf-claim_re_multi}

\begin{proof}%
First, as a corollary of equation \eqref{Li_multi1}, by averaging over $i = 1,\dots,t$, with high probability over the randomness of the training examples, the averaged risk satisfies
\begin{align}
        &
        \left|\frac{1}{t}\sum_{i=1}^t L_i(\hat{\beta}_i^{\MTL}) - \frac1t\bignorm{\Sigma^{1/2} B^{\star} (A^\star {A^\star}^{\top} - \id_{t\times t})}_F^2 - \frac{  \sigma^2p}{n-p} \frac{r}{t}  \right|  \nonumber\\
         \le   &
         \bignorm{{B^\star}^\top\Sigma B^\star}_2  n^{\frac 2 {\varphi} - \frac 1 2+c}  +   \sigma^2 n^{-\frac 1 2+c}. \label{Li_multi2}
    \end{align}
However, note that the error bound in equation \eqref{Li_multi2} for the averaged risk is tighter than that in equation \eqref{Li_multi1} for each individual risk. The proof of estimate \eqref{Li_multi2} is postponed until we complete the proof of Proposition \ref{claim_re_multi}.
    
    By Definition \eqref{eq_A_star}, the bias term is equal to the sum of the smallest $(t - r)$ eigenvalues of ${B^{\star}}^{\top}\Sigma B^{\star}$:
    \begin{align*}%
        \bignorm{\Sigma^{1/2} B^{\star} (A^\star {A^\star}^{\top} - \id_{t\times t})}_F^2
        = \bigtr{{B^{\star}}^{\top} \Sigma B^{\star} (\id_{p \times p} - A^{\star}{A^{\star}}^{\top})}.
    \end{align*}
   By the concentration estimates of random vectors in Lemma \ref{largedeviation}, with high probability, the $(i, j)$-th entry of ${B^{\star}}^{\top} \Sigma B^{\star}$ is equal to  
    \begin{align*}%
    	{\beta^{(i)}}^{\top}\Sigma  {\beta^{(j)}} =\beta_0^\top \Sigma \beta_0 + \delta_{ij} \frac{\mu^2 }{p}\bigtr{\Sigma} + \OO\left(p^{-1/2+c'}\|\beta_0\|^2+ p^{-1/2+c'} \mu^2\right),
    \end{align*}
    for any constant $c'>0$. %
    Thus, ${B^{\star}}^{\top} \Sigma B^{\star}$ is (approximately) equal to a rank-$1$ matrix with spectral norm $t\cdot \beta_0^{\top} \Sigma \beta_0$ plus a scalar matrix $\frac{\mu^2 }{p}\bigtr{\Sigma} \id_{t\times t}$. %
    Thus, we have that
    \begin{align*}
        \lambda_1 &= \Big(1+\OO(p^{-1/2+c'})\Big)\left(t\cdot \beta_0^\top \Sigma \beta_0  + \frac{\mu^2}{ p}\bigtr{\Sigma}\right) ,  \\
        \lambda_i &= \Big(1+\OO(p^{-1/2+c'})\Big) \frac{\mu^2}{ p}\bigtr{\Sigma}, \ \forall~i=2,\cdots, t.
    \end{align*}
    This implies that the sum of the smallest $(t - r)$ eigenvalues of ${B^{\star}}^{\top}\Sigma  B^{\star}$ is equal to 
   \[ \left(1+\OO(p^{-1/2+c'})\right)\cdot(t-r)\frac{\mu^2 \bigtr{\Sigma}}{p}, \]
   with high probability.
   Thus, we conclude that with high probability,
    \begin{align}
        g_r(n, \mu) - \frac{\sigma^2 p}{n - p} &=\left(1+\OO(p^{-1/2+c'})\right)\cdot \left({1 - \frac{r}{t}} \right)\frac{\mu^2}{p}\bigtr{\Sigma} + \frac{r}{t} \cdot \frac{ \sigma^2 p}{n - p} - \frac{\sigma^2 p}{n - p} \nonumber \\
        &= \left(1 - \frac r t\right)  \left[\left(1+\OO(p^{-1/2+c'})\right)\cdot\frac{\mu^2\bigtr{\Sigma}}{p} - \frac{\sigma^2p}{n - p}\right] . \label{eq_re_final}
    \end{align}
    Now, we are ready to finish the proof. Let $c'$ be a positive constant smaller than $c$.
    \begin{enumerate}%
        \item If $\mu^2 >(1+p^{-\frac 1 2 + c}) \frac{\sigma^2 p^2}{(n - p) \tr[\Sigma]}$, the coefficient of $(1 - \frac r t)$ in equation \eqref{eq_re_final} is nonnegative and Claim 1 follows.
        \item If $\mu^2 < (1 - p^{-\frac 1 2 + c}) \frac{\sigma^2 p^2}{(n - p)\tr[\Sigma]}$, the coefficient of $(1 - \frac r t)$ in equation \eqref{eq_re_final} is negative.
        Thus, equation \eqref{eq_re_final} is minimized when $r = 1$, and Claim 2 also follows.%
    \end{enumerate}
    This completes the proof of Proposition \ref{claim_re_multi}.
\end{proof}

It remains to prove the estimate in equation \eqref{Li_multi2}.

\begin{proof}
    Summing up equation \eqref{eq_Li_err_add0} over $i$, we obtain that
	\begin{align}
		&\bigabs{\sum_i L_i(\hat{\beta}_i^{\MTL}(U_{\hat A})) -\bignorm{\Sigma^{\frac 1 2} B^{\star} (U_{\hat A}U_{\hat A}^{\top} - \id_{t\times t})}_F^2 - \frac{\sigma^2 p \cdot r}{n-p} } \nonumber \\
		\le& n^{-\frac 1 2 + c} \left(\bignorm{\Sigma^{\frac 1 2} B^{\star} (U_{\hat A}U_{\hat A}^{\top} - \id_{t\times t})}_F^2 + \sigma^2 \right) + p^{-C}  \|{\Sigma^{\frac 1 2} B^{\star}}\|_F^2,\label{eq_Li_err_add1}
	\end{align}
	with high probability, where we use that $\sum_{i=1}^t\| \hat a_i\|^2=r$ and
	$$\sum_{i=1}^t L_i(B^{\star} \hat {a}_i)=\bignorm{\Sigma^{\frac 1 2} B^{\star} (U_{\hat A}U_{\hat A}^{\top} - \id_{t\times t})}_F^2.$$
    With the result from equation \eqref{same_hA}, we can get
	\begin{align*}
		  \bignorm{\Sigma^{\frac 1 2} B^{\star} (U_{\hat A}U_{\hat A}^{\top} - \id_{t\times t})}_F^2 &= \frac{h(U_{\hat A})-\sigma^2 (n t - p r)}{n} \\
		& \ge \frac{h( A^\star)-\sigma^2 (n t - p r)}{n} =  \bignorm{\Sigma^{\frac 1 2} B^{\star} (A^\star{A^\star}^{\top} - \id_{t\times t})}_F^2,
	\end{align*}
	where in the second step we use the fact that $A^\star$ is the global minimizer of $h$. On the other hand, using equation \eqref{eq_g_gap}, we can get that with high probability
	\begin{align*}
	    &\frac{h(U_{\hat A})-\sigma^2 (n t - p r)}{n} 	\le \frac{h(A^\star)-\sigma^2 (n t - p r)}{n} + \OO\left( n^{-\frac 1 2+c} \left(Q \|\Sigma^{\frac 1 2} B^{\star}\|_F^2 +  \sigma^2\right)\right).
	\end{align*}
	Combining the above two estimates, we get that with high probability,
	\begin{align*}
		\bignorm{\Sigma^{\frac 1 2} B^{\star} (U_{\hat A}U_{\hat A}^{\top} - \id_{t\times t})}_F^2 = \bignorm{\Sigma^{\frac 1 2} B^{\star} (A^\star{A^\star}^{\top} - \id_{t\times t})}_F^2
		+ \OO\left( n^{-\frac 1 2+c} \left(Q \|\Sigma^{\frac 1 2} B^{\star}\|_F^2 +  \sigma^2\right)\right).
    \end{align*}
    Plugging this estimate into equation \eqref{eq_Li_err_add1}, we conclude that equation \eqref{Li_multi2} is true.
\end{proof}

\section{Proofs for SPS Estimation}\label{sec:pf_decomp_SPS}

We first derive the bias-variance decomposition for the SPS estimator in Lemma \ref{lem_SPS_loss}.

\subsection{Proof of Lemma \ref{lem_SPS_loss}}

\begin{proof}
Minimizing equation \eqref{target-SPS} leads to the solution for $z$ as:
\begin{equation*}
    \hat z = \bigbrace{{X^{(1)}}^{\top} X^{(1)}+\lambda}^{-1}{X^{(1)}}^{\top} \bigbrace{Y^{(1)}-X^{(1)}B}.
\end{equation*}
Inserting the expression of $\hat z$ into $\ell(B, z)$ leads us to the solution of $B$, which is the SPS estimator given as follows (recall that $\SigSPS$ has been defined in equation \eqref{eq:SigSPS}): 
\begin{align}
    \hat{\beta}_2^{\SPS} 
        &= \SigSPS^{-1} \Bigbrace{\lambda^{-1}{X^{(1)}}^{\top} X^{(1)} {X^{(2)}}^{\top} Y^{(2)} +  {X^{(1)}}^{\top} Y^{(1)} + {X^{(2)}}^{\top}Y^{(2)}} \nonumber\\
        &= \beta^{(2)} + \SigSPS^{-1}
        \Big( {X^{(1)}}^{\top} X^{(1)}\bigbrace{\beta^{(1)} - \beta^{(2)}} 
        +  {X^{(1)}}^{\top}\varepsilon^{(1)} + \bigbrace{\lambda^{-1}{X^{(1)}}^{\top}{X^{(1)}}+1} {X^{(2)}}^{\top}\varepsilon^{(2)} \Big).\nonumber
\end{align}
Hence, the excess risk of $\hat{\beta}_2^{\SPS}$ is given by 
$$L(\hat{\beta}_2^{\SPS})=\left\|{\Sigma^{(2)}}^{\frac 1 2}\SigSPS^{-1}
        \Big( \hat\Sigma^{(1)}\bigbrace{\beta^{(1)} - \beta^{(2)}} 
        +  {X^{(1)}}^{\top}\varepsilon^{(1)} + \bigbrace{ \lambda^{-1}\hat\Sigma^{(1)} +1} {X^{(2)}}^{\top}\varepsilon^{(2)} \Big)\right\|^2.$$

We can now take the expectation of $L(\hat{\beta}_2^{\SPS})$ with respect to $\varepsilon^{(1)}$ and $\varepsilon^{(2)}$.
Equations \eqref{Lbias_SPS} and \eqref{Lvar_SPS} follow after rearranging the terms in the expectation, and we can conclude that equation \eqref{L_SPS_simple} is true.
\end{proof}

\subsection{Proof of Proposition \ref{thm_SPS}} %

As part of the proof, we are going to need the following result.

\begin{proposition}\label{prop_aniso_SPS}
Under the setting of Proposition \ref{thm_SPS}, suppose that $Z^{(1)}$ satisfies the bounded support condition \eqref{eq_support} with $Q= n^{\frac{2} {\varphi}}$. Then, the following anisotropic local law holds: 
for any deterministic unit vectors $\mathbf u, \mathbf v \in \mathbb R^{p}$, we have that
\begin{equation}\label{aniso_SPS}
 \sup_{t\in [-1/2,1/2]}\left| \mathbf u^\top [G_t(-1)-\Gi_t(-1)] \mathbf v \right|  \prec  n^{-1/2}Q,
\end{equation}
\end{proposition}
\begin{proof}
For each fixed $t$, the anisotropic local law was established in the work of \citet{DY} (and its follow-up). Then, using a standard $\ve$-net argument in $t$, we can extend the result uniformly to all $t\in [-1/2,1/2]$ and conclude that equation \eqref{aniso_SPS} is true.
\end{proof}

Now we are ready to present the proof of Proposition \ref{thm_SPS}.
\begin{proof}%
As in Lemma \ref{lem_HPS_loss}, we can first show the following estimate that is similar to equation \eqref{large_devh1}: %
\[ \bigabs{L(\hat{\beta}_2^{\MTL}) - L_{\bias}(\lambda) - L_{\vari}(\lambda) } \prec \frac {L_{\bias}(\lambda) + L_{\vari}(\lambda)}{\sqrt p}. \]
The proof essentially relies on the estimates \eqref{Op_norm} and \eqref{Op_norm2} for the SPS estimation, employing the same reasoning as outlined in Appendix \ref{app_firstpf}.

Next, for $t\in \R$, we denote
\[ \bv:={\Sigma^{(2)}}^{\frac 1 2}\left(\beta^{(1)}- \beta^{(2)}\right),\ \  H_t :=n_2 \left({Z^{(1)}}^{\top}{Z^{(1)}}\right)^{-1} +\frac{n_2}{\lambda}\Sigma^{(2)} + (1+t) \id_{p\times p}. \]
Then, under the setting of Theorem \ref{thm_SPS}, we can write $L_{\bias}(\lambda)$ and $L_{\vari}(\lambda)$ in equations \eqref{Lbias_SPS} and \eqref{Lvar_SPS} as 
\begin{align}\label{SPS_var_bias}
    L_{\bias}(\lambda) = \bv^{\top}  H_0^{-2}\bv,\quad
    L_{\vari}(\lambda) = \frac{p\sigma^2}{n_2} - \frac{\sigma^2}{n_2} \bigtr{H_0^{-1}}- \frac{\sigma^2}{\lambda} \bigtr{\Sigma^{(2)} H_0^{-2}}.
\end{align}
Note that $H_0^{-2}=-\left.\partial_t H_t^{-1}\right|_{t=0}.$
Hence, to evaluate equation \eqref{SPS_var_bias}, we need to derive precise estimates on $\bigtr{H_t^{-1}}$ and $\bigtr{\Sigma^{(2)}H_t^{-1}}$ for $t$ around 0. 
For this purpose, we rewrite $H_t^{-1}$ as 
$$H_t^{-1}= \frac{n_1}{n_2}\wt\Sigma_t^{-\frac 1 2}\left( \id_{p\times p} - \frac{1}{n_1^{-1}\wt\Sigma_t^{\frac 1 2}{Z^{(1)}}^{\top}{Z^{(1)}}\wt\Sigma_t^{\frac 1 2}+1}\right)\wt\Sigma_t^{-\frac 1 2},  $$
where $  \wt \Sigma_t :=\frac{n_1}{\lambda} \Sigma^{(2)} + (1+t) \frac{n_1}{n_2 }\id_{p\times p}.$
Now, define the resolvent $G_t(z)$ and its averaged trace as 
$$G_t(z):=\left(\frac1{n_1}\wt\Sigma_t^{\frac 1 2} {Z^{(1)}}^{\top}{Z^{(1)}} \wt\Sigma_t^{\frac 1 2}-z\id_{p\times p}\right)^{-1},\quad m_t(z):= \frac1{n_1}\tr[G_t(z)].$$
It is well-known that $m_t(z)$ converges to the Stieltjes transform of the MP law \cite{MP}, defined as the unique solution $a_t(z)$ to the deformed MP equation
\be\label{eq_self_at}
\frac{1}{a_t(z)}=- z + \frac{1}{n_1}\bigtr{\frac{\wt\Sigma_t}{\id_{p\times p}+a_t(z)\wt \Sigma_t}}
\ee
such that $\im a_t(z)\ge 0$ whenever $\im z\ge 0$. Furthermore, $G_t(z)$ is close to a deterministic matrix 
$$\Gi_t(z):=-\frac{1}{z(1+a_t(z)\wt\Sigma_t)},$$
in the sense of Proposition \ref{prop_aniso_SPS}, in which we take $z = -1$.

Now, we define 
\be\label{Pit-SPS}\Pi_t: = \frac{n_1}{n_2}\wt\Sigma_t^{-\frac 1 2}\left[ \id_{p\times p} - \Gi_t(-1)\right]\wt\Sigma_t^{-\frac 1 2}=\frac{n_1}{n_2}\frac{a_t(-1)}{\id_{p\times p}+a_t(-1)\wt\Sigma_t}. \ee
By Proposition \ref{prop_aniso_SPS}, $H_t^{-1}$ is well-approximated by $\Pi_t$ in the sense of anisotropic local law. Furthermore, using the argument in the proof of Theorem \ref{cor_MTL_loss} (i.e., the argument above equation \eqref{eq_derivative}), we obtain that for any deterministic unit vectors $\mathbf u, \mathbf v \in \mathbb R^{p}$,
\begin{equation}\label{aniso_SPS2}
\left| \mathbf u^\top ( H_0^{-2} +\left.\partial_t\Pi_t\right|_{t=0} )\mathbf v \right|  \prec  n^{-\frac 1 4}Q^{\frac 1 2}.
\end{equation}
With estimates \eqref{SPS_var_bias}, \eqref{aniso_SPS} and \eqref{aniso_SPS2}, to conclude estimates \eqref{Lvar_sps} and \eqref{Lbias_sps}, it remains to calculate $\Pi_0$ and $\left.\partial_t\Pi_t\right|_{t=0}$.

First, when $t=0$, $\Pi_0$ is equal to the matrix $M_1$ defined in equation \eqref{def-M12-SPS} with $a_0(-1)=a_0$ and $\wt\Sigma_0=\wt\Sigma$. Second, taking the derivative of equation \eqref{eq_self_at} with respect to $t$ at $t=0$, we obtain that 
\be\label{solve_a_SPS}
    \frac{\partial_t a_t(-1)|_{t=0}}{a_0^2}= -\frac{n_1}{n_2}x_0 + \partial_t a_t(-1)|_{t=0} \frac{1}{n_1}\bigtr{\left(\frac{\wt\Sigma_0}{\id_{p\times p}+a_0 \wt \Sigma_0}\right)^2},
\ee
where $x_0$ is defined in equation \eqref{x0-SPS}. Using equation \eqref{eq_self_at}, we obtain that  
\begin{align*}
    \frac{1}{n_1}\bigtr{\left(\frac{\wt\Sigma_0}{\id_{p\times p}+a_0 \wt \Sigma_0}\right)^2}&= \frac{1}{a_0^2}\left( -\frac{p}{n_1} + \frac{2a_0}{n_1}\tr \frac{\wt\Sigma_0}{1+a_0\wt\Sigma_0} + x_0\right)\\
    &= \frac{1}{a_0^2}\left[  -\frac{p}{n_1} + 2a_0\left( \frac{1}{a_0}-1\right)  + x_0\right]
    = \frac{1}{a_0^2}\left( 2 -\frac{p}{n_1} - 2a_0 + x_0\right).
\end{align*}
Plugging it into estimate \eqref{solve_a_SPS}, we can solve that 
\be\label{dervay_SPS}
\partial_t a_t(-1)|_{t=0}= \frac{n_1}{n_2}y_0,
\ee
where $y_0$ is defined in equation \eqref{x0-SPS}. With estimate \eqref{dervay_SPS}, we can calculate that
\begin{align}\label{M2_sps}
    \left.\partial_t\Pi_t\right|_{t=0} = \left(\frac{n_1}{n_2}\right)^2\frac{y_0  - a_0^2}{(\id_{p\times p}+a_0 \wt \Sigma_0)^2}=-M_2,
\end{align}
for $M_2$ defined in equation \eqref{def-M12-SPS}.
 
Finally, we combine the above ingredients to complete the proof. Using the truncation argument in the proof of Theorem \ref{thm_main_RMT} (i.e., the argument below equation \eqref{truncate1}), it suffices to assume that $Z^{(1)}$ satisfies the bounded support condition \eqref{eq_support} with $Q= n^{\frac {2} {\varphi}}$. Then, applying equations \eqref{aniso_SPS}, \eqref{Pit-SPS}, \eqref{aniso_SPS2} and \eqref{M2_sps}, we can evaluate equation \eqref{SPS_var_bias} and conclude that both equations \eqref{Lvar_sps} and \eqref{Lbias_sps} are true. The proof of Proposition \ref{thm_SPS} is finished.
\end{proof}

\bibliography{rf.bib}

\end{document}